\documentclass[a4paper,fleqn]{cas-sc}

\usepackage[authoryear,longnamesfirst]{natbib}

\ExplSyntaxOn
\bool_gset_true:N \g_stm_nologo_bool
\ExplSyntaxOff

\newtheorem{theorem}{Theorem}
\newtheorem{lemma}[theorem]{Lemma}
\newtheorem{corollary}[theorem]{Corollary}
\newdefinition{remark}{Remark}
\newdefinition{definition}{Definition}
\newproof{proof}{Proof}

\graphicspath{{figures/}}

\begin{document}
\let\WriteBookmarks\relax
\def\floatpagepagefraction{1}
\def\textpagefraction{.001}



\shorttitle{Making two action heads agree}
\shortauthors{J. Sun, W. Zhou, B. Yang, X. Xiao and L. Yang}

\title[mode = title]{Making two action heads agree: coordination mechanisms and a runtime collapse certificate for flow-matching policies}


\author[1]{Jinhui Sun}[orcid=0009-0000-4600-2771]
\fnmark[1]
\ead{jinhuisun@njust.edu.cn}
\credit{Conceptualization, Methodology, Software, Validation,
        Formal analysis, Investigation, Data curation,
        Writing -- original draft, Visualization}

\author[1]{Wei Zhou}[orcid=0000-0003-3225-0576]
\cormark[1]
\fnmark[1]
\ead{weizhou@njust.edu.cn}
\credit{Conceptualization, Methodology, Formal analysis, Validation,
        Writing -- review \& editing, Supervision, Project administration,
        Funding acquisition}

\author[1]{Bowen Yang}[orcid=0009-0008-0229-0060]
\ead{bowenyang@njust.edu.cn}
\credit{Software, Investigation, Validation}

\author[1]{Xinliang Xiao}[orcid=0009-0001-9670-0575]
\ead{xinliangxiao@njust.edu.cn}
\credit{Investigation, Data curation, Visualization}

\author[1]{Li Yang}[orcid=0000-0002-8152-7642]
\ead{yangli945@njust.edu.cn}
\credit{Resources, Supervision, Funding acquisition,
        Writing -- review \& editing}

\affiliation[1]{organization={School of Automation, Nanjing University of Science and Technology},
                city={Nanjing}, postcode={210094},
                state={Jiangsu}, country={China}}

\cortext[1]{Corresponding author.}
\fntext[1]{These authors contributed equally to this work.}

\begin{abstract}
A dual-representation flow-matching policy decodes each predicted motion into joint space and into end-effector space, and the residual between the two kinematically equivalent decodings is a physically interpretable runtime signal. On multimodal tasks, independently sampled branches can select different valid modes, and the residual then fires while nothing is wrong. We ask whether the two branches can be made to choose the same mode, and at what cost. Across two robot environments and a non-robotic testbed, the tested mechanisms sort into four classes. An auxiliary latent supplied to both branches but absent from the flow-matching construction is erased at the population optimum: a provable dead end, confirmed within a prespecified $\pm2\%$ equivalence band. Sharing the source noise can coordinate or anti-coordinate: its effect swings by 90.3 percentage points and changes sign with the representation map, a reversal confirmed in a preregistered factorial at ten seeds per cell, tracking the alignment of the decoders' mode basins. Consistency regularization buys intermediate coordination but reduces the valid-pair rate. Training-supported discrete partitions reach near-ceiling coordination robustly. The chance-corrected effect $\Delta$ obeys $|\Delta|\le C_K=\sqrt{G(q)G(p)}$, a capacity built from each branch's own Gini--Simpson diversity alone; the bound organizes the mechanisms into an attainable region and yields a certificate needing no ground-truth labels: it separates coordination from collapse where zero mismatch is ambiguous. On LIBERO-Plus, benign multimodality adds 1.57 percentage points of false alarms to the residual, which remains the strongest evaluated failure signal; the preregistered token intervention neither meets its registered false-alarm criterion nor shows a seed-robust detection change, in a regime whose monitored heads offer almost nothing to coordinate. Code, models and per-run configurations are released at \url{https://github.com/kimo423/dual-head-coordination}.
\end{abstract}

\begin{highlights}
\item Shared source noise coordinates or anti-coordinates; the representation map decides
\item Auxiliary latents outside the CFM construction are erased at the population optimum
\item Training-supported partition tokens reach near-ceiling coordination robustly
\item Zero mismatch can be collapse; a label-free runtime certificate tells them apart
\item Benign multimodality adds 1.57 pp false alarms to the best evaluated failure signal
\end{highlights}


\maketitle

%
%

\section{Introduction}
\label{sec:intro}

Diffusion and flow-matching action heads fit the multimodal action distributions of manipulation \citep{chi2023diffusion,lipman2023flow}, and runtime failure detection for such policies has developed along two main lines: observation-side distribution-shift scores, which are largely agnostic to the policy's action distribution, and distribution-side quantities such as entropy or ensemble disagreement, which can add substantial training or inference cost. A dual-representation policy suggests a third signal. One predicted motion admits two kinematically equivalent decodings, in joint space and in end-effector space, reconciled through differentiable forward kinematics and compared in physical units. Such a residual is inexpensive relative to a deep ensemble and directly tied to the policy output. Its interpretation is nevertheless fragile on multimodal tasks. If two valid routes pass an obstacle on different sides, independently sampled branches may select different routes. The resulting residual is large even though neither branch is wrong. Related failures of disagreement-based uncertainty at multimodal decision points have been observed for policy ensembles \citep{menda2019ensembledagger,lee2025diffdagger}; here the disagreement occurs between two representations of one policy.

A natural response is to give both branches a shared random variable intended to select the same mode. This design raises a more general question: under what conditions can a shared variable coordinate two generative branches that are trained separately and sample from private randomness? The answer depends on both the training objective and the way coordination is measured.

The measurement half is not routine. Whether two branches coordinate cannot be read off the rate at which they disagree: a low mismatch rate can arise from chance agreement, from marginal imbalance, or from collapse onto a single mode. Established evaluation measures have been shown to misreport the property they name, for out-of-domain generalization and for uncertainty estimates alike \citep{nn2025hwang,nn2024sluijterman}. This paper therefore runs the two questions together. First, which shared variables can make two separately trained branches choose the same mode, and at what cost? Second, how must the resulting agreement be measured so that success is not mistaken for chance or collapse? We use the chance-corrected effect
\[
\Delta=\mathrm{BL}-\mathrm{obs},
\]
where $\mathrm{obs}$ is the observed mode-mismatch rate and $\mathrm{BL}$ is the mismatch expected if the branches chose independently from their own marginals. For $K$ modes,
\[
|\Delta|\le C_K:=\sqrt{G(q)G(p)}\le 1-1/K,
\]
with $G$ the Gini--Simpson diversity. The inequality defines an attainable region whose horizontal coordinate is the modal-diversity capacity $C_K$ and whose vertical coordinate is the chance-corrected coordination effect $\Delta$ (Fig.~\ref{fig:region}). Both coordinates are computed from the branches' own outputs, with no ground-truth labels, which is what makes the region a runtime object rather than a post-hoc analysis: Remark~\ref{rmk:correspondence} gives its exact dictionary into classical agreement statistics, and Section~\ref{sec:calculus} states what the dual-generative setting adds, namely the both-valid conditioning, the collapse certificate, and the per-condition baselines.

The experiments occupy four characteristic parts of this region. An auxiliary latent that does not participate in constructing the interpolation state or regression target is ignored by the population squared-CFM optimum and remains on the zero-coordination floor. An amortized posterior can produce zero mismatch while collapsing the modal capacity, placing the model near the origin. Endpoint-consistency regularization produces intermediate coordination but reduces the valid-pair rate. Data-derived discrete symbols provide a robust route to the high-capacity corner in the robot environments. They are sufficient rather than necessary: on a balanced non-robotic testbed, an amortized posterior also reaches the corner. The same testbed separates mechanism from environment by showing that shared source noise can coordinate or anti-coordinate as the representation map changes, and that skewed target priors alter both coordination and behavioral-prior fidelity.

The paper makes three main contributions.
\begin{enumerate}
\item \textbf{A mechanism atlas for cross-branch coordination.} One hundred seventeen robot-side training runs and 423 testbed runs compare auxiliary latents, amortized posteriors, source-noise coupling, consistency regularization, and data-derived symbols under one frozen evaluation calculus. Lemma~\ref{lem:zblind} closes one route in advance: an auxiliary latent absent from the construction of the CFM interpolation state and regression target is ignored at the population squared-loss optimum, and equivalence tests in both robot environments and the testbed place its measured effect within the prespecified $\pm2\%$ band at seed level. Shared source noise is the opposite case: it is part of the interpolation construction, acts as a coupling of the samplers, and its effect swings by 90.3 percentage points and changes sign with the representation map, a reversal confirmed in a preregistered factorial at ten seeds per cell (interaction permutation $p=10^{-4}$). Consistency regularization pays for intermediate coordination with valid-pair rate; training-supported discrete symbols reach the high-capacity corner robustly; and the study separates coordination from matching the data's mode frequencies.

\item \textbf{A coordination calculus whose certificate needs no ground-truth labels.} Theorem~\ref{thm:region} gives an attribution identity, a modal-diversity bound, its reverse use as a coverage certificate, the tight bound at fixed marginals, and the correction required when conditioning on both branches producing valid trajectories. Because $C_K$ depends only on each branch's own modal marginal, the certificate is computable at deployment, where zero mismatch is otherwise ambiguous between coordination and collapse. The resulting reporting scheme keeps coordination, modal diversity, valid-pair rate, and behavioral-prior fidelity separate.

\item \textbf{An external monitoring characterization on LIBERO-Plus.} Benign multimodality raises the cross-representation residual's false-alarm rate by 1.57 percentage points while observation-side baselines show no corresponding surcharge, and the residual remains the best-performing evaluated signal family at a fixed 100-step horizon. The preregistered token intervention does not meet its registered false-alarm criterion, and the detection contrast is inconclusive across three monitor seeds; the monitored heads have very low modal capacity and an almost deterministic shared symbol, so the experiment bounds the regime in which coordination has material to work with, rather than the general value of coordination.
\end{enumerate}

The paper is an analysis of measurement and mechanism rather than a proposal for a new policy or monitor. The constructed environments isolate the mechanisms; the external benchmark tests one practical consequence. Section~\ref{sec:limitations} states the corresponding scope limits.

\begin{figure}[pos=tbp]
  \centering
  \includegraphics[width=\textwidth]{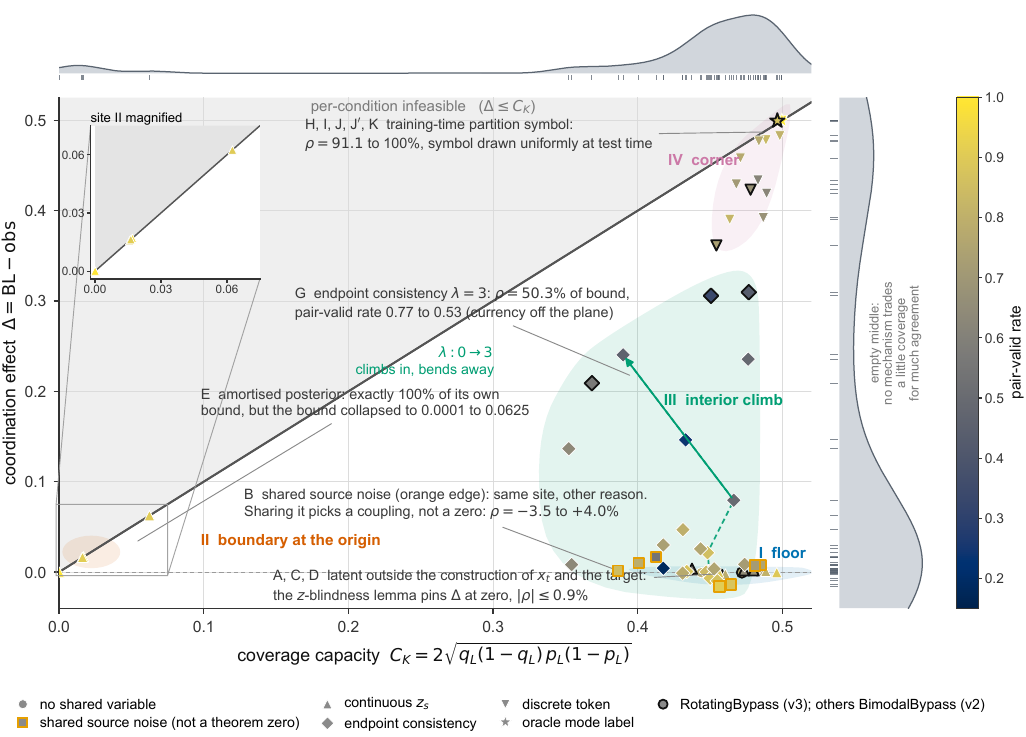}
  \caption{Attainable region for the two robot environments ($K=2$). The horizontal axis is modal-diversity capacity $C$ and the vertical axis is the chance-corrected effect $\Delta=\mathrm{BL}-\mathrm{obs}$, both formed per condition and aggregated with both-valid weights. The diagonal is the feasibility boundary. Each marker is one training run; fill indicates valid-pair rate. Seventeen runs have exact both-valid marginals. The remaining 50 use stored head-valid marginals and are therefore approximate; six lie at most $0.6\%$ above the line, whereas every re-estimable condition satisfies the bound under exact accounting. The dashed path is the shared-source-noise consistency sweep.}
  \label{fig:region}
\end{figure}

\section{Setting and notation}
\label{sec:setting}

\subsection{Dual-branch flow-matching action heads}
\label{sec:heads}

Let $c$ denote the conditioning information (an encoding of proprioceptive state and task context) and let the action chunk have length $H=16$. Two branches model the same future trajectory segment in different spaces:
\begin{itemize}
\item a joint branch, $X^{1}\in\mathbb{R}^{H\times 7}$, with velocity field $v_q(x,t,E(c),z)$;
\item a position branch, $X^{2}\in\mathbb{R}^{H\times 3}$, with velocity field $v_T(x,t,E(c),z)$.
\end{itemize}
The branches share the condition encoder $E(\cdot)$ but share \emph{no} velocity-field parameters (verified on every checkpoint by an empty intersection of parameter identities), and neither branch reads the other's output. The variable $z$ is an optional shared random input. Section~\ref{sec:third-rep} adds a third branch, an end-effector \emph{orientation} head $X^{3}$ that runs flow matching on the continuous 6D rotation representation.

Both branches are trained by conditional flow matching on linear interpolation paths: $x_0\sim\mathcal{N}(0,I)$, $x_1\sim p(\cdot\mid c)$, $x_t=(1-t)x_0+tx_1$, regression target $x_1-x_0$, and branch loss
\begin{equation}
\label{eq:cfm}
\mathcal{L}_i \;=\; \mathbb{E}_{c,\,x_1,\,x_0^i,\,t,\,z}\,\bigl\lVert\, v_i\bigl(x_t^i,\,t,\,E(c),\,z\bigr)-\bigl(x_1^i-x_0^i\bigr)\,\bigr\rVert^2 .
\end{equation}
Coupling of the two branches' source noise is controlled by $\beta\in[0,1]$: $x_0^{1}=\sqrt{1-\beta}\,s+\sqrt{\beta}\,\varepsilon_q$ and $x_0^{2}=\sqrt{1-\beta}\,s_{[:3]}+\sqrt{\beta}\,\varepsilon_T$, with $s\sim\mathcal{N}(0,I_7)$ and the $\varepsilon$'s independent, so that $\beta=0$ is fully shared source noise and $\beta=1$ fully independent. The family leaves the marginal source distribution untouched and is therefore a legitimate instance of initial-noise coupling design \citep{jia2026couple}. All runs with a continuous $z_s$ or a discrete token use $\beta=1$: the only shared quantity is $z$.

An optional endpoint-consistency regularizer acts on the clean prediction at the denoising endpoint, $\hat a = x_t+(1-t)v$:
\begin{equation}
\label{eq:cons}
\mathcal{L}_{\mathrm{cons}} \;=\; \lambda\cdot\mathbb{E}\,\bigl\lVert\, \mathrm{FK}(\hat a^{1})-\hat a^{2}\,\bigr\rVert^2_{W},
\end{equation}
where $\mathrm{FK}$ is differentiable forward kinematics; the analogous penalty on noisy intermediate states is ill-posed. The orientation branch has a counterpart $\mathcal{L}_{\mathrm{rot}}$, the squared geodesic angle between the rotation of $\mathrm{FK}(\hat a^{1})$ and the orientation branch's rotation, with weight $\lambda_{\mathrm{rot}}$.

\subsection{Mode assignment, the coordination effect, and its target population}
\label{sec:mode-def}

Every sampled trajectory is assigned to one of two modes or to \texttt{invalid} (collision, or goal not reached). The assignment rule reads the trajectory's \emph{lateral geometry}---a fact that Section~\ref{sec:tautology} turns on: in BimodalBypass the rule takes the sign of $y$ at the point of largest $|y|$ along the trajectory, with $|y|<\tau$ declared invalid; in RotatingBypass it takes the sign of the larger of two orthogonal projections on a whitened cross-section, with the same threshold role for $\tau$.

\paragraph{Evaluation conventions}
The mode ontology, validity rule, and threshold are fixed before computing any coordination effect.  Both branches are evaluated under the same environment-defined partition, so no label matching is fitted to their outputs.  The threshold $\tau$ is set from task geometry; the coverage-matched values in Section~\ref{sec:envs} use only the independent family's valid-pair rate, and the sensitivity analysis in \ref{app:sensitivity} is reported separately.  Samples inside the dead band are treated as \texttt{invalid} rather than assigned to the nearest mode.  This convention is consequential on LIBERO-Plus, where tightening the boundary rule can reduce the estimated capacity from 0.0055 to zero (Section~\ref{sec:liberoplus-coverage}).

\paragraph{Definitions}
Write $M_q$ and $M_T$ for the two branches' assignments. Every coordination quantity below lives on the \emph{both-valid} event and is estimated from the $N_{\mathrm{BV}}$ sample pairs in which neither branch returned \texttt{invalid}:
\begin{itemize}
\item \emph{valid-pair rate}: $\Pr(M_q\neq I \wedge M_T\neq I)$, reported next to every coordination number below;
\item \emph{mismatch given validity}: $\mathrm{obs}=\Pr(M_q\neq M_T \mid M_q\neq I,\, M_T\neq I)$;
\item per-branch side fractions on that same event, $q_L=\Pr(M_q=L\mid\text{both valid})$ and $p_L=\Pr(M_T=L\mid\text{both valid})$;
\item \emph{independent-choice baseline}: $\mathrm{BL}=q_L(1-p_L)+(1-q_L)p_L$;
\item \emph{same-label coordination effect}: $\Delta=\mathrm{BL}-\mathrm{obs}$, positive when the branches land on the \emph{same} mode more often than their own marginals predict;
\item \emph{coverage capacity}: $C=2\sqrt{q_L(1-q_L)\,p_L(1-p_L)}$, the outer bound of Theorem~\ref{thm:region}(ii);
\item \emph{normalized effect}: $\rho=\Delta/C$, a correlation-like normalization---the phi coefficient of the branches' mode choices, i.e.\ the MCC of the $2\times2$ table---not the fraction attainable at these marginals, which the smaller $C_{K,\mathrm{tight}}$ of Section~\ref{sec:calculus} governs. Defined only for $C>0$: at $C=0$ we report $\Delta=0$ with $\rho$ undefined, never $0\%$ or $100\%$; at $N_{\mathrm{BV}}=0$ nothing here is estimable and only the valid-pair rate is reported.
\end{itemize}
The convention holds throughout: $\Delta$ and $\rho$ are positive for coordination, negative for anti-coordination. Some archived plot panels use $y$ for the same effect; in those panels $y\equiv\Delta$. $\mathrm{BL}$ is the mismatch rate expected if each branch chose modes independently from its own marginal; reporting $\mathrm{obs}$ alone would book marginal differences as coordination, and Section~\ref{sec:calculus} shows the decomposition is exact for arbitrary $K$, of which the robots' $K=2$---their assignment rule being the sign of a projection---is a verbatim specialization. Evidence for $K>2$ is testbed-only (Section~\ref{sec:toy}).

\paragraph{The both-valid conditioning is a selection}
$\Delta$ is computed on survivors, so two estimands must stay apart: \emph{conditional coordination}, defined above on each condition's both-valid sub-population, and the \emph{validity-weighted outcome} over all sampled pairs, which $\Delta$ alone does not recover. A mechanism preferentially invalidating the pairs that would have disagreed raises the first while lowering the second, and the consistency regularizer is a live instance. Every headline effect is therefore reported with its valid-pair rate; Section~\ref{sec:interior} examines this trade-off for the consistency regularizer.

\subsection{Reconciled residual}
\label{sec:residual-def}

\begin{figure}[pos=tbp]
  \centering
  \includegraphics[width=0.92\textwidth]{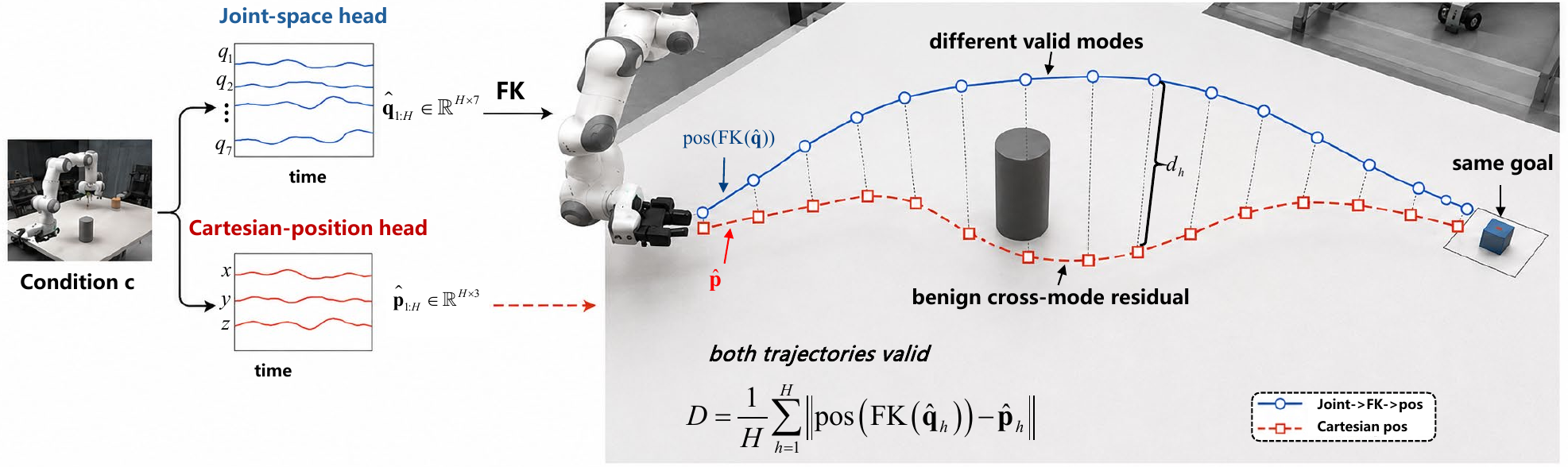}
  \caption{Why two correct branches can disagree. The same condition is decoded into two kinematically equivalent action representations. Independent samples can select different valid routes, so the reconciled forward-kinematics residual becomes large even though both endpoints are correct. The drawing is schematic; measured decompositions appear in Section~\ref{sec:share}.}
  \label{fig:mismatch}
\end{figure}

The cross-branch residual is $D=\operatorname{mean}_h\lVert \mathrm{pos}(\mathrm{FK}(\hat q_h))-\hat p_h\rVert$, in millimetres; Fig.~\ref{fig:mismatch} shows why it can be large even when both branches are correct. The constructed environments use this residual diagnostically, for mechanism and decomposition, rather than as a detector. Detector performance is evaluated only in the external experiment of Section~\ref{sec:liberoplus}. Section~\ref{sec:tautology} further splits the residual into a lateral component and its orthogonal complement. The orientation branch's counterpart is the mean geodesic angle in degrees.

\paragraph{Terminology}
The line $\Delta=C$ is a feasibility boundary, not an empirical Pareto frontier.  We use \emph{modal coverage} for the output-marginal diversity that enters $C$; \emph{valid-pair rate} for the probability that both samples are valid; and \emph{support coverage} for whether test-time latent values are represented during training.  High modal coverage does not imply fidelity to the data distribution: a model can spread mass over the wrong modes and still have large $C$.  Behavioral-prior fidelity is therefore reported separately (Section~\ref{sec:corner-entry}).

\section{The effect-size calculus and the attainable region}
\label{sec:calculus}

This section expresses cross-branch mode agreement on a chance-corrected scale.  The calculus is stated for an arbitrary number of modes $K$; both robot environments use the binary specialization in Section~\ref{sec:mode-def}.  Its relation to classical agreement statistics is summarized in Remark~\ref{rmk:correspondence}.

\paragraph{Notation}
Let $M_q, M_T$ take values in a \emph{common} $K$-element mode set $\{1,\dots,K\}$ plus an absorbing \texttt{invalid}. On the both-valid event write the joint distribution $P_{jk}$ with marginals $q_k$, $p_k$, and
\begin{equation*}
\begin{aligned}
p_o &= \textstyle\sum_k P_{kk}, & \mathrm{obs} &= 1-p_o, & p_e &= \textstyle\sum_k q_k p_k, \\
\mathrm{BL} &= 1-p_e, & \Delta &= \mathrm{BL}-\mathrm{obs} = p_o-p_e, & G(m) &= 1-\textstyle\sum_k m_k^2 ,
\end{aligned}
\end{equation*}
with $G$ the Gini--Simpson diversity. For $K=2$ these reduce verbatim to Section~\ref{sec:mode-def}, with $C_2=2\sqrt{q_L(1-q_L)p_L(1-p_L)}\le1/2$; the factors of 2 decorating binary formulas come from summing over $k=1,2$, and also explain the two factors of 2 in \eqref{eq:simpson} below.

\paragraph{Label matching (a hypothesis for $K>2$)}
$p_o$ depends on how the branches' labels are paired; $C_K$ below does not. The pairing is canonical at $K=2$ and canonical here for all $K$ because both branches are scored by \emph{one} environment-defined partition. Applied to two independently clustered mode sets, the calculus requires fixing a matching first, else $p_o$ is undefined; the unsupervised-token family of Section~\ref{sec:corner} avoids the issue because both branches are fed the \emph{same} symbol, not because its labels align.

\begin{theorem}[Attribution, coverage bound, coverage certificate, tight bound]
\label{thm:region}
Fix a condition $c$.  For part~(i), suppose the branches are conditionally independent given $(c,Z)$; parts~(ii)--(iv) hold for any joint law.  Write $a_k(z)=\Pr(M_q=k\mid M_q\neq I,c,z)$ and $b_k(z)=\Pr(M_T=k\mid M_T\neq I,c,z)$.  Let $\tilde z$ denote the distribution of $Z$ tilted by $w(z)=\Pr(M_q\neq I\mid c,z)\Pr(M_T\neq I\mid c,z)$, which accounts for conditioning on both branches being valid.  Then
\begin{align}
\text{(i)}\quad & \Delta(c) \;=\; \sum_k \operatorname{Cov}_{\tilde z}\!\bigl(a_k(Z),\,b_k(Z)\bigr) \;=\; \operatorname{tr}\operatorname{Cov}_{\tilde z}\!\bigl(a(Z),\,b(Z)\bigr), \label{eq:attrib}\\
\text{(ii)}\quad & \bigl|\Delta(c)\bigr| \;\le\; C_K := \sqrt{G(q)\,G(p)} \;\le\; 1-\tfrac{1}{K}, \label{eq:bound}\\
\text{(iii)}\quad & G(q) \;\ge\; \tfrac{K}{K-1}\,\Delta(c)^2, \quad\text{and likewise } G(p), \label{eq:certificate}\\
\text{(iv)}\quad & \Delta \;\le\; C_{K,\mathrm{tight}} := \sum_k \min(q_k,p_k)-\sum_k q_kp_k \;=\; \mathrm{BL}-\mathrm{TV}(q,p), \label{eq:tight}
\end{align}
with $\mathrm{TV}(q,p)=\tfrac12\sum_k|q_k-p_k|$. Always $C_{K,\mathrm{tight}}\le C_K$, with equality iff $q=p$ or $C_K=0$ (one marginal a point mass; both bounds vanish).
\end{theorem}

The proof in \ref{app:proofs} uses one-hot indicators, the law of total covariance, Cauchy--Schwarz, and a Fr\'echet--Hoeffding coupling.  Equation~\eqref{eq:attrib} is the standard latent-class covariance decomposition \citep{uebersax1990,agresti1992}, adapted to the both-valid population through the tilted distribution $\tilde z$.  The outer bound is Gorodkin's $R_K\le1$ \citep{gorodkin2004}; \eqref{eq:certificate} is its reverse reading; and \eqref{eq:tight} is the fixed-marginal maximum-kappa numerator \citep{umesh1989}, with $\mathrm{BL}-C_{K,\mathrm{tight}}=\mathrm{TV}(q,p)$ corresponding to the multiclass bias index \citep{byrt1993}.

\paragraph{Interpretation}
Equation~\eqref{eq:attrib} measures how the branches' mode propensities co-vary with a shared variable, rather than how small the raw mismatch happens to be.  If neither branch responds to that variable, the effect is zero.  The bound in \eqref{eq:bound} then separates coordination from collapse: a model that always selects one mode has $\mathrm{obs}=0$, but it also has $\mathrm{BL}=C_K=\Delta=0$.  Conversely, a nonzero effect certifies modal diversity through \eqref{eq:certificate}.  For example, the learned-prior token effect $\Delta=0.4710$ implies $G(q)\ge0.4436$ and hence $q_L\in[0.332,0.668]$, consistent with the measured modal entropies.  Because the inequality holds per condition and is preserved by non-negative weighted averaging, run-level points obey the same boundary.

\paragraph{$C_K$ is an outer bound, not the maximum attainable at fixed marginals}
For fixed unequal marginals the tighter limit is $C_{K,\mathrm{tight}}$ in \eqref{eq:tight}: whenever $q\neq p$ and $C_K>0$ we have $\Delta\le C_{K,\mathrm{tight}}<C_K$, so a run whose two branch marginals differ cannot reach its own $C_K$, and the boundary $\Delta=C_K$ is attained only with matching marginals and perfect same-label agreement.  We retain $C_K$ on the horizontal axis because it depends only on the diversity of each branch and therefore preserves the interpretation of the axis as modal capacity.  It also charges marginal disagreement between the branches rather than normalizing it away: for $2\times2$ tables, dividing each of a family of association coefficients by its maximum under the observed margins returns one and the same coefficient \citep{warrens2008similarity,cureton1959}, so a fixed-marginal denominator discards precisely the marginal discrepancy the horizontal axis is meant to carry.  The fixed-marginal bound is reported as a secondary diagnostic in \ref{app:tables}; at the available per-condition sample sizes its plug-in estimate is more sensitive to marginal noise, but it does not change the qualitative occupancy assignments.  Feasible-region readings of an agreement statistic have a classifier-ensemble precedent in kappa--error diagrams, computed there on correct/wrong tables against ground truth \citep{kuncheva2013bound}; Section~\ref{sec:related} details the relation, and the absence of any correctness label from both coordinates here is what makes this region a runtime object.

\paragraph{The range moves with $K$; the classification does not}
The horizontal range is $[0,1-1/K]$ (0.5, 0.75, 0.875 at $K=2,4,8$), so the corner's absolute position depends on $K$, while the four-way classification is fixed by two dimensionless coordinates: the normalized effect $\rho=\Delta/C_K$ and relative capacity $C_K/(1-1/K)$; cross-$K$ comparisons use this pair (Section~\ref{sec:toy-relabel}). The lower Cauchy--Schwarz side $\Delta\ge-C_K$ is not tight for $K\ge3$ (the folklore range $[-1,1]$ for multiclass correlation is tight only at $K=2$, \citealp{jurman2012}); we use the upper side only, and read negative effects as anti-coordination without interpreting their distance to $-C_K$.

\paragraph{What the effect does not measure}
$\Delta=\langle I_K,\,P-qp^{\!\top}\rangle$ is a \emph{signed linear functional of the diagonal} of the dependence matrix $P-qp^{\!\top}$, not a full dependence measure, and the distinction is not cosmetic once $K\ge3$. Two ways to have $\Delta=0$ with the branches dependent: diagonal deviations of opposite sign that cancel in the trace ($P_{11}>q_1p_1$ offset by $P_{22}<q_2p_2$), and dependence carried entirely off the diagonal (branch~1 picking mode~1 makes branch~2 prefer mode~2 over mode~3, with every diagonal cell at its independent value). In numbers: at $K=3$ with uniform marginals, let branch~2's label be a deterministic permutation of branch~1's fixing one label and swapping the other two. Then $p_o=p_e=1/3$, so $\Delta=0$ and $\rho=0$ at $C_3=C_{3,\mathrm{tight}}=2/3$---the point lands on the floor at two-thirds of maximum capacity while the branches are perfectly dependent, $I(M_q;M_T)=\ln3$. A statistic returning zero there measures same-label agreement, not coordination in general. Only at $K=2$ does the ambiguity close---a permutation fixing one of two labels is the identity, so no such example exists---and there $\Delta=2(P_{LL}P_{RR}-P_{LR}P_{RL})$ is twice the determinant of the $2\times2$ table, so $\Delta=0$ holds \emph{iff} the branches are independent. The two robotic environments and the external benchmark are all binary and therefore not exposed to the gap; every $K>2$ reading of Section~\ref{sec:toy} is. What $\Delta$ measures throughout is the same-label direction: how much more often than chance the two branches land on \emph{the same} mode. That is the quantity a cross-representation reconciliation actually consumes---a residual is small when the branches picked the same mode, and no smaller when they picked two different modes in a correlated way---so the projection is the operationally right one here, but at $K\ge3$ it must not be read as ``the branches are independent'' when it returns zero. The implication runs one way: independence forces $\Delta=0$, not the converse.

\begin{remark}[The Simpson term in pooled reporting]
\label{rmk:simpson}
Pooling across conditions gives
\begin{equation}
\label{eq:simpson}
\mathrm{BL}_{\mathrm{pooled}}-\mathbb{E}_c[\mathrm{obs}(c)]
= \sum_k \mathbb{E}_c\!\left[\operatorname{Cov}_{\tilde z}(a_k,b_k\mid c)\right]
+\sum_k \operatorname{Cov}_c\!\left(q_k(c),p_k(c)\right),
\end{equation}
under equal weighting; both-valid weighting preserves the same decomposition.  The first term is within-condition coordination, while the second is co-movement of the branch marginals across conditions.  A pooled baseline therefore attributes condition-level changes to the mechanism.  Section~\ref{sec:floor} contains an observed example, and the posterior runs show the related consequence that a pooled normalized effect can slightly exceed a pooled-marginal bound even though every exact per-condition table satisfies the bound.  Similar aggregation effects arise in temporal uncertainty summaries \citep{tang2026shifting}.
\end{remark}

\begin{remark}[Conditional independence and disagreement monitors]
\label{rmk:scope}
For two conditionally independent draws, the observed mismatch equals the independent-choice baseline and hence $\Delta=0$.  This is the second-moment counterpart of the data-processing statement
\[
I(M_q;M_T\mid c)\le
\min\{I(M_q;Z\mid c),I(M_T;Z\mid c)\}
\]
for conditionally independent local channels \citep{wyner1975}.  If neither channel responds to the common variable, their conditional dependence vanishes.

The role of Lemma~\ref{lem:zblind} is to connect this familiar criterion to CFM training: an auxiliary latent may be supplied to both velocity fields yet be absent from the population-optimal predictor.  The effect $\Delta$ makes the implication experimentally useful because it has a per-condition capacity and a measurable sampling floor.  In the present protocol, effects within approximately $\pm0.003$ are indistinguishable from the structural zero at 75 conditions and 1280 pairs per condition (Section~\ref{sec:monitors}).  At $K\ge3$, the converse does not hold: $\Delta=0$ need not imply full independence because the statistic measures only same-label agreement.
\end{remark}

\begin{remark}[Correspondence with chance-corrected agreement statistics]
\label{rmk:correspondence}
Treating the branches as two raters gives the exact dictionary in Table~\ref{tab:correspondence}.  The chance model uses each branch's own marginal, as in Cohen's kappa \citep{cohen1960}.

\begin{table}[t]
\centering
\caption{Correspondence with chance-corrected agreement statistics (arbitrary $K$).}
\label{tab:correspondence}
\footnotesize
\begin{tabular}{ll}
\toprule
This paper & Agreement-statistics object \\
\midrule
$\mathrm{obs}$ & observed disagreement $1-p_o$ \\
$\mathrm{BL}=1-\sum_k q_kp_k$ & chance disagreement, own-marginal convention \citep{cohen1960} \\
$\Delta=\mathrm{BL}-\mathrm{obs}$ & numerator of Cohen's kappa; at $K{=}2$, $2(P_{LL}P_{RR}-P_{LR}P_{RL})$ \\
$\Delta/\mathrm{BL}$ & Cohen's kappa \\
$C_K=\sqrt{G(q)G(p)}$ & denominator of Gorodkin's $R_K$ \citep{gorodkin2004} \\
$\rho=\Delta/C_K$ & $R_K$; at $K{=}2$, the phi coefficient/MCC \citep{matthews1975} \\
Theorem~\ref{thm:region}(ii), (iii) & $R_K\le1$ and its rearrangement \citep{gorodkin2004} \\
$C_{K,\mathrm{tight}}$ & fixed-marginal maximum-kappa numerator \citep{cohen1960,umesh1989} \\
$\mathrm{BL}-C_{K,\mathrm{tight}}=\mathrm{TV}(q,p)$ & multiclass bias index \citep{byrt1993} \\
$G(m)$ & Gini impurity/Simpson diversity; total one-hot variance \\
\bottomrule
\end{tabular}
\end{table}

Only the attribution identity requires conditional independence; the two bounds are properties of any non-negative contingency table.  The classical ratios also explain why the two-dimensional presentation is useful.  A collapsed posterior can have kappa and $R_K$ equal to one while both their numerator and denominator approach zero.  Plotting $\Delta$ against $C_K$ places that case near the origin rather than on an undifferentiated ``perfect agreement'' line.  The underlying dependence on the marginals is the one behind the high-agreement/low-kappa paradox \citep{feinstein1990,cicchetti1990}, with the sign reversed: there an unbalanced margin drives kappa down while observed agreement is high, here a degenerate margin drives the ratio to one while modal diversity vanishes.  Reporting the numerator against its capacity is the two-dimensional counterpart of reporting prevalence and bias next to kappa \citep{byrt1993}.  Conversely, $C_K$ is maximized by uniform generated marginals and therefore says nothing about fidelity to a skewed reference distribution; Section~\ref{sec:toy-skew} reports that discrepancy separately.

These are one family rather than several: kappa, Fleiss's kappa, informedness, markedness and the MCC share a numerator and differ in how their denominators average prevalence and bias \citep{powers2012kappa}.  What the vertical axis carries is that shared numerator, left undivided.  The contribution here is the use of these statistics for cross-representation mode choices in generative action policies, including the both-valid conditioning and the interpretation of the bound as a modal-capacity certificate.  An information-theoretic normalization could be developed in parallel; the second-moment form is retained because it yields a simple per-condition boundary and connects directly to kappa, $R_K$, and the MCC literature \citep{thompson1988,simwright2005,vinh2010}.
\end{remark}

\subsection{Equality conditions and empirical coverage}
\label{sec:tight}

Bound~\eqref{eq:bound} combines Cauchy--Schwarz with $G(m)\le1-1/K$.  The upper boundary $\Delta=C_K$ requires perfect same-label agreement with matching branch marginals; the global corner $\Delta=1-1/K$ additionally requires uniform marginals.  At $K=2$, the oracle and window-token families approach both conditions simultaneously.

\paragraph{Empirical coverage of the theorem}
The $K$-ary form has support at $K\in\{2,3,4,8\}$, all of it from the non-robotic testbed (Section~\ref{sec:toy}): 186 runs $\times$ 75 conditions, 135 threshold/NFE re-evaluations, and 153 skewed-prior runs, with zero violations of \eqref{eq:bound} and \eqref{eq:tight} throughout. The robotic assignment rule is a projection sign, so robotic evidence is confined to $K=2$; $K>2$ in those environments requires replacing the rule, not a parameter.

\subsection{An auxiliary latent outside the target's construction is pinned to the floor}
\label{sec:lemma}

The shared-intent construction can be characterized at the population optimum. The relevant distinction is whether the shared variable participates in constructing the CFM interpolation state or regression target, or is merely appended to the velocity-field input.

\paragraph{Notation}
For branch $i$, let $U=x_1^i-x_0^i$ be the regression target, $S=(x_t^i,t,E(c))$ the remaining velocity-field input, and $Z$ the auxiliary latent. The \emph{training tuple} comprises the random variables used to construct $(S,U)$: the condition, data endpoint, source noise, and interpolation time. Source noise is therefore part of the tuple; an auxiliary latent is not.

\paragraph{Training assumptions}
The population statement needs two conditions:
\begin{enumerate}\itemsep0pt\parskip0pt
\item[(A1)] \emph{Mean irrelevance}: $\mathbb{E}[U\mid S,Z]=\mathbb{E}[U\mid S]$ almost surely.
\item[(A2)] \emph{Bayes attainability}: $\mathbb{E}\lVert U\rVert^2<\infty$, the population objective is the squared CFM risk in \eqref{eq:cfm}, and the function class attains the conditional-mean predictor.
\end{enumerate}
In the continuous-prior experiments, (A1) follows from drawing $Z$ independently of the full training tuple $(c,x_0^1,x_0^2,x_1,t)$. The word \emph{auxiliary} refers to this design: $Z$ is appended to the velocity-field input but does not generate source noise, choose the interpolation path, partition the data, or reweight the loss. Shared source noise and data-derived tokens are therefore different cases even before their empirical effects are considered.

\begin{lemma}[$Z$-blindness of the CFM optimum]
\label{lem:zblind}
Under (A1)--(A2), any population minimizer of \eqref{eq:cfm} satisfies, almost surely under the training input distribution,
\begin{equation}
\label{eq:zblind}
v_i^*(x,t,c,z)
=\mathbb{E}\bigl[U\mid S=(x,t,c),Z=z\bigr]
=\mathbb{E}\bigl[U\mid S=(x,t,c)\bigr],
\end{equation}
so $v_i^*$ is independent of $z$ on that support.
\end{lemma}

\begin{proof}
Under squared loss, the Bayes predictor is the conditional expectation. Assumption (A2) makes that predictor attainable, and (A1) removes the dependence on $Z$. The standard conditional-expectation identity used when (A1) follows from joint independence is stated in \ref{app:proofs}.
\end{proof}

The conclusion is an almost-sure statement on the training input distribution. Applying it during sampling therefore requires the sampled ODE states to remain on a support where the same equality holds.

\paragraph{From a blind velocity field to a measured zero}
The population statement becomes a sampling statement only under the following additional conditions:
\begin{enumerate}\itemsep0pt\parskip0pt
\item[(A3)] \emph{Inference support}: \eqref{eq:zblind} holds along the ODE states visited at inference, and $Z$ enters neither the initial condition nor the solver.
\item[(A4)] \emph{Branch-local randomness}: conditional on $c$, source noise, solver state, dropout, posterior draws, and randomized post-processing are private to each branch; neither branch reads the other's state, output, or random state.
\item[(A5)] \emph{Branch-local scoring}: mode and validity are functions of one branch's trajectory, and $\mathrm{obs}$ and $\mathrm{BL}$ use the same non-empty both-valid population and weights.
\end{enumerate}
Training-seed variation is treated separately from this per-checkpoint statement.

\begin{corollary}
\label{cor:zblind}
Under (A1)--(A2):

\emph{(a) Excess-risk decomposition.} For any measurable $v(S,Z)$, define
$\bar v(S)=\mathbb{E}[v(S,Z)\mid S]$. Then
\begin{equation}
\label{eq:excess}
\mathcal{L}(v)-\mathcal{L}(v^*)
=\mathbb{E}\!\left[\operatorname{tr}\operatorname{Var}\!\left(v(S,Z)\mid S\right)\right]
+\mathbb{E}\bigl\lVert\bar v(S)-v^*(S)\bigr\rVert^2.
\end{equation}

\emph{(b) Branch conditional independence.} If (A3)--(A5) also hold, then $M_q\perp M_T\mid c$. Consequently, $\mathrm{obs}(c)=\mathrm{BL}(c)$ and $\Delta(c)=0$ for every eligible condition.
\end{corollary}

\begin{proof}
For part (a), the squared-loss projection identity gives
$\mathcal{L}(v)-\mathcal{L}(v^*)=\mathbb{E}\lVert v(S,Z)-v^*(S)\rVert^2$. Conditioning on $S$ and applying the conditional bias--variance decomposition yields \eqref{eq:excess}; no independence between $S$ and $Z$ is required. For part (b), Lemma~\ref{lem:zblind} removes $Z$ from both sampled flows along the inference paths. Branch-local random sources then make the two trajectories conditionally independent, and branch-local mode and validity maps preserve that independence on the both-valid population.
\end{proof}

Equation~\eqref{eq:excess} shows that variation with $Z$ at fixed $S$ contributes excess risk and vanishes at the population optimum. It does not prescribe the optimization trajectory. Empirically, however, the saved checkpoints show a monotone decline in $Z$-sensitivity from $0.398$ to $0.00153$ over 50{,}000 steps, a factor of about 260 (Section~\ref{sec:floor}). A shared variable can acquire a role by violating mean irrelevance---for example, through a data-dependent posterior or a trajectory partition---or by adding an objective term that rewards cross-branch agreement. These two routes organize the mechanism study in Section~\ref{sec:atlas}.

\paragraph{Scope}
The result is a population squared-loss statement. Finite optimization and restricted function classes can leave residual $Z$-dependence. A consistency regularizer changes the objective; a data-dependent posterior or partition can violate mean irrelevance; shared source noise belongs to the interpolation construction rather than to the auxiliary case; and shared inference-time randomness invalidates the branch-local sampling argument. Parameter sharing by itself is neither required nor sufficient for dependence: a common deterministic trunk can coexist with private sampling, whereas disjoint parameters do not rule out a shared solver or post-processing step. Partially shared generative states are outside the experiments considered here.

\section{Experimental setup}
\label{sec:setup}

\subsection{Two environments}
\label{sec:envs}

\textbf{BimodalBypass (v2).} RoboSuite 1.4.1, Panda 7-DoF: the Lift task plus a central wall and a goal region (geometry in \ref{app:prereg}). Both bypass paths are scripted, and an initial state enters the dataset only if both succeed from it: 500 pairs / 1000 trajectories, split 350/75/75 by pair, 127{,}707 training windows. Within the $H{=}16$ branching window the modes are $143.68\pm0.66$\,mm apart (CV 0.0046; non-lateral component 0.77\,mm); full trajectories share their endpoint (0.00025\,m). Threshold $\tau=0.11$.

\textbf{RotatingBypass (v3).} The v2 construction has a geometric shortcut: the mode label is readable from one fixed coordinate. Scoring 39 scalar features by pooled AUROC, v2's top four are all summaries of $\mathrm{eef}_y$, each at exactly 1.0000. v3 removes this shortcut: each condition draws an azimuth $\theta\sim U(0,\pi)$ and radius $R\sim U(0.170,0.215)$\,m; branch A uses $\psi_A=\theta$, branch B $\psi_B=(\theta+\pi/2)\bmod\pi$; offsets act on $(y,z-0.955)$ with an anisotropic vertical scaling of 0.68; the transport yaw is $0.45\cos\psi$; the obstacle becomes a suspended beam. 240 pairs from 240 attempts, 0 rejections, split 168/36/36. The strongest single-coordinate AUROC falls to \textbf{0.5586} (below the prespecified 0.75 threshold), and mode separation becomes $162.43\pm16.50$\,mm (CV 0.1016), no longer a constant---with $\theta\sim U(0,\pi)$ each branch's lateral sign is 50/50, so pooled AUROC is $1/2$ by construction.

Two properties of v3 affect interpretation. There is no global ``left/right'': a condition's branches are defined by its own azimuths, and which member of the unordered pair the diagnostics call A is equiprobable---a symmetry of the generative process, not relabeling. What was designed away is the fixed separation direction and constant spacing; the separation direction still does not sweep the full circle, staying within $\pm34^\circ$ of the lateral axis ($-33.8^\circ$ to $+33.6^\circ$, sd $19.2^\circ$; lateral share of magnitude $\ge83.1\%$, mean 94.5\%), and removing this remaining bias would require a different task geometry (Section~\ref{sec:tautology}).

\emph{Comparability.} The same nominal $\tau$ is not the same stringency (v3's decision quantity is whitened, magnifying the vertical by $1/0.68=1.47$): $\tau=0.11$ in v3 gives validity 0.999, and raw comparison would erase v2's coverage cost. We therefore match by coverage: thresholds at which the independent family has equal validity---$\tau=0.11$ in v2 (0.773), $\tau=0.18$ in v3 (0.738). All v3 coordination and $\rho$ numbers use this matched point; the stratified AUROC of Section~\ref{sec:tautology} uses the reference $\tau=0.11$ stated there; the two operating points are never mixed. Threshold sensitivity: \ref{app:sensitivity}.

\subsection{Models, training, inventory}
\label{sec:models}

The condition encoder is 256-dimensional. Each velocity field is a four-layer, 256-unit MLP, and the two action branches share no velocity-field parameters. Models are trained for 50{,}000 steps with batch size 128 and learning rate $3\times10^{-4}$. The 16-dimensional shared latent $z_s$ is encoded to 32 dimensions, concatenated with the condition representation, and supplied to both branches at every ODE step; nonzero input gradients confirm that both fields can respond to it. The pose model adds a third, separately parameterized velocity field over the continuous 6D rotation representation.

The main mechanism study contains 68 training runs in the two robot environments (49 in v2 and 19 in v3). Table~\ref{tab:main} reports the principal v2 families, the remaining grid cells appear in the mechanism-specific analyses, and the v3 runs are summarized in Table~\ref{tab:v3} and Section~\ref{sec:third-rep}. The complete inventory, including the additional token and ensemble runs, is provided in \ref{app:prereg}.

\subsection{Frozen evaluation protocol}
\label{sec:gate}

All mechanism families use the same evaluation protocol; family-specific evaluators differ only in how the shared input is drawn. Each v2 run is evaluated on 75 paired initial-state conditions, with 64 samples and a 20-draw Monte Carlo bank per condition (96{,}000 samples); v3 uses 36 conditions and 46{,}080 samples. We integrate with 20 Euler steps. A shared variable is drawn once per sample, held fixed along the ODE, and supplied to both branches. Bootstrap intervals resample paired conditions with 2000 draws and seed 20260802.

\paragraph{Aggregation order}
A pooled baseline can introduce an apparent effect through condition-level marginal co-movement (Remark~\ref{rmk:simpson}), so the order is part of the protocol. From each condition's both-valid contingency counts we form $\mathrm{obs}_c,\mathrm{BL}_c,\Delta_c,C_c$ using that condition's own marginals; the run value is their $N_{\mathrm{BV}}$-weighted mean, with $\rho$ the ratio of those weighted means and never the mean of per-condition ratios, which would weight near-zero-capacity conditions like the rest; the family value is the mean over seeds, and where a ratio of means is quoted instead the aggregation convention is stated at the point of use (Table~\ref{tab:v3}). No stratum above the condition---seed, checkpoint, environment---is pooled before the baseline is formed, and every bootstrap replicate re-executes the same order, recomputing marginals, $\Delta$ and $C$ rather than holding a denominator fixed. With only 36 conditions the pooled-baseline small-sample bias reaches 0.09, so all v3 results use per-condition baselines (in v2 the two differ by $\le0.004$); Section~\ref{sec:floor} shows a case where this decides the conclusion.

\paragraph{Sample sizes}
The median number of both-valid pairs per condition is 1042 in v2 and 966 in v3; the respective 5th--95th percentile ranges are 581--1279 and 130--1274. Under the $N_{\mathrm{BV}}$ weighting used here, finite-pair plug-in bias shrinks a nonzero $\Delta$ by 0.10\% for the median run and by at most 0.43\%, below the reporting precision.

\paragraph{Marginal accounting}
Exact evaluation requires the mode marginals, mismatch rate, and validity event to be computed from the same both-valid contingency table. Of the 68 trained runs, 67 carry per-condition coordinates (\ref{app:prereg} inventories both counts); seventeen of these store the required both-valid marginals, and the remaining 50 store head-valid marginals only. The 49 runs of the preregistered seed expansion (Table~\ref{tab:seeds}) store them natively. Re-estimating the 17 recoverable runs removes every apparent excursion above $C$ and $C_{K,\mathrm{tight}}$ while leaving the mechanism ordering unchanged (Section~\ref{sec:reestimation}). Each supplementary row states which marginal convention is available.

\subsection{Equivalence-testing protocol}
\label{sec:tost}

The population result predicts an exact zero, whereas finite models can establish only practical equivalence. We therefore apply two one-sided tests (TOST; \citealp{schuirmann1987,lakens2017}) to the normalized effect $\rho=\Delta/C$. Testing $\rho$ rather than the raw effect prevents a collapsed, near-zero capacity from making any raw effect appear negligible; where $C$ itself approaches zero, the equivalence test is correspondingly uninformative.

The structural-zero controls give $\max(|\hat\rho|+1.645\,\mathrm{se})=1.44\%$, while the weakest mechanism treated as nonzero has $\rho=5.93\%$. We therefore use the prespecified margin $\delta=2\%$ and report the margin-free width $\hat d=\max(|\mathrm{CI}_{\mathrm{lo}}|,|\mathrm{CI}_{\mathrm{hi}}|)$ alongside it. Decisions use 90\% percentile bootstrap intervals, as required for two one-sided tests at $\alpha=0.05$; Fig.~\ref{fig:forest} displays the corresponding 95\% descriptive intervals. Both are computed from the same 2000 condition-level resamples. Sensitivity to wider margins is reported in the supplement.

\subsection{Exact-accounting re-estimation and summary of the main results}
\label{sec:reestimation}

Table~\ref{tab:caliber} recomputes every estimand from a single both-valid contingency table for the 17 runs whose artifacts retain the necessary marginals.

\begin{table}[t]
\centering
\caption{Exact-accounting re-estimation: head-valid versus both-valid marginals from the same contingency tables, for the 17 of 67 runs whose artifacts store both. Run-level $N_{\mathrm{BV}}$-weighted mean of per-condition quantities, then mean over seeds. Last column: per-condition excursions above $C$ and above $C_{K,\mathrm{tight}}$, pooled over the family.}
\label{tab:caliber}
\footnotesize
\setlength{\tabcolsep}{4pt}
\begin{tabular}{@{}l c rrr rrr c@{}}
\toprule
& & \multicolumn{3}{c}{head-valid marginals} & \multicolumn{3}{c}{exact both-valid marginals} & excursions \\
\cmidrule(lr){3-5}\cmidrule(lr){6-8}
Family & $n$ & $\Delta$ & $C$ & $\rho$ & $\Delta$ & $C$ & $\rho$ & HV $\to$ BV \\
\midrule
C\ \ $z_s$ dim-0 anchor & 3 & $+0.0020$ & 0.4858 & $+0.4\%$ & $+0.0020$ & 0.4857 & $+0.4\%$ & 0/0 $\to$ 0/0 \\
D\ \ continuous $z_s$ prior & 3 & $-0.0006$ & 0.4708 & $-0.1\%$ & $-0.0009$ & 0.4708 & $-0.2\%$ & 0/0 $\to$ 0/0 \\
E\ \ posterior (3 KL scales) & 5 & $+0.0260$ & 0.0256 & $+101.7\%$ & $+0.0223$ & 0.0223 & $+100.0\%$ & 301/303 $\to$ 0/0 \\
H\ \ episode token, uniform & 3 & $+0.4481$ & 0.4885 & $+91.7\%$ & $+0.4339$ & 0.4766 & $+91.1\%$ & 3/107 $\to$ 0/0 \\
I\ \ episode token, learned prior & 3 & $+0.4876$ & 0.4907 & $+99.4\%$ & $+0.4710$ & 0.4810 & $+97.9\%$ & 68/217 $\to$ 0/0 \\
\bottomrule
\end{tabular}
\end{table}

Exact both-valid accounting removes all apparent bound excursions in these 1275 conditions. The episode-token effects remain near the boundary ($\rho=91.1\%$ and $97.9\%$), the floor estimates are unchanged to reporting precision, and the posterior family becomes exactly degenerate: $\Delta=C$ at capacities between $0.0001$ and $0.0625$. The 50 runs lacking both-valid marginals remain flagged as approximate; comparison with the recoverable runs suggests that their normalized effects may be slightly high, without changing any occupancy classification.

\paragraph{Evidence status of the main claims}
Table~\ref{tab:claims} records not only what was observed but how each statement is licensed. The labels distinguish algebraic identities, empirical support, preregistered tests, inconclusive seed-level comparisons, and analyses designed after an outcome was known.

\begin{table}[t]
\centering
\caption{Evidence status of the headline claims. A status applies only within the regime stated in the evidence column.}
\label{tab:claims}
\scriptsize
\setlength{\tabcolsep}{3.5pt}
\begin{tabular}{@{}p{0.28\textwidth} p{0.36\textwidth} p{0.29\textwidth}@{}}
\toprule
Claim & Evidence & Evidential status \\
\midrule
$|\Delta(c)|\le C_K(c)$ & Algebraic bound; no violation in 1275 exact-accounting robot conditions or in the testbed evaluations & \textbf{Established}: identity, not an empirical finding \\
Auxiliary latent outside the CFM construction & Two robot environments, three seeds each; TOST at $\delta=2\%$; $\hat d=1.16/1.66\%$ & \textbf{Supported as finite-model equivalence}; the exact zero is the population prediction under (A1)--(A2), (A3)--(A5) \\
Shared source noise & v2 eight-seed mean $-0.5\%$ (sd $3.4$, four signs each way); testbed factorial at ten seeds per cell: interaction range 90.1 pp, permutation $p=10^{-4}$ & \textbf{Map dependence preregistered-and-confirmed}; seed-unstable on the robots; quote with seed and representation map \\
Endpoint consistency coordinates without a data-derived partition & $\rho=50.3\%$ in v2 and $63.2\%$ in v3; valid-pair rate $0.768\to0.525$ in v2 & \textbf{Supported in the tested robot regimes}, with the validity cost part of the result \\
Partition tokens reach the robot-side high-capacity corner & Robot-side tokens reach $\rho=91.1$--$100\%$; an amortized posterior reaches the same region on the balanced testbed & \textbf{Supported as a robust sufficient robot-side route}, not as a necessary condition in general \\
Benign multimodality raises residual false alarms & Residual surcharge $+1.57$ pp, CI $[+0.70,+2.46]$; observation-side control $-0.01$ pp & \textbf{Preregistered and supported} on this benchmark and operating point \\
Coordination shrinks the multimodal false-alarm surcharge & Coordinated minus uncoordinated gap difference $-0.03$ pp, CI $[-1.00,+0.86]$ & \textbf{Preregistered; criterion not met} in this regime; the interval leaves both no effect and a partial reduction open \\
Coordination improves failure detection & Three-seed difference $+0.0052\pm0.0288$; the ranking reverses across seeds & \textbf{Inconclusive at seed level}; neither refuted nor shown equivalent \\
Low capacity and token entropy explain the external null & $C=0.0055$; token entropy $0.011$--$0.013$ nats; both measured on the frozen traces after the monitor result & \textbf{Post hoc explanatory analysis}; measured and falsifiable, but not a preregistered success and not sufficient to identify the binding bottleneck \\
\bottomrule
\end{tabular}
\end{table}

The last three rows are deliberately graded apart. A prediction fixed before scoring is graded against its registered criterion, and a comparison whose interval still spans meaningful effects in the predicted direction is reported as criterion-not-met rather than refuted; a contrast smaller than its between-seed spread is inconclusive; and a later capacity analysis may explain the regime without becoming confirmatory evidence for the original prediction.

\section{Four occupancies of the attainable region}
\label{sec:atlas}

\subsection{Main tables and the occupancy classification}
\label{sec:main-tables}

Table~\ref{tab:main} summarizes the mechanism families under the frozen BimodalBypass protocol. Table~\ref{tab:v3} reports the corresponding subset in RotatingBypass, Fig.~\ref{fig:atlas} presents the v2 results as a heat map, and Fig.~\ref{fig:region} places all 67 robot-side runs in the $(C,\Delta)$ plane.

\begin{figure}[pos=tbp]
  \centering
  \includegraphics[width=\textwidth]{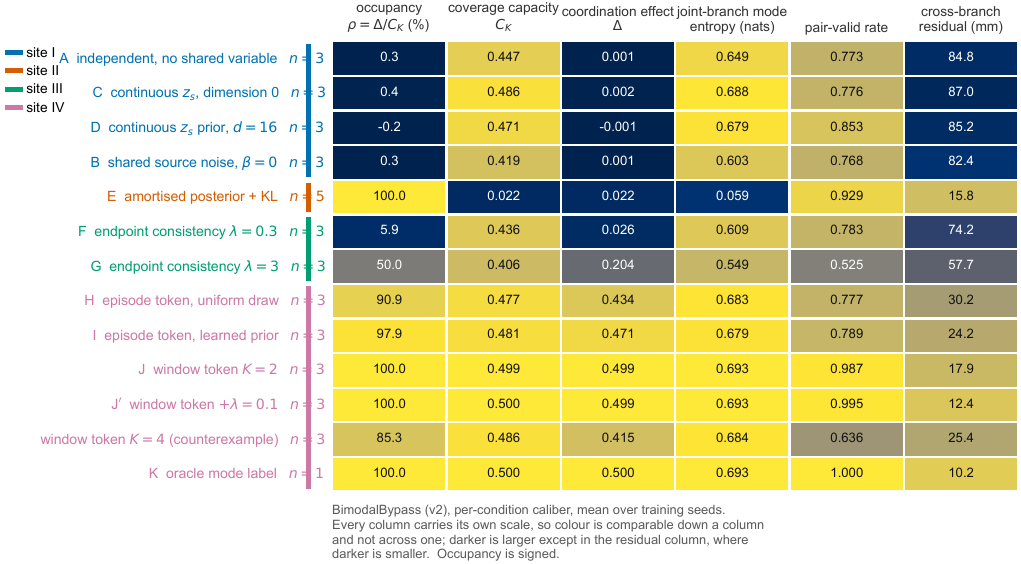}
  \caption{Mechanism summary for BimodalBypass, averaged over training seeds. The first two columns give capacity and coordination effect; the remaining columns report modal entropy, valid-pair rate, and cross-branch residual. Each column has its own scale. The near-origin posterior family illustrates why the normalized effect must be read with capacity: $\rho=100\%$ accompanies $C=0.022$.}
  \label{fig:atlas}
\end{figure}

\begin{table}[t]
\centering
\caption{Mechanism study in the BimodalBypass (v2) robot environment, using per-condition estimates; $n$ is the number of training runs and $K=2$. Sites I--IV denote the floor, near-origin boundary, interior, and high-capacity corner. Families H and I use exact both-valid marginals; the remaining rows use the stored convention stated in the supplement.}
\label{tab:main}
\scriptsize
\setlength{\tabcolsep}{2.6pt}
\begin{tabular}{@{}ll l r rr r@{\hskip 5pt} l r r r c@{}}
\toprule
\# & Family & Shared & $n$ & $\mathrm{obs}$ & $\mathrm{BL}$ & $\Delta$ & $C$ & $\rho$ & Valid & Res.\ (mm) & Site \\
\midrule
A & independent, $\beta1\,\lambda0$ & --- & 3 & 0.4826 & 0.4839 & $+0.0014$ & 0.4474 & $+0.3\%$ & 0.7729 & 84.79 & I \\
B & shared source noise, $\beta0\,\lambda0$ & noise & 3 & 0.4362 & 0.4369 & $+0.0007$ & 0.4185 & $+0.2\%$ & 0.7679 & 82.42 & I$^{\ddagger}$ \\
C & continuous $z_s$, dim-0 anchor & --- & 3 & 0.4921 & 0.4941 & $+0.0020$ & 0.4857 & $+0.4\%$ & 0.7758 & 87.02 & I \\
D & continuous $z_s$ prior, $d{=}16$ & $z{\sim}\mathcal{N}$ & 3 & 0.4843 & 0.4834 & $-0.0009$ & 0.4708 & $-0.2\%$ & 0.8526 & 85.18 & I \\
E & continuous $z_s$ posterior (3 KL) & $z{\sim}q$ & 5 & 0.0000 & $=C$ & $+(0.0001$--$0.0625)$ & 0.0001--0.0625 & $+100.0\%$ & 0.9292 & 15.78 & II \\
F & endpoint consistency $\lambda0.3$ (+B) & loss & 3 & 0.4235 & 0.4492 & $+0.0257$ & 0.4363 & $+5.9\%$ & 0.7830 & 74.17 & III \\
--- & endpoint consistency $\lambda1$ ($\beta0$) & loss & 1 & 0.4107 & 0.4902 & $+0.0795$ & 0.4663 & $+17.0\%$ & 0.4975 & 73.48 & III \\
G & endpoint consistency $\lambda3$ (+B) & loss & 3 & 0.2379 & 0.4420 & $+0.2042$ & 0.4061 & $+50.3\%$ & 0.5254 & 57.66 & III \\
--- & endpoint consistency $\lambda3$, $\beta0.1$ & loss & 1 & 0.2921 & 0.4385 & $+0.1464$ & 0.4330 & $+33.8\%$ & 0.2370 & 69.57 & III \\
--- & endpoint consistency $\lambda3$, $\beta1$ & loss & 1 & 0.5352 & 0.5396 & $+0.0044$ & 0.4176 & $+1.0\%$ & 0.1814 & 75.42 & I \\
H & episode token $K2$, uniform draw & token & 3 & 0.0433 & 0.4772 & $+0.4339$ & 0.4766 & $+91.1\%$ & 0.7769 & 30.16 & IV \\
I & episode token $K2$, learned prior & token & 3 & 0.0101 & 0.4811 & $+0.4710$ & 0.4810 & $+97.9\%$ & 0.7894 & 24.22 & IV \\
J & unsup.\ window token $K2$, uniform & token & 3 & 0.0002 & 0.4996 & $+0.4994$ & 0.4994 & $+100.0\%$ & 0.9871 & 17.87 & IV \\
J$'$ & window token $K2$ + $\lambda0.1$ & token & 3 & 0.0001 & 0.4996 & $+0.4995$ & 0.4996 & $+100.0\%$ & 0.9954 & 12.39 & IV \\
--- & window token $K4$ (one inactive token) & token & 3 & 0.0730 & 0.4879 & $+0.4149$ & 0.4863 & $+85.3\%$ & 0.6356 & 25.40 & IV$^{\ast}$ \\
K & oracle mode label & label & 1 & 0.0000 & 0.4996 & $+0.4996$ & 0.4996 & $+100.0\%$ & 0.9999 & 10.18 & IV \\
\bottomrule
\end{tabular}
\end{table}

\emph{Configuration note.} The table is a mechanism inventory rather than a single-variable ablation. A/C/D/E use $\beta{=}1,\lambda{=}0$; B/F/G use shared source noise with the stated $\lambda$; H/I include $\lambda{=}0.1$ and token dropout; J uses $\lambda{=}0$; J$'$ adds $\lambda{=}0.1$; and K includes weak source-noise sharing and consistency. The separate $\lambda$ sweep shows that these auxiliary settings contribute at most $0.033$ to $\Delta$, well below the token effects of $0.43$--$0.50$.

\begin{table}[t]
\centering
\caption{RotatingBypass (v3) results at the coverage-matched threshold $\tau=0.18$. Values are per-condition estimates averaged over seeds. The $\rho$ column averages seed-level ratios; $^{\dagger}$ denotes a head-valid estimate slightly above one.}
\label{tab:v3}
\scriptsize
\setlength{\tabcolsep}{3.5pt}
\begin{tabular}{@{}l r rr l rrrr@{}}
\toprule
Family & $n$ & $\mathrm{obs}$ & $\mathrm{BL}$ & $\Delta$ (per seed) & $C$ & $\rho$ & Valid & Resid.\ (mm) \\
\midrule
A independent & 3 & 0.4811 & 0.4811 & $0.0000$ ($-0.0005/{-}0.0016/{+}0.0020$) & 0.4700 & $0.0\%$ & 0.7382 & 98.5 \\
B shared source noise & 3 & 0.4567 & 0.4652 & $+0.0085$ ($+0.0079/{+}0.0103/{+}0.0074$) & 0.4555 & $+1.9\%$ & 0.7501 & 95.7 \\
D continuous $z_s$ prior & 3 & 0.4723 & 0.4748 & $+0.0025$ ($+0.0019/{+}0.0019/{+}0.0037$) & 0.4636 & $+0.5\%$ & 0.6474 & 96.9 \\
G endpoint consistency $\lambda{=}3$ & 3 & 0.1709 & 0.4458 & $+0.2748$ ($+0.3097/{+}0.2090/{+}0.3057$) & 0.4319 & $+63.2\%$ & 0.4449 & 60.8 \\
H episode token, uniform draw & 2 & 0.0784 & 0.4708 & $+0.3923$ ($+0.4232/{+}0.3614$) & 0.4661 & $+84.1\%$ (88.6/79.6) & 0.7056 & 38.9 \\
K oracle & 1 & 0.0000 & 0.4994 & $+0.4994$ & 0.4964 & $+100.0\%^{\dagger}$ & 0.8704 & 13.79 \\
\bottomrule
\end{tabular}
\end{table}

v3 does not cover the posterior family, the window tokens, the full grid, or the learned-prior family; the boundary-at-origin occupancy in particular has no v3 replication.

Every row occupies one of four characteristic regions. \emph{I, the floor} (A, C, D and the $\beta=1$ row) combines high capacity with $|\rho|\le1.0\%$; family B lies nearby for the map-dependent reason developed in Section~\ref{sec:basin}. \emph{II, the boundary near the origin} (E) has $\rho=100\%$ only because both the effect and the capacity have collapsed. \emph{III, the interior climb} (F/G and the $\beta=0$ path) reaches $\rho=50.3\%$ in v2 and $63.2\%$ in v3 at nearly constant $C$, while valid-pair rate decreases. \emph{IV, the high-capacity corner} (H--K) contains the robot-side partition-token families, with $C\approx0.48$--$0.50$ and $\rho=91.1$--$100\%$. These locations describe the tested robot mechanisms, not necessary conditions for the region: the balanced testbed later places an amortized posterior in the same corner. The empirical points also do not trace a Pareto frontier; the consistency path bends away from the boundary as it rises, and the middle of the boundary remains largely unoccupied.

\subsection{Class I: the floor (families A, C, D)}
\label{sec:floor}

\begin{figure}[pos=tbp]
  \centering
  \includegraphics[width=\textwidth]{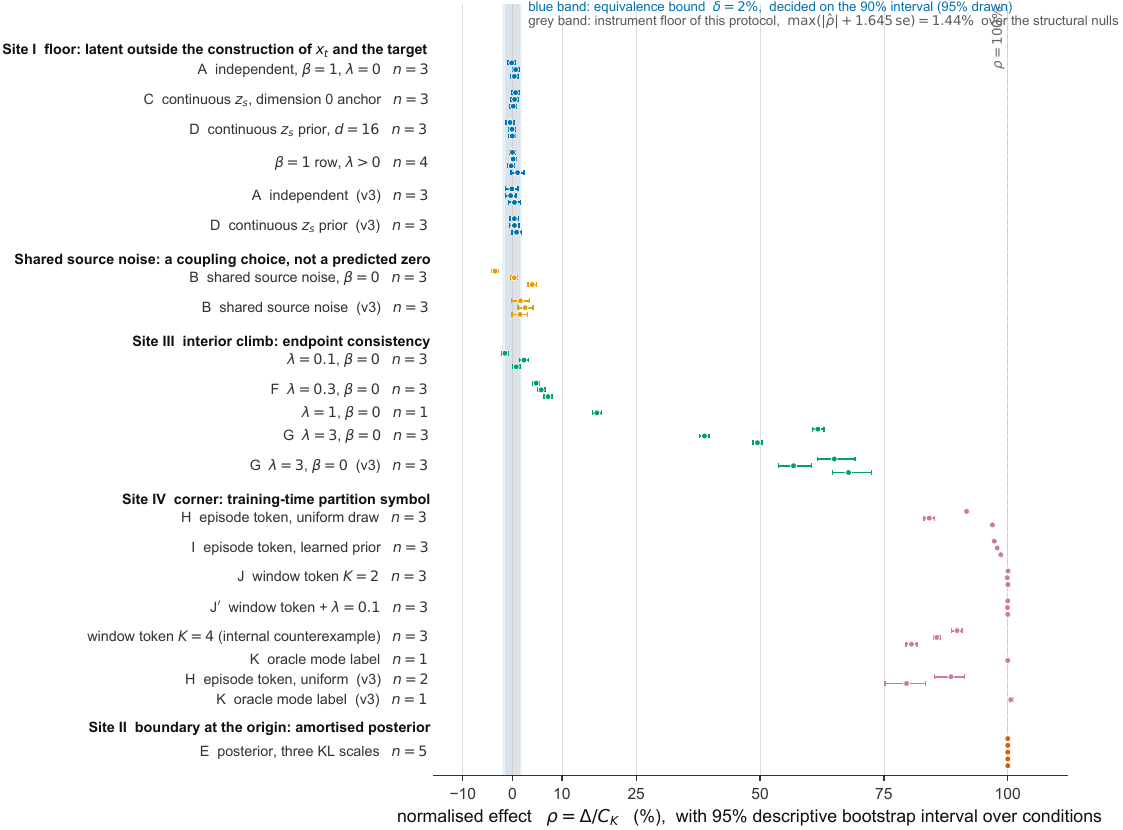}
  \caption{Run-level normalized effects $\rho=\Delta/C$. Points are shown against zero, the measured structural-null range, and the prespecified equivalence band $\delta=\pm2\%$. Error bars are 95\% condition-bootstrap intervals; TOST decisions use the corresponding 90\% intervals. Shared source noise is shown separately because it couples the sampled flows and is not covered by Lemma~\ref{lem:zblind}.}
  \label{fig:forest}
\end{figure}

The continuous-$z_s$ prior family, in which the same 16-dimensional latent is fed to both velocity fields at every ODE step, has $\Delta=-0.0009$ and $\rho=-0.2\%$ in v2. The dimension-zero anchor and independent family provide structural zeros for calibrating the equivalence test: across the two robot environments, $\max(|\hat\rho|+1.645\,\mathrm{se})=1.44\%$. With the prespecified margin $\delta=2\%$, all continuous-prior runs pass TOST; the smallest supported bands are 1.16\% in v2 and 1.66\% in v3. Thus the finite models are empirically equivalent to zero at the stated resolution, in agreement with Corollary~\ref{cor:zblind}(b). The same procedure rejects equivalence for the posterior, consistency, and token families.

The prior family differs slightly from the dimension-zero anchor ($-0.19\%$ versus $+0.42\%$ in $\rho$), but both estimates remain well within the equivalence margin. The comparison therefore does not change the substantive conclusion at the resolution of this protocol. The $z_s$-sensitivity of both branches also decreases monotonically over the saved checkpoints, by factors of about 260 and 300 over 50{,}000 steps (Table~\ref{tab:sensitivity}). Equation~\eqref{eq:excess} predicts that this sensitivity vanishes at the population optimum; the observed monotone trajectory is an empirical optimization pattern.

Shared source noise has a small family mean in the two robot environments but is not part of the auxiliary-latent class. Its v2 seed values are $\rho=-3.50/{+}0.39/{+}4.01\%$, and the valid-pair rates also vary widely (0.9088/0.8503/0.5445). The mean therefore masks substantial seed dependence. Section~\ref{sec:basin} explains the structural distinction, and Section~\ref{sec:toy-mapping} shows that the same coupling changes sign across representation maps.

Two examples illustrate why the baseline is formed per condition. For family B, the observed mismatch is 4.63 percentage points below that of the independent family, but its independent-choice baseline falls by 4.58 points; almost all of the raw gap is explained by changed marginals. In v3, one seed has a pooled effect of $+0.1022$ but a per-condition effect of $+0.0103$, showing the Simpson term in Remark~\ref{rmk:simpson} directly.

\subsection{Why shared source noise is not a member of the floor class}
\label{sec:basin}

Family B is close to the floor numerically, but its mechanism is different. An auxiliary latent is appended to the velocity-field input without changing the interpolation state or target; under the conditions of Lemma~\ref{lem:zblind}, the Bayes predictor averages it out. Source noise instead participates in both $x_t$ and $x_1-x_0$ and is the initial state of the sampled flow. Sharing it therefore selects a common-random-number coupling between the two marginal samplers.

The sign of this coupling depends on how the two noise-to-mode maps partition source space. If their mode basins align, a shared draw promotes agreement; if the basins are unrelated, the effect is near zero; if they are oppositely aligned, the same draw produces anti-coordination. The CFM marginal losses do not determine this cross-branch pairing. Representation geometry and independent training can therefore change the joint coupling without changing either branch's marginal distribution.

The mapping study in Section~\ref{sec:toy-mapping} isolates this effect: changing the representation map moves shared-source-noise coordination across a 90-point range and reverses its sign, and the reversal survived a preregistered confirmatory factorial at ten seeds per cell. The robot-side seed variation is consistent with weak and unstable basin alignment for that particular representation pair; at eight seeds the v2 family mean is $-0.5\%$ with a between-seed sd of $3.4$ and four signs each way. The same distinction applies to other inference-time shared randomness, such as a common solver stream or reused dropout mask; these variables can couple branches even when the velocity fields themselves are $Z$-blind.

\paragraph{Seed expansion}
The claim-bearing families were re-trained to six to ten seeds under a preregistered protocol (design, seed lists, and analysis script committed before any new training or unblinding; \ref{app:prereg}). Table~\ref{tab:seeds} reports seed-level means. The floors hold with between-seed sds below 0.3 pp; the interior ladder is monotone with the $\lambda{=}1$ rung de-singled at five seeds; the corner families hold at the exact both-valid caliber (episode token, learned prior: all eight seeds above 85\%, minimum 94.1\%); and the v3 token family softens from its two-seed value of $+84.1\%$ to a six-seed mean of $+78.5\%$, which is reported as the current estimate. Every expansion run stores both-valid contingency marginals natively.

\begin{table}[t]
\centering
\caption{Preregistered seed expansion: seed-level occupancy $\rho$ (head-valid caliber unless noted; seed count in parentheses). Between-seed sd in brackets.}
\label{tab:seeds}
\footnotesize
\begin{tabular}{@{}l rr l rr@{}}
\toprule
v2 family & $\bar\rho$ & sd & v3 family & $\bar\rho$ & sd \\
\midrule
A independent (6) & $+0.27\%$ & 0.27 & A independent (6) & $+0.00\%$ & 0.20 \\
B shared noise (8) & $-0.54\%$ & 3.41 & B shared noise (8) & $+0.48\%$ & 2.09 \\
G consistency $\lambda3$ (8) & $+45.3\%$ & 8.72 & G consistency $\lambda3$ (6) & $+41.9\%$ & 1.79 \\
$\lambda{=}1$ rung (5) & $+20.2\%$ & 5.18 & episode token (6) & $+78.5\%$ & 3.26 \\
$\beta{=}1,\lambda{=}3$ control (4) & $+0.58\%$ & 0.48 & continuous $z_s$ prior (6) & $+0.48\%$ & 0.19 \\
episode token, learned (8)$^{\mathrm{bv}}$ & $+97.2\%$ & 1.36 & & & \\
episode token, uniform (6)$^{\mathrm{bv}}$ & $+93.1\%$ & 4.90 & & & \\
continuous $z_s$ prior $d{=}16$ (6) & $-0.03\%$ & 0.15 & & & \\
\bottomrule
\end{tabular}
\end{table}

\subsection{Class II: the boundary at the origin (family E)}
\label{sec:origin}

A data-dependent posterior changes the problem by allowing the latent to carry information about the target during training. The five amortized-posterior runs all have $\mathrm{obs}=0$, and their residuals fall to 15--18\,mm. Their modal entropies, however, range from 0 to 0.166 nats, with side fractions near 0 or 1. The branches agree because both concentrate on one mode. Under exact both-valid accounting, $\Delta=C$ in every run, but $C$ is only 0.0001--0.0625. The family therefore lies near the origin of Fig.~\ref{fig:region}: a normalized effect of 100\% with almost no modal capacity. Different seeds collapse to different modes, confirming that the selected side is symmetry breaking rather than task semantics.

\begin{table}[t]
\centering
\caption{Amortized-posterior runs under exact both-valid accounting, with a window-token reference. Each posterior run has $\mathrm{obs}=0$ and $\Delta=C$, so the normalized effect is 100\% despite capacities close to zero. The pooled ratio is shown only to illustrate the aggregation effect.}
\label{tab:posterior}
\scriptsize
\setlength{\tabcolsep}{3pt}
\begin{tabular}{@{}l c rrrr rrrr r@{}}
\toprule
Configuration & seed & $\mathrm{obs}$ & Valid & $q_L$/$p_L$ & Entropy (nats) & $\mathrm{BL}$ & $\Delta$ & $C$ & $\rho$ (pool/p.-c.) & Resid. \\
\midrule
$\beta_{\mathrm{KL}}{=}1$, fb 0.1 & 0 & 0.0000 & 0.8804 & 0.992/0.986 & 0.046/0.074 & 0.0158 & $+0.0158$ & 0.0158 & 104.1\%/100.0\% & 17.61 \\
$\beta_{\mathrm{KL}}{=}1$, fb 0.1 & 1 & 0.0000 & 0.9087 & 0.992/0.988 & 0.048/0.067 & 0.0163 & $+0.0163$ & 0.0163 & 102.1\%/100.0\% & 15.95 \\
$\beta_{\mathrm{KL}}{=}0.01$ & 0 & 0.0000 & 0.9620 & 0.990/0.991 & 0.057/0.051 & 0.0170 & $+0.0170$ & 0.0170 & 100.2\%/100.0\% & 14.76 \\
$\beta_{\mathrm{KL}}{=}0.01$ & 1 & 0.0000 & 0.9010 & \textbf{0.032/0.039} & 0.143/0.166 & 0.0625 & $+0.0625$ & 0.0625 & 100.5\%/100.0\% & 15.56 \\
$\beta_{\mathrm{KL}}{=}1$, fb 0.5 & 0 & 0.0000 & 0.9937 & \textbf{1.000/1.000} & \textbf{0.000/0.000} & 0.0001 & $+0.0001$ & 0.0001 & 100.0\%/100.0\% & 15.04 \\
\midrule
(ref.) window token $K{=}2$ & 0 & 0.0000 & 0.9800 & 0.502/0.501 & 0.693/0.693 & 0.5000 & $+0.5000$ & 0.5000 & 100.0\%/100.0\% & 22.51 \\
\bottomrule
\end{tabular}
\end{table}

The immediate cause is aggregate-posterior/prior mismatch: training samples $z_s$ from $q(z\mid x_1,c)$, whereas inference samples from $\mathcal{N}(0,I)$. The decoder then maps poorly covered latent regions to one behavior. The mechanism is not intrinsically collapsing. On the balanced testbed, the same posterior reaches $C_K=0.494$--$0.498$ with entropy $\ln2$ and $\rho=100\%$; under a $(0.92,0.08)$ target prior it collapses again, while the $\beta_{\mathrm{KL}}=0$ variant does not. The observed regime therefore depends jointly on KL pressure, target imbalance, and separability (Section~\ref{sec:toy-skew}). These factors do not yet explain the balanced robot environments, which indicates that data dimension, demonstration geometry, or chunk structure also contributes.

The token results suggest an additional practical distinction. The $K=2$ window-token usage is 0.674/0.326 rather than uniform, yet both values occur more than 40{,}000 times in training. Test-time draws therefore remain within well-sampled support. In contrast, a continuous posterior can be informative during training while leaving poorly covered regions under the inference prior. The results suggest that a useful shared variable should influence the target during training and be sampled from regions on which the decoder has learned stable behavior.

\subsection{Class III: the interior climb (families F, G)}
\label{sec:interior}

\begin{figure}[pos=tbp]
  \centering
  \includegraphics[width=\textwidth]{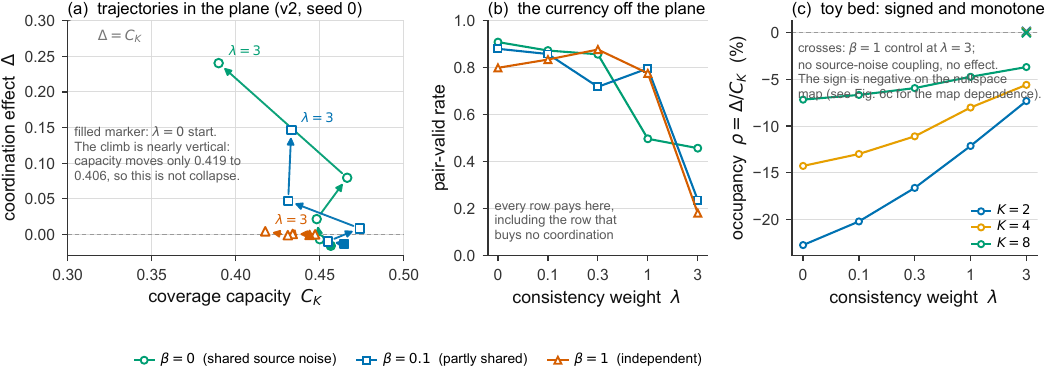}
  \caption{Endpoint-consistency sweep. (a) Paths through the $(C,\Delta)$ plane for the $\beta\times\lambda$ grid. With shared source noise ($\beta=0$), coordination increases while capacity changes little; with independent source noise ($\beta=1$), the path remains near the floor. (b) Valid-pair rate for the same runs. (c) Signed normalized effects on the non-robotic testbed for $K\in\{2,4,8\}$, with three seeds per cell.}
  \label{fig:ladder}
\end{figure}

The consistency regularizer lies outside Lemma~\ref{lem:zblind} because \eqref{eq:cons} changes the objective by rewarding cross-branch agreement. Tables~\ref{tab:grid} and~\ref{tab:l3seeds} report the full $\beta\times\lambda$ grid and the three-seed results for its strongest setting.

\emph{Coordination increases in both environments.} At $\lambda=3$, all three v2 condition-bootstrap intervals exclude zero and the seed-level normalized effects span 38.5--61.5\%. In v3, the corresponding effects are $+0.3097/{+}0.2090/{+}0.3057$, with mean $\rho=63.2\%$. This is the only tested mechanism that produces appreciable coordination without a data-derived partition, although the variation across seeds remains substantial.

\emph{The main cost lies outside the coordination--capacity plane.} In v2, valid-pair rate falls from 0.7679 to 0.5254 (0.750 to 0.445 in v3), while $C$ changes only from 0.4185 to 0.4061. The trajectory is therefore nearly vertical in Fig.~\ref{fig:ladder}: coordination rises without the capacity collapse seen for the posterior family. This behavior depends on the balanced target distribution. On the skewed testbed, the same regularizer also reduces modal capacity, including in the $\beta=1$ control where it produces almost no coordination (Section~\ref{sec:toy-skew}).

Because $\Delta$ is conditioned on both branches being valid, a lower valid-pair rate could in principle increase the measured coordination by removing difficult pairs. The controls do not support that explanation as the sole cause. At $\beta=1,\lambda=3$, valid-pair rate falls to 0.1814 but $\rho$ remains 1.0\%; conversely, the token families reach $\rho\ge91\%$ with valid-pair rates between 0.777 and 0.995. The consistency family therefore combines two effects: greater agreement among valid pairs and fewer valid pairs overall. Both quantities are reported together.

\emph{In this grid, coordination requires a shared source-noise channel.} The $\beta=1$ row remains near zero across the full $\lambda$ range ($\rho\le2.9\%$), whereas the $\beta=0$ and $\beta=0.1$ rows increase with $\lambda$. With independent source noise, the regularizer can reduce disagreement only by narrowing each branch's output distribution; at $\beta=1,\lambda=3$ the valid-pair rate falls to 0.1814 without a meaningful coordination effect. Shared source noise provides a common variable that the loss can use to align the two mode choices. The strongest setting is supported by three seeds; other high-$\lambda$ grid cells are single-seed and are treated as supporting ablations.

\paragraph{A third representation: consistency-driven coordination is pairwise}
\label{sec:third-rep}
Consistency-driven coordination is specific to the pair linked by the loss. In v3, $\lambda_{\mathrm{cons}}=3$ gives $\Delta=+0.3353$ for joint--position, whereas the same model has $+0.0018$ for joint--orientation, with a confidence interval spanning zero. This pattern is more consistent with a pairwise coupling than with a single internal mode variable shared by all branches. The orientation mode information is introduced by the task construction, so this ablation does not establish that orientation naturally encodes the mode. Table~\ref{tab:pose} reports the full $\lambda_{\mathrm{cons}}\times\lambda_{\mathrm{rot}}$ comparison.

\subsection{Class IV: the corner (families H--K), and the residual floor}
\label{sec:corner}

In the robot environments, low-cardinality discrete symbols provide the most reliable route to the high-capacity corner. The oracle and the $K=2$ window token have $\Delta=+0.4996$ and $+0.4994$, respectively, and both reach $\rho=100\%$. These symbols are informative during training and take test-time values that are well represented in the training support. They are not the only possible route in general: the balanced testbed places an amortized posterior at the corner as well. The token experiments instead establish a robust constructive route whose behavior can be interpreted and controlled.

The $K>2$ evidence comes only from the balanced testbed. Uniform oracle draws achieve 80.1--87.3\% across $K\in\{2,3,4,8\}$, while the unsupervised token remains close to the oracle at $K\le3$ and falls to about 50\% at $K=4$ and $8$. Thus shared symbols remain effective when their cardinality and learned use are compatible with the target modes, but the robot experiments themselves support only $K=2$.

The token constructions differ in how the symbol is defined and in the cost they leave behind.
\begin{itemize}
\item \emph{Window-level $K{=}2$ k-means} (fitted on the 127{,}707 training windows): NMI with the true mode 0.1717, usage 0.674/0.326. It coordinates essentially perfectly---on a partition the clustering invented, not the task's two sides. At test time the token is drawn uniformly and fed to both branches; no mode-inference module is used. With $\lambda{=}0.1$ added (J$'$), validity 0.9954 and residual 12.39\,mm, the closest non-oracle approach to the oracle floor (10.18\,mm).
\item \emph{Window-level $K{=}4$} leaves one token value unused by valid behavior: every draw under that value is invalid, reducing valid-pair rate to 0.6356 while $\rho$ remains 85.3\%. This separates coordination from behavioral validity.
\item \emph{Episode-level $K{=}2$ token} (fitted on 700 transport trajectories): NMI $=1.0000$; it \emph{is} the mode label. With a conditional categorical prior (I): $\Delta=+0.4710$ ($\rho=97.9\%$), residual 24.22\,mm, validity 0.7894. With no inference and uniform draws (H): $+0.4339$ ($91.1\%$), residual 30.16\,mm, validity 0.7769; in v3, head-valid, $+0.4232/{+}0.3614$ (88.6/79.6\%). H and I are separate trainings that differ in the test-time token distribution; because the learned prior here is nearly uniform (Section~\ref{sec:corner-entry}), their difference is not a clean measure of the value of inference.
\end{itemize}

The episode-token family does not reduce valid-pair rate relative to the uncoordinated baseline: family I averages 0.7894 versus 0.7729 for family A. Its main cost is instead the same-mode residual floor. The oracle floor is 10.18\,mm; the window-token floors are 10.72--13.70\,mm with $\lambda=0.1$ and 13.93--22.45\,mm without it; the uniform episode-token floor is 17.61--20.08\,mm. The learned-prior episode token was not evaluated under the stratified protocol, so its same-mode floor is not reported. Across the measured families, token granularity rather than coordination strength explains the remaining floor: one symbol per episode is coarser than one per 16-step window, whereas strong consistency regularization has similar floors despite a much smaller coordination effect.

\subsection{Coordination and behavioral-prior fidelity}
\label{sec:corner-entry}

\begin{figure}[pos=tbp]
  \centering
  \includegraphics[width=0.92\textwidth]{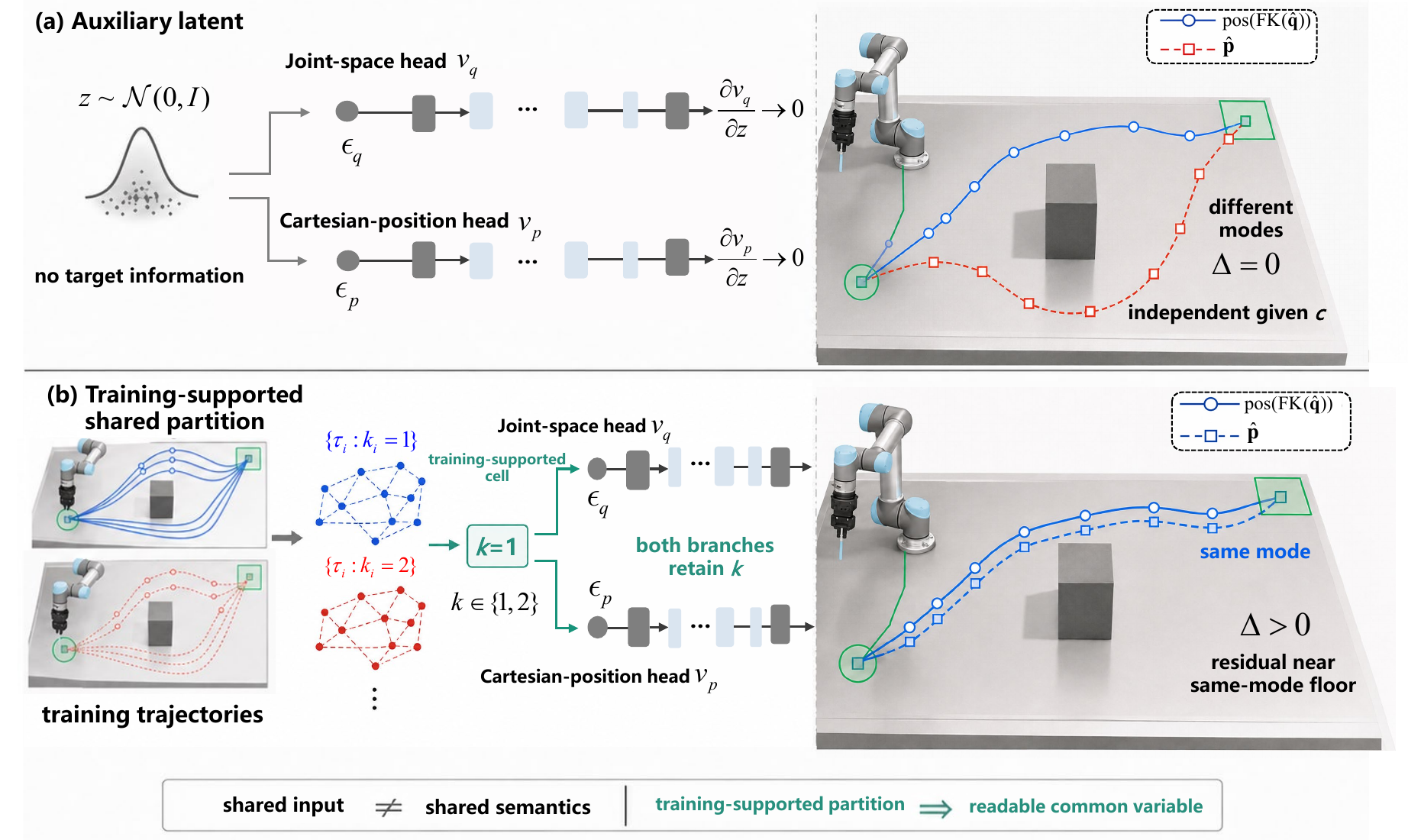}
  \caption{Two shared-variable constructions. An auxiliary random input that is irrelevant to the CFM regression target is ignored at the population optimum (top). A low-cardinality symbol fitted from training trajectories can be learned by both branches and coordinate their samples (bottom). The symbol provides a shared partition; it need not coincide with a semantic task label.}
  \label{fig:symbol}
\end{figure}

The token experiments separate two objectives that coincide only in balanced settings. Coordination requires the two branches to respond similarly to the same symbol. It does not require that the symbol correspond to a task-semantic mode: the $K=2$ window partition has NMI 0.1717 with the task label and still reaches the boundary, while the testbed converts an almost zero-information symbol into a 77.8\% normalized effect. Nor does coordination require test-time mode inference; families H and J draw the symbol uniformly.

Behavioral-prior fidelity asks a different question: whether the generated mode frequencies match the data and whether every symbol value leads to feasible behavior. In the balanced robot environments, the learned categorical prior is essentially uniform (conditional entropy $\ln2$, mean probabilities 0.497/0.503), so uniform and learned draws should be viewed as separate training repeats rather than evidence that inference improves coordination. Two deviations expose the distinction. First, the $K=4$ window token contains a value with no valid behavior and therefore loses valid-pair rate. Second, on skewed testbed targets, a uniform symbol marginal produces total-variation distortion of 0.103--0.330 and its normalized coordination effect falls from 82.4\% to 46.3\% as the target becomes more imbalanced.

Matching the symbol marginal to the learned mode frequencies improves both quantities on the skewed testbed: $\rho$ increases by 3.5--41.2 percentage points and the marginal distortion falls by factors of 7--42. This improvement does not come from state-dependent mode inference; the prior head learns a condition-independent marginal. The resulting design rule is therefore modest: a shared symbol can coordinate branches without semantic labels or test-time inference, but the sampler should match the symbol's marginal to the desired behavior frequencies and ensure that each symbol value is supported by valid training examples. Figure~\ref{fig:symbol} illustrates the distinction.

The coordination estimates depend on the assignment threshold and integration resolution. Full sensitivity curves are provided in \ref{app:sensitivity}; the reference operating points are $\tau=0.11$, NFE 20 for v2 and $\tau=0.18$, NFE 20 for v3.

\section{A synthetic testbed for mechanism transfer}
\label{sec:toy}

The same coordination calculus is applied to a two-dimensional conditional flow-matching testbed with no robotic component. The aim is to distinguish mechanism-dependent behavior from properties of the two robot environments. Three robot-side occupancy patterns transfer directly: the auxiliary-latent floor, the consistency-driven interior, and high-capacity coordination. The collapsed posterior does not transfer under balanced targets; it reaches the high-capacity corner instead. The testbed also isolates a strong dependence of source-noise coupling on the representation map and shows how skewed target priors change modal capacity, coordination, and behavioral-prior fidelity.

\subsection{Testbed design and evidence types}
\label{sec:toy-setup}

\paragraph{Target and branches}
RotatingRing uses conditions $c=(\theta,r)$, with $\theta\sim U(0,2\pi)$ and $r\sim U(0.70,1.30)$, to rotate and scale a ring of $K$ equiprobable modes. Each sample is a length-$H=16$ single-peak path toward mode $k$ at azimuth $\theta+2\pi k/K$. Branch B models $\mathbb{R}^2$; branch A models a known inverse representation $g^{-1}(\mathbb{R}^2)$. The representation study uses three exact round-trip maps: an orthogonal map, a nonlinear polar map, and a primary $\mathbb{R}^2\to\mathbb{R}^6$ null-space lift that plays the structural role of forward kinematics. The $K$-way assignment rule generalizes the RotatingBypass rule, with $\tau=0.55$ at the reference operating point. Generator and map definitions are given in \ref{app:prereg}.

\paragraph{Symmetry controls}
Because $\theta$ is uniform over the full circle, $x_1\mid k$ has the same world-frame distribution for every mode. Fixed world-frame features are therefore uninformative about $k$: the strongest of 36 scalar features has pooled AUROC 0.5183--0.5253, whereas a probe that also receives the condition reaches accuracy 1.0000. The per-condition separation direction is close to uniform in the world frame ($R=0.0289$--$0.0301$), and the mode spacing varies more than in either robot environment (CV 0.1767). The balanced study contains 186 training runs and 135 threshold/NFE re-evaluations.

\paragraph{Evidence types}
Some results follow directly from the construction: fixed world-frame features cannot identify the mode, and balanced targets keep $C_K$ close to its maximum. These are implementation checks rather than empirical findings. The mechanism comparisons remain empirical: the auxiliary-latent and independent-noise controls could have moved away from zero; the consistency ladder could have failed to be monotone; source-noise effects could have retained their sign across maps; and the posterior could have collapsed despite balanced targets.

\subsection{Occupancy patterns on the balanced testbed}
\label{sec:toy-occupancy}

\begin{figure}[pos=tbp]
  \centering
  \includegraphics[width=\textwidth]{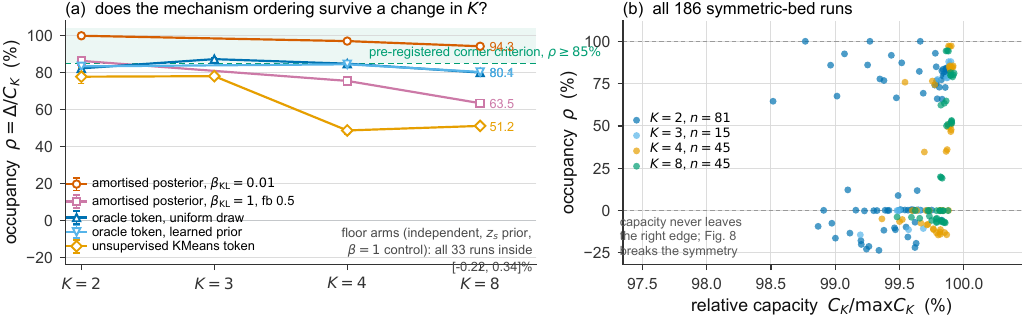}
  \caption{Effect of mode count on the balanced non-robotic testbed. (a) Normalized effects for five mechanism families across $K\in\{2,3,4,8\}$, averaged over three seeds; error bars show the seed range and the shaded line marks the prespecified $85\%$ criterion. The unsupervised token degrades at $K=4$ and $8$, whereas the low-KL posterior remains above the criterion. (b) All 186 runs in relative-capacity--effect coordinates. The construction keeps capacity near its maximum; unequal target probabilities in Fig.~\ref{fig:skew} provide the complementary capacity sweep.}
  \label{fig:ksweep}
\end{figure}

\begin{figure}[pos=tbp]
  \centering
  \includegraphics[width=\textwidth]{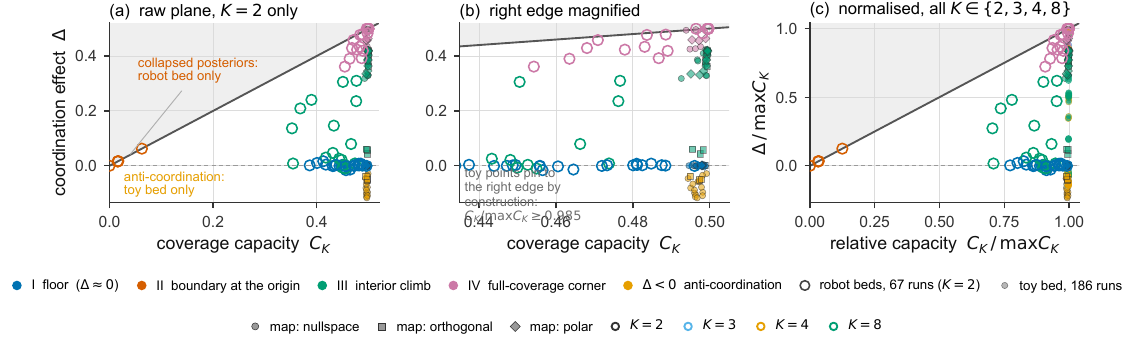}
  \caption{Robot-environment and balanced-testbed runs in one attainable region. (a) Raw coordinates for $K=2$; (b) the high-capacity region magnified; (c) both axes normalized by $\max C_K=1-1/K$ to compare $K\in\{2,3,4,8\}$. The balanced testbed stays near maximum capacity by construction, whereas the robot runs span the capacity axis and include the collapsed-posterior cluster. Pronounced negative effects occur only in map-dependent testbed families.}
  \label{fig:overlay}
\end{figure}

\textbf{Floor.} The continuous-$z_s$ prior and the $\beta{=}1,\lambda{=}3$ control satisfy $|\rho|\le0.49\%$ over 27 runs in 9 ($K$, map) cells; all effect CIs contain zero; TOST at $\delta=2\%$ passes at both run and seed levels; $\hat d\le1.04\%$ (0.30\% at $K{=}8$, where $C_K$ grows while the effect stays zero). The same procedure rejects equivalence for the posterior, the shared-noise ladder, and the token families.

\textbf{Interior climb.} The $\beta{=}0$ ladder is strictly monotone in the signed effect at $K=2/4/8$ (15 family means, no inversion; panel (c) of Fig.~\ref{fig:ladder}), while the $\beta{=}1$ control stays at zero ($+0.13\%/{+}0.04\%/{-}0.03\%$): consistency regularization alone produces no coordination; it requires source-noise coupling. The robot-side pattern is therefore reproduced with three seeds per cell. The preregistered criterion was stated as non-decreasing $|\Delta|$ and fails in its literal form under the \texttt{nullspace} map, where the $\lambda=0$ anchor is already anti-coordinated. The signed effect $\Delta$ is strictly monotone, so the criterion form is recorded as failed while the mechanism pattern is supported.

\textbf{High-capacity region.} Using the prespecified criterion $\rho\ge85\%$ and $C_K\ge0.80\max C_K$, the two posterior settings qualify at $K=2$, the oracle qualifies at $K=3$, and the low-KL posterior qualifies at $K=4$ and $8$. Oracle and token families remain at 78--85\% across the larger mode counts. The low-KL posterior is therefore the most consistent high-capacity occupant on this balanced testbed, directly showing that an explicit partition token is sufficient but not necessary.

\begin{table}[t]
\centering
\caption{Normalized coordination effect $\rho$ across mode counts on the balanced testbed. Bold entries meet the prespecified high-capacity criterion $\rho\ge85\%$ and $C_K\ge0.80\max C_K$. All rows have $C_K\ge0.992\max C_K$ by construction.}
\label{tab:toy-corner}
\footnotesize
\begin{tabular}{@{}l rrrr@{}}
\toprule
Family & $K{=}2$ & $K{=}3$ & $K{=}4$ & $K{=}8$ \\
\midrule
amortized posterior, $\beta_{\mathrm{KL}}{=}0.01$ & \textbf{100.0\%} & --- & \textbf{97.1\%} & \textbf{94.3\%} \\
amortized posterior, $\beta_{\mathrm{KL}}{=}1$, fb 0.5 & \textbf{86.4\%} & --- & 75.6\% & 63.5\% \\
oracle token, uniform test draw & 82.4\% & \textbf{87.3\%} & 84.9\% & 80.1\% \\
oracle token, learned prior & 83.4\% & --- & 84.5\% & 80.4\% \\
unsupervised k-means token & 77.8\% & 78.1\% & 48.7\% & 51.2\% \\
\bottomrule
\end{tabular}
\end{table}

\textbf{Collapsed origin.} No family collapses on the balanced testbed: across 186 runs, modal entropy divided by $\ln K$ is at least 0.9995 and relative capacity is at least 0.9852. In particular, the $\beta_{\mathrm{KL}}=0.01$ posterior has $\mathrm{obs}=0$, $\mathrm{BL}=C_K=0.4938$--$0.4983$, and $\rho=100\%$. The same mechanism therefore occupies opposite ends of the diagonal in the balanced testbed and robot environments. Its testbed valid-pair rate is 17.4--22.5 percentage points below the independent family at the same threshold, showing that a high-capacity corner can still carry a validity cost. Section~\ref{sec:toy-skew} subsequently induces collapse on the same testbed by changing the target prior.

\subsection{Dependence on the representation map}
\label{sec:toy-mapping}

\begin{figure}[pos=tbp]
  \centering
  \includegraphics[width=\textwidth]{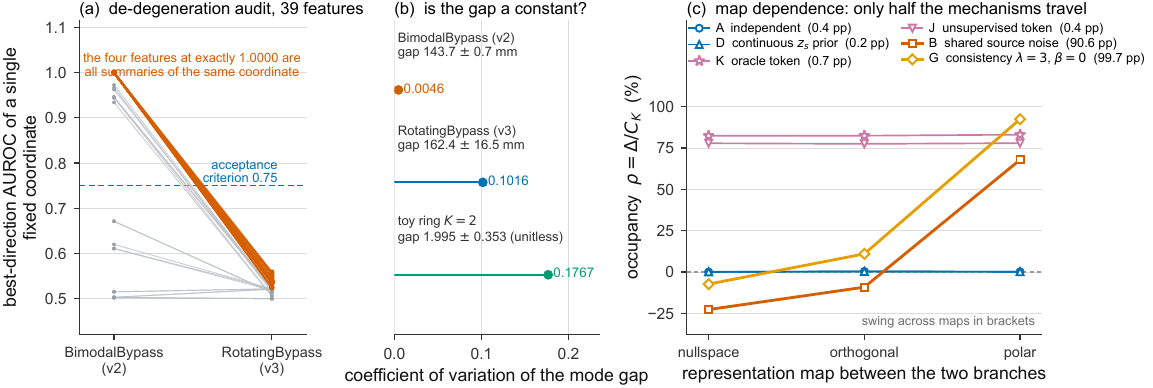}
  \caption{Environment and representation-map checks. (a) Best single-coordinate AUROC for mode prediction in BimodalBypass and RotatingBypass; the fixed-coordinate shortcut present in v2 is substantially reduced in v3. (b) Coefficient of variation of the mode gap in the two robot environments and the non-robotic testbed. (c) Normalized effect $\Delta/C_K$ under three representation maps. The floor and symbol-based families vary by less than 0.7 percentage points, whereas shared source noise and the $\beta=0$ consistency family change sign over ranges of 90.6 and 99.7 points.}
  \label{fig:maps}
\end{figure}

The floor and symbol-based families are stable across the three maps in Table~\ref{tab:maps}, varying by less than 0.7 percentage points. Shared source noise and the $\beta=0$ consistency family are not: their normalized effects span 90.6 and 99.7 points and reverse sign. This confirms the basin-alignment interpretation of Section~\ref{sec:basin}. A shared source draw coordinates branches only when their noise-to-mode maps assign that draw to corresponding modes. The positive end of this continuum is consistent with multi-view generators that operate on aligned latent grids \citep{dong2025noisecontroller,liu2024syncdreamer}; the cross-representation setting need not share that alignment.

\paragraph{Preregistered confirmatory factorial}
Because the reversal is the paper's most consequential mechanism finding, it was re-tested prospectively. A registered design added seeds 3--9 to the published 0--2 for arms A and B on all three maps (ten seeds per cell), with the arm-by-map interaction as the only gated test and the analysis script committed before unblinding. The reversal replicates: $D(\texttt{polar})=+66.8$ pp with 95\% CI $[+65.9,+67.7]$, $D(\texttt{nullspace})=-23.3$ pp $[-23.9,-22.8]$, $D(\texttt{orthogonal})=-9.3$ pp $[-10.6,-7.9]$; the interaction range is 90.1 pp with permutation $p=1\times10^{-4}$ ($10^{4}$ permutations), and between-seed sds are 0.8--1.9 pp. Ten-seed means move the three-seed values of Table~\ref{tab:maps} by at most 0.7 pp (shared source noise: $-23.2/-9.0/+67.1$; range 90.3 pp). The independent-noise floor passes seed-level TOST at $\pm2\%$ on every map, the ladder keeps its monotone climb at ten seeds per rung ($-20.7\%$ to $-8.0\%$ across $\lambda=0.1$--$3$) with a flat valid-pair rate (0.982--0.985, so the validity cost is a robot-side phenomenon rather than a property of the regularizer itself), and the corner arms stay within 2 pp of their three-seed values (kmeans token $79.5\pm3.6$, oracle $83.6\pm2.5$).

\begin{table}[t]
\centering
\caption{Occupancy $\rho$ by representation map ($K{=}2$; family pairs quoted first/second).}
\label{tab:maps}
\footnotesize
\begin{tabular}{@{}l rrr r@{}}
\toprule
Family & \texttt{nullspace} & \texttt{orthogonal} & \texttt{polar} & Range \\
\midrule
independent / continuous $z_s$ prior & $+0.02\%$ / $+0.10\%$ & $+0.37\%$ / $+0.30\%$ & $+0.15\%$ / $+0.17\%$ & 0.2--0.4 pp \\
unsupervised token / oracle token & $+77.8\%$ / $+82.4\%$ & $+77.5\%$ / $+82.4\%$ & $+77.9\%$ / $+83.0\%$ & 0.4--0.7 pp \\
shared source noise & $\mathbf{-22.7\%}$ & $\mathbf{-9.2\%}$ & $\mathbf{+67.8\%}$ & \textbf{90.6 pp} \\
$\beta{=}0$ consistency, $\lambda{=}3$ & $\mathbf{-7.3\%}$ & $\mathbf{+10.9\%}$ & $\mathbf{+92.4\%}$ & \textbf{99.7 pp} \\
\bottomrule
\end{tabular}
\end{table}

\textbf{Anti-coordination.} Sixty-nine of the 186 balanced-testbed runs have negative effects. Eighteen lie within the floor range ($|\rho|\le0.22\%$); the remaining 51 belong to shared-source-noise or consistency settings under the null-space and orthogonal maps, with $\rho=-3.3\%$ to $-23.5\%$. The robot-side negative effects are much smaller (maximum magnitude 3.50\%), so the pronounced anti-coordination band is specific to the representation maps explored on the testbed.

\subsection{Unequal mode probabilities}
\label{sec:toy-skew}

\begin{figure}[pos=tbp]
  \centering
  \includegraphics[width=\textwidth]{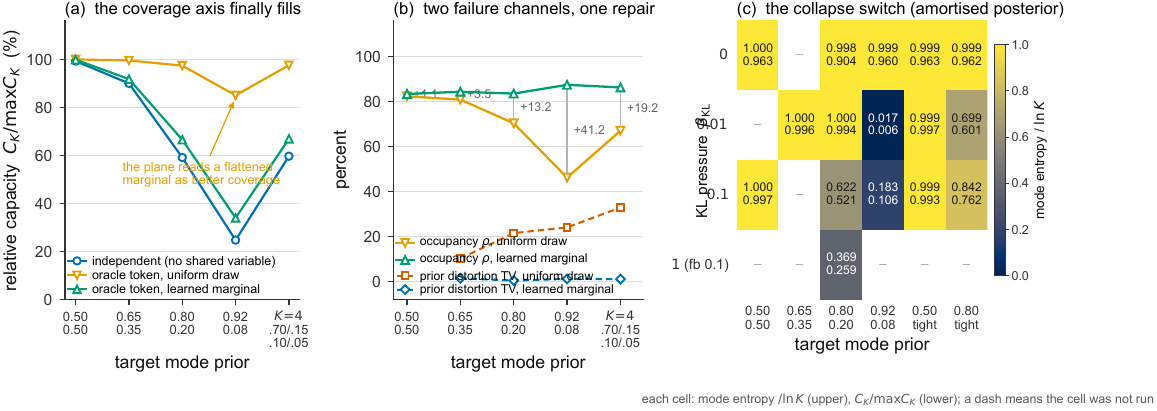}
  \caption{Unequal-probability testbed (153 runs, null-space map, $\tau=0.55$, NFE 20). (a) Relative capacity as the target prior becomes more imbalanced. Uniform token draws retain high generated diversity even when it no longer matches the target frequencies. (b) Normalized coordination and marginal distortion under uniform and learned-marginal token draws. Matching the token marginal improves both. (c) Amortized-posterior behavior over KL weight and target skew; each cell reports modal entropy and relative capacity. Strong skew with a nonzero KL term reproduces the low-capacity, zero-mismatch regime observed in the robot experiments.}
  \label{fig:skew}
\end{figure}

The balanced construction leaves little variation along the capacity axis, so a second study varies target mode probabilities and separability. The $K=2$ priors are $(0.65,0.35)$, $(0.80,0.20)$, and $(0.92,0.08)$; the $K=4$ prior is $(0.70,0.15,0.10,0.05)$. Unequal-spacing and reduced-separability controls are combined with a grid over KL weight and target skew, giving 153 runs (51 cells, three seeds) on the null-space map at $\tau=0.55$ and NFE 20. All per-condition tables satisfy the two bounds, and the binary and $K$-ary implementations agree numerically. Figure~\ref{fig:skew} summarizes the resulting regimes.

\emph{(a) The capacity axis fills.} The independent family's $C_K/\max C_K$ falls to 0.591 ($p{=}.80$), 0.248 ($p{=}.92$), 0.597 ($K{=}4$); the 153 runs span 0.0047--0.9986. Unequal spacing alone has no measurable effect (paired control: occupancy differences $\le0.2$ points, capacity $\le0.021$). On the four standard-separability skewed datasets, the independent family's $C_K/\max C_K$ sits systematically below the fidelity-consistent capacity (relative gaps 3.3/10.7/18.7/8.8\%)---flow matching itself slightly under-covers minority modes on skewed targets.

\emph{(b) Uniform token draws create two costs under skew.} Marginal distortion $\mathrm{TV}=0.103/0.216/0.241/0.330$ with validity unchanged, and occupancy falling $82.4\to80.9\to70.4\to46.3\%$ as the override rate rises. The normalized effect falls below 60\% at $p=0.92$ and under reduced separability with $p=0.80$. The learned-marginal fix repairs both channels at once, and on the balanced testbed the two variants differ by 1.0 point. This also explains why the uniform and learned-prior variants are nearly indistinguishable in the balanced setting.

\emph{(c) Posterior collapse depends on the regime.} At $p=(0.92,0.08)$ and $\beta_{\mathrm{KL}}=0.01$, the generated marginal is $[.9988,.0012]$ against target $[.9169,.0831]$, with entropy$/\ln K=0.017$, $C_K/\max C_K=0.006$, $\mathrm{obs}=0$, and $\rho=100\%$. This reproduces the low-capacity, zero-mismatch regime of family E. Across the observed grid, collapse appears only with a nonzero KL term and strengthens with target skew; reduced separability modifies the transition but does not induce collapse on its own. The grid still does not explain why the balanced robot environments collapse while the balanced testbed does not (Section~\ref{sec:origin}).

\begin{table}[t]
\centering
\caption{The collapse switch: amortized posterior across KL pressure $\times$ target skew (cells: $C_K/\max C_K$ / modal entropy $\div\ln K$; ``tight'' = reduced separability).}
\label{tab:collapse}
\footnotesize
\begin{tabular}{@{}l cccccc@{}}
\toprule
KL pressure & symm.\ .50 & $p{=}.65$ & $p{=}.80$ & $p{=}.92$ & tight .50 & tight .80 \\
\midrule
$\beta_{\mathrm{KL}}=0$ & .963 / 1.000 & --- & .904 / .998 & .960 / .999 & .963 / .999 & .962 / .999 \\
$\beta_{\mathrm{KL}}=0.01$ & .992 / 1.000 & .997 / 1.000 & .994 / 1.000 & \textbf{.006 / .017} & .997 / .999 & .601 / .695 \\
$\beta_{\mathrm{KL}}=0.1$ & .997 / 1.000 & --- & \textbf{.521 / .624} & \textbf{.106 / .189} & .993 / .999 & .762 / .820 \\
$\beta_{\mathrm{KL}}=1$ (fb 0.1) & .994 / 1.000 & --- & \textbf{.259 / .368} & --- & --- & --- \\
\bottomrule
\end{tabular}
\end{table}

\emph{(d) The attainable region changes under skew.} The auxiliary-latent floor extends across the capacity axis ($C_K/\max C_K=0.248$--$0.899$ with $|\rho|\le0.43\%$), confirming that its zero prediction does not depend on mode balance. The consistency path remains monotone but is no longer vertical because the regularizer also reduces modal capacity. Full-capacity corner points necessarily correspond to nearly uniform generated marginals and therefore conflict with fidelity to a skewed target; the learned-marginal family preserves fidelity at a lower, target-consistent capacity. One $\beta=1,\lambda=3$ cell at $p=(.80,.20)$ shows a capacity reduction with essentially zero coordination ($\rho=-0.3\%$), suggesting that consistency regularization can drop minority modes without coordinating the branches. Because this pattern appears in one tested cell, it is treated as preliminary.

\subsection{Scope of the testbed}
\label{sec:toy-scope}

The testbed establishes that the auxiliary-latent floor, the consistency-driven interior, and high-capacity coordination are not specific to robot kinematics. It also shows that the floor is insensitive to mode imbalance, that posterior collapse is regime-dependent, and that valid-pair rate and behavioral-prior fidelity can vary independently of the coordination coordinates.

Its limits are equally important. Invalidity is induced by the assignment threshold rather than by collision dynamics; the trajectories do not include contact, grasp, or acceleration phases; and every $K>2$ and skew-prior result is confined to this synthetic construction. The skew study uses only the \texttt{nullspace} map, and three maps establish map dependence without characterizing all possible representations. Finally, the testbed does not address settings in which the preferred mode is partially predictable from the observation.
\label{sec:toy-relabel}

Figure~\ref{fig:overlay} summarizes the main contrast between the two settings: the near-origin cluster occurs only in the robot-environment runs (five v2 posterior runs, $\rho$ exactly 100.0\%, $C_K\in[6.3\times10^{-5},0.063]$), the $\rho<0$ band to the testbed alone (51 substantive runs, sign-flipping to $+66.6\%$--$+92.9\%$ under \texttt{polar}), and balanced testbed points hug the right edge by construction (min $C_K/\max C_K$ 0.985, median 0.998) while the robot points genuinely spread (down to $1.3\times10^{-4}$; 26 of 67 below 0.9). In the balanced overlay, the robot runs provide most of the variation along the capacity axis; the 153 skewed runs fill the testbed interior in Fig.~\ref{fig:skew}. Threshold-matched comparisons, approximate head-valid coordinates, and valid-pair ranges are documented in \ref{app:tables}.

\section{Residual contamination and disagreement-based monitors}
\label{sec:residual}

The attainable-region analysis describes mode agreement; this section connects it to the cross-representation residual. Two questions are kept separate. First, how much of the residual is produced by valid cross-mode samples? Second, does that contamination materially affect failure detection? The first is studied in the constructed environments and in model ensembles; the second is evaluated on LIBERO-Plus.

\subsection{Residual decomposition}
\label{sec:share}

\begin{figure}[pos=tbp]
  \centering
  \includegraphics[width=\textwidth]{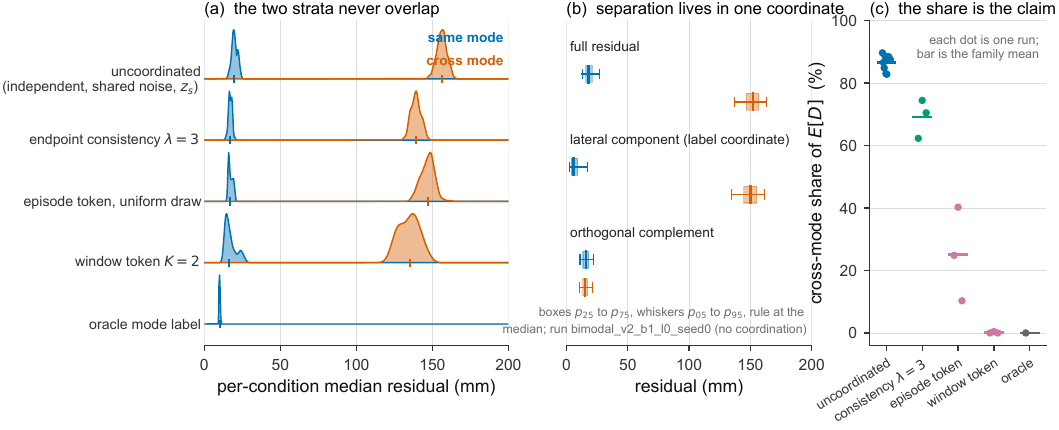}
  \caption{Residual decomposition in BimodalBypass. (a) Per-condition median residuals for same-mode and cross-mode sample pairs, grouped by mechanism family. (b) Full, lateral, and label-orthogonal residual components for one uncoordinated run; the separation is concentrated in the coordinate used by the mode assignment. (c) Cross-mode contribution to the both-valid residual mean, one point per run. Stronger coordination reduces this contribution from roughly $85\%$ to nearly zero, while the cross-mode magnitude itself remains set by the task geometry.}
  \label{fig:strata}
\end{figure}

For each checkpoint, the both-valid residual is decomposed into same-mode and cross-mode contributions. Table~\ref{tab:resid} reports the uncoordinated families; Fig.~\ref{fig:strata} shows the full mechanism progression.

\begin{table}[t]
\centering
\caption{Residual decomposition for uncoordinated families; ranges are over runs.}
\label{tab:resid}
\footnotesize
\begin{tabular}{@{}l cc@{}}
\toprule
Quantity & v2 (10 runs) & v3 (9 runs) \\
\midrule
same-mode floor $\mathbb{E}[D\mid\mathrm{same}]$ & 18.0--24.0\,mm & 25.0--33.5\,mm \\
cross-mode residual $\mathbb{E}[D\mid\mathrm{cross}]$ & 151.4--159.5\,mm & 174.7--183.9\,mm \\
both-valid mean $\mathbb{E}[D]$ & 77.3--88.5\,mm & 87.0--105.2\,mm \\
$\mathbb{E}[D]$ / same-mode floor & 3.32--4.68$\times$ & 2.96--3.96$\times$ \\
cross-mode contribution share & \textbf{82.8--89.7\%} & \textbf{81.2--87.0\%} \\
\bottomrule
\end{tabular}
\end{table}

The cross-mode share decreases as coordination strengthens: it is 0.81--0.90 for uncoordinated models, 0.51--0.75 under strong endpoint consistency, 0.10--0.45 for episode tokens, and approximately zero for window tokens and the oracle. Thus, in the constructed multimodal tasks, different valid mode choices account for more than four fifths of the uncoordinated residual mean. This is a statement about the composition of the mean, not about the residual's ability to distinguish success from failure.

\subsection{Geometric origin of the cross-mode separation}
\label{sec:tautology}

In v2, the residual separates same-mode and cross-mode samples with AUROC 1.0000 wherever both strata are present. This separation follows directly from the task geometry. The mode label is defined by the lateral trajectory component, and the residual contains that same component. The lateral residual has AUROC 1.0000 in v2 and 0.977--0.999 in v3, whereas the orthogonal component is 0.39--0.71 in v2 (mean 0.52 over adequately powered runs) and 0.60--0.71 in v3.

The magnitudes are likewise set by the environment. The v2 modes are $143.68$\,mm apart within the branching window, and the cross-mode residual is 0.97--1.11 times this distance. In v3 the mode gap varies across conditions, but the mean cross-mode residual remains 1.08--1.13 times the mean gap. The informative quantities are therefore the cross-mode probability, which the coordination mechanisms change, and the direction of the excess residual. A cross-mode sample is an alternative valid prediction rather than a uniformly worse prediction.

\subsection{Implications for disagreement monitors}
\label{sec:monitors}

Remark~\ref{rmk:scope} applies equally to two samples from one model and to samples from different models. If the compared draws are conditionally independent given the condition, then $\Delta=0$ by construction. On the v2 evaluation grid, private-noise ensembles and same-model Monte Carlo pairs have per-condition effects within 0.0010 of zero and pooled effects within 0.0027, giving an empirical estimator scale of approximately $\pm0.003$ for 75 conditions and 1280 samples per condition.

Independent ensembles have an even larger cross-mode share than the cross-representation residual: 0.922--0.950 in end-effector space and 0.932--0.948 for same-model Monte Carlo sampling. Their within-mode disagreement is smaller, so valid mode changes dominate the total more strongly. Under independent sampling this share is determined by the mode marginals and cannot be reduced by improving either marginal model alone.

\paragraph{Common random numbers}
\label{sec:crn}
Sharing the initial noise produces a large coordination effect across ensemble members but almost none across the joint and position representations (Table~\ref{tab:crn}). Ensemble members use the same representation, architecture, and data, so their noise-to-mode maps are closely aligned. The two action representations use different state spaces and independently parameterized fields, so a shared numerical draw need not represent the same latent direction. The mixed-objective ensemble confirms that the representation, rather than the objective weight, is the main distinction. Across the ensemble pairs, shared $x_0$ retains 62--93\% of the same-mode component while reducing the total both-valid mean to 17--57\% of its private-noise value.

\begin{table}[t]
\centering
\caption{Effect of shared initial noise across representations and across ensemble members. $\Delta$ is reported on the pooled accounting used for this comparison; per-condition values are 0.0018 lower on average.}
\label{tab:crn}
\footnotesize
\begin{tabular}{@{}l l rrr@{}}
\toprule
Object & Mechanism & $\mathrm{obs}$ & $\Delta$ & Valid-pair rate (uncoupled) \\
\midrule
cross-representation (family B) & shared source noise & 0.4362 & $+0.0039$ & 0.768 (0.773) \\
ensemble \texttt{b1\_l0} (ten pairs) & shared $x_0$ & 0.1434 & $+0.3632$ & 0.9305 (0.9296) \\
ensemble \texttt{b0\_l0} & shared $x_0$ & 0.1859 & $+0.3259$ & 0.9330 (0.9334) \\
ensemble \texttt{b0\_l3} & shared $x_0$ & 0.1079 & $+0.4003$ & 0.8353 (0.8303) \\
ensemble, mixed objectives & shared $x_0$ & 0.0551 & $+0.4130$ & 0.8656 (0.8587) \\
\bottomrule
\end{tabular}
\end{table}

\paragraph{VFD}
\label{sec:vfd}
The VFD construction \citep{romer2026vfd} introduces three dependencies absent from Action-L2: shared $x_0$, cross-evaluation at another member's state, and the time weight $\kappa_s=s/(1-s)$. Applied to the ensembles used here, the estimated cross-mode share falls from 92--95\% for Action-L2 to 44--72\% after shared-noise and same-state coupling, and to 1.1--33.5\% with the full weighted score. These values are obtained on our ensembles rather than the original VFD models, so they characterize the construction rather than the published method's absolute performance. Full score definitions and pairwise results are in \ref{app:mechtables}.

\subsection{External failure-monitor experiment}
\label{sec:liberoplus}

LIBERO-Plus \citep{fei2025liberoplus}, built on the LIBERO suites \citep{liu2023libero}, supplies external tasks, perturbations, severity labels, and success predicates. The executing policy is an action-chunk CFM policy trained on the official demonstrations; the monitored dual-representation heads are evaluated offline and do not affect the rollouts. The dataset contains 1664 episodes and 283{,}375 steps. Five base tasks are used for calibration and five for testing; the test split contains 851 episodes, including 369 failures. AUROCs are averaged over 20 perturbation-dimension--severity strata with at least five successes and five failures. The monitor protocol, including the episode-level primary comparison, baseline set, stratification, aggregation, and mechanism tests, was preregistered before the first scoring output; the full specification is given in \ref{app:liberoplus}.

All test failures are 250-step timeouts, whereas successful episodes terminate earlier. Elapsed step count alone consequently reaches AUROC 0.9969, so the preregistered episode-level aggregate is length-confounded. After identifying this confound, we added a fixed-horizon analysis and treat it as the primary interpretable comparison, not as a preregistered success criterion. Truncating every episode to its first 100 steps gives AUROC 0.6200 for the coordinated residual versus 0.5285 for $\max|z|$, a margin of $+0.0915$ with paired-bootstrap lower bound $+0.0058$; the uncoordinated residual reaches 0.6856. A preregistered sweep over four further horizons (Table~\ref{tab:horizons}) gives the stability question a split answer. The family-best signal, the uncoordinated residual, exceeds every observation-side baseline at every horizon. The coordinated arm's margin is horizon-sensitive: it sits 0.026 below random-network distillation at 50 steps and 0.002 below at 75, and at 150 steps, where only 48 successful episodes survive truncation, it falls below action-chunk entropy. Statements about the residual family therefore hold across horizons through its uncoordinated arm, while the coordinated arm's advantage over observation-side baselines is specific to the 100--125-step range.

\begin{table}[t]
\centering
\caption{Fixed-horizon detection AUROC (stratified, EWMA-max aggregation, monitor seed 0) as the truncation length $K$ varies. Failures are 250-step timeouts, so all 369 remain at every $K$; successes shorter than $K$ drop out.}
\label{tab:horizons}
\footnotesize
\begin{tabular}{@{}r rr rrrrr@{}}
\toprule
$K$ & fail & succ & coord.\ residual & uncoord.\ residual & $\max|z|$ & RND & ACE \\
\midrule
50  & 369 & 482 & 0.4878 & \textbf{0.5226} & 0.4808 & 0.5137 & 0.3080 \\
75  & 369 & 482 & 0.5740 & \textbf{0.6212} & 0.5352 & 0.5756 & 0.3590 \\
100 & 369 & 312 & 0.6200 & \textbf{0.6856} & 0.5285 & 0.5662 & 0.5458 \\
125 & 369 & 99  & 0.6139 & \textbf{0.7309} & 0.5237 & 0.6066 & 0.5269 \\
150 & 369 & 48  & 0.7667 & \textbf{0.9000} & 0.3333 & 0.4000 & 0.8667 \\
\bottomrule
\end{tabular}
\end{table} Table~\ref{tab:liberoplus} reports the preregistered episode-level aggregation as a secondary, confounded analysis.

\begin{table}[t]
\centering
\caption{Secondary episode-level analysis on 851 test episodes (369 timeout failures), stratified over 20 perturbation--severity cells, for the first monitor seed. Values are affected by episode length and are used only for relative ordering.}
\label{tab:liberoplus}
\footnotesize
\begin{tabular}{@{}l l r@{}}
\toprule
Signal & Family & Stratified AUROC \\
\midrule
uncoordinated residual, 16-sample mean & cross-representation & 0.8780 \\
uncoordinated residual & cross-representation & 0.8479 \\
coordinated residual, 16-sample mean & cross-representation & 0.8401 \\
coordinated residual & cross-representation & 0.8209 \\
\midrule
random-network distillation \citep{romer2025fiper} & observation side & 0.7524 \\
$\max|z|$ & observation side & 0.7066 \\
action-chunk entropy \citep{romer2025fiper} & action distribution & 0.5866 \\
\midrule
sampling variance, uncoordinated head & same-head dispersion & 0.5975 \\
sampling variance, coordinated head & same-head dispersion & 0.5621 \\
sampling variance, $\pi_{\mathrm{exec}}$ & same-head dispersion & 0.5378 \\
\bottomrule
\end{tabular}
\end{table}

On the episode-level analysis, the coordinated residual exceeds $\max|z|$ by $+0.1143$, with 95\% CI $[+0.0650,+0.1810]$; the corresponding residual scores from two additional monitor seeds are 0.8104 and 0.8319. The fixed-horizon result is more conservative and is the value used in the abstract and conclusions. The residual requires 11.70\,ms per step, compared with 0.0090\,ms for $\max|z|$; this cost--performance trade-off is not resolved by the present experiment.

\paragraph{Scoring correction}
An initial implementation standardized the condition vector twice before computing $\max|z|$, which weakened the comparison baseline and biased every margin quoted against it in the residual's favor. All $\max|z|$ comparisons were recomputed from the raw conditions under the specified normalization, and only the corrected values appear above and in \ref{app:liberoplus}; the correction reduced the margins without changing the ordering of the signal families. The false-alarm comparison, the monitored-head sampling variances, and the action-chunk-entropy calibration do not involve $\max|z|$ and were unaffected.

\paragraph{False alarms at multimodal steps}
\label{sec:fa-mechanism}
The prespecified multimodality flag identifies 17{,}140 of 55{,}279 successful test steps (31.0\%). At signal-specific conformal thresholds calibrated to 5\% false alarms, both cross-representation residuals alarm more often on multimodal than on unimodal steps. The coordinated residual's surcharge is 1.57 percentage points, with 95\% CI $[+0.70,+2.46]$; the observation-side $\max|z|$ control has a gap of $-0.01$ points (Table~\ref{tab:fagap}). This pattern links the surcharge specifically to representation disagreement rather than to multimodal states being uniformly harder.

\begin{table}[t]
\centering
\caption{False-alarm rates on successful test episodes, split by the prespecified multimodality flag. Thresholds are calibrated separately for each signal at 5\%; intervals use episode-level cluster bootstrap.}
\label{tab:fagap}
\footnotesize
\begin{tabular}{@{}l rrrr@{}}
\toprule
Signal & FA multimodal & FA unimodal & Gap (pp) & 95\% CI (pp) \\
\midrule
coordinated residual & 9.75\% & 8.19\% & $+1.57$ & $[+0.70,+2.46]$ \\
uncoordinated residual & 8.52\% & 6.98\% & $+1.54$ & $[+0.71,+2.46]$ \\
$\max|z|$ & 5.41\% & 5.42\% & $-0.01$ & $[-1.17,+1.30]$ \\
random-network distillation & 5.33\% & 6.38\% & $-1.05$ & $[-2.17,+0.14]$ \\
action-chunk entropy & 5.23\% & 5.63\% & $-0.39$ & $[-1.38,+0.51]$ \\
\bottomrule
\end{tabular}
\end{table}

The preregistered mechanism prediction was that coordination reduces this surcharge. Its registered success criterion was not met: the coordinated-minus-uncoordinated gap difference is $-0.03$ percentage points with 95\% CI $[-1.00,+0.86]$, an interval that leaves both no effect and a partial reduction open. The detection contrast is less precise. Coordinated minus uncoordinated AUROC averages $+0.0052\pm0.0288$ across the three monitor seeds, and the ranking reverses across seeds. This comparison is inconclusive at the seed level: it supports only the limited statement that no robust gain was observed, and it neither refutes a detection benefit nor establishes equivalence between the arms. In all three seeds, both residual arms remain above the observation-side and same-head dispersion baselines.

The attainable-region bound alone does not predict this null because it is an upper bound rather than a lower bound on coordination. The following measurement characterizes the regime in which the comparison was made.

\subsection{Modal capacity of the external monitored heads}
\label{sec:liberoplus-coverage}

This capacity analysis was designed after the monitoring result and was not part of the preregistered protocol. On the same 851 test episodes and frozen traces, 32 action chunks are drawn from each monitored branch at every one of 147{,}529 test steps. Coordination is measured with the paper's per-condition both-valid estimator under two partitions: the training-time token partition used by the coordinated model and a local two-cluster partition used by the multimodality flag. The analysis uses three training seeds and two independent Monte Carlo draws. It is used to characterize the regime and generate a falsifiable explanation, not to convert the failed mechanism prediction into a confirmatory result.

The demonstration neighborhoods contain a clear two-cluster structure: at flagged steps, their endpoint centroids are 0.163\,m apart with within-cluster dispersion 0.078\,m. The monitored heads cover little of this structure. Their sampled chunks span roughly 0.014--0.016\,m between centroids, and only 1.8--2.0\% of samples occupy the minority demonstration cluster. The resulting token-partition capacity is $C=0.0055$, compared with 0.4810 for the corresponding robot-side learned-prior token family.

\begin{table}[t]
\centering
\caption{Coordination on the LIBERO-Plus test split. Each test step is a condition; estimates use 32 samples per branch, both-valid marginals, and a three-seed mean. The last two rows retain only steps whose token-prior entropy exceeds 0.30 nats.}
\label{tab:liberoplus-coord}
\footnotesize
\begin{tabular}{@{}l rrr rrr@{}}
\toprule
& \multicolumn{3}{c}{Token partition} & \multicolumn{3}{c}{Local partition} \\
\cmidrule(lr){2-4}\cmidrule(l){5-7}
Arm & $\Delta$ & $C$ & $\rho$ & $\Delta$ & $C$ & $\rho$ \\
\midrule
coordinated & $\mathbf{+0.000135}$ & 0.0055 & $+2.5\%$ & $+0.000020$ & 0.0281 & $+0.1\%$ \\
\quad independent tokens & $-0.000007$ & 0.0055 & $-0.1\%$ & $+0.000004$ & 0.0282 & $+0.0\%$ \\
\quad shuffled pairing & $-0.000003$ & 0.0055 & $-0.1\%$ & $-0.000072$ & 0.0281 & $-0.3\%$ \\
\quad uncoordinated network & $+0.000012$ & 0.0068 & $+0.2\%$ & $+0.000049$ & 0.0309 & $+0.2\%$ \\
\midrule
\multicolumn{7}{@{}l}{\emph{prior entropy $>0.30$ nats}} \\
coordinated & $\mathbf{+0.00712}$ & 0.0315 & 22.6\% & $\mathbf{+0.00174}$ & 0.0144 & 12.1\% \\
\quad uncoordinated network & $+0.00034$ & 0.0225 & 1.5\% & $-0.00004$ & 0.0113 & $-0.3\%$ \\
\bottomrule
\end{tabular}
\end{table}

The coordinated model has a small but reproducible token-partition effect, $\Delta=0.000135\pm0.000040$, while all controls are close to zero. Its absolute size is limited by three factors. First, the capacity is low. Second, the learned token prior has conditional entropy 0.011--0.013 nats, with median zero, so the token is nearly deterministic given the condition. Third, when the token does vary, each branch follows it only weakly: on the 1.5--1.8\% of steps with prior entropy above 0.30 nats, token obedience is 53.0--55.9\% against a 51\% base rate. Restricting to these steps raises the capacity to 0.0315 and the effect to 0.00712, or $\rho=22.6\%$.

Low capacity and low token entropy are both sufficient to suppress the average effect, and the experiment cannot identify which is the dominant bottleneck. It therefore does not test the utility of coordination in a policy that already covers multiple modes with a genuinely stochastic shared symbol. Within the measured regime, the practical result is narrower: multimodality adds a specific false-alarm surcharge, the residual family remains the strongest evaluated signal at every tested horizon through its uncoordinated arm (the coordinated arm's margin holds only near 100--125 steps), and the available coordination intervention produces no robust monitoring improvement.

\section{Discussion}
\label{sec:discussion}

\subsection{Interpreting the attainable region}
\label{sec:claims}

The attainable region separates quantities that a mismatch rate alone combines.  The vertical coordinate, $\Delta=\mathrm{BL}-\mathrm{obs}$, measures same-label agreement beyond that implied by the two branch marginals.  The horizontal coordinate, $C_K=\sqrt{G(q)G(p)}$, is an outer capacity bound determined by the two modal marginals.  Valid-pair rate and behavioral-prior fidelity remain separate outcomes.  Reporting the components rather than one compressed score is the same design requirement that produces robustness conditions on multitask and task-similarity metrics \citep{nn2024penwarden,nn2025liu}.  This distinction is essential at both ends of the diagonal: a posterior can report zero mismatch while collapsing $C_K$ toward zero, whereas a balanced model can retain near-maximal capacity while remaining uncoordinated.

The mechanism study gives the region a concrete interpretation.  An auxiliary latent that is absent from the construction of the CFM state and target is ignored at the population squared-loss optimum; under branch-local sampling this yields $\Delta=0$.  Shared source noise is a different object because it determines the interpolation path and the sampling initial condition.  Its effect is therefore set by the alignment of the two noise-to-sample maps and changes sign across the representation maps studied here.  Endpoint consistency creates an interior effect, but at a substantial cost in valid-pair rate and only when a common source-noise coordinate is available.  Low-cardinality partition symbols are the most reliable route to the high-capacity corner in the robot experiments.  They are not a necessary route in general: the amortized posterior reaches the same region on the balanced testbed. Across these successful corner cases, the shared variable is informative during training and usable over its test-time support.

The skewed testbed further separates coordination from fidelity.  A uniform symbol marginal can keep $C_K$ high while distorting the generated mode frequencies; matching the symbol marginal to the data largely repairs that distortion and can also improve coordination when the branches otherwise override rare symbols.  Thus the top-right corner is not, by itself, a complete design objective.  A coordinated policy should be reported together with its valid-pair rate and with the discrepancy between generated and reference mode frequencies.

The external monitor experiment bounds the practical reading.  Benign multimodality increases false alarms for cross-representation residuals, and the residual remains the strongest evaluated signal family at a fixed 100-step horizon.  The token intervention did not meet its registered false-alarm criterion, and the detection contrast, whose sign changes across training seeds, is inconclusive.  The monitored heads have both very low modal capacity and an almost deterministic token prior, so the intervention creates little first-stage change in branch behavior; the experiment therefore constrains this regime rather than settings where the heads have modes to reconcile.

\subsection{Design implications}
\label{sec:waysout}

The results point to three distinct ways of handling cross-representation disagreement.  The first is to share part of the generative state rather than append an auxiliary random input.  A common trunk or canonical latent representation removes the conditional-independence structure used by Lemma~\ref{lem:zblind}; Mixture of Frames follows this route by fusing frame-specific information into one diffusion state \citep{wang2026mof}.  Such designs can improve consistency, but they also reduce or eliminate the residual that would otherwise be used as a monitor.

The second is to couple the samplers.  Hierarchical Diffusion Policy refines a joint-space trajectory toward the pose-space prediction through differentiable forward kinematics at inference \citep{ma2024hdp}.  Common random numbers provide a softer version: across models with closely aligned noise-to-sample maps, sharing $x_0$ removes much of the cross-mode term while retaining 62--93\% of the same-mode component (Section~\ref{sec:crn}).  Across heterogeneous representations, however, the same construction has no predictable sign unless the latent coordinates have a meaningful correspondence.

The third is to introduce an explicit low-cardinality symbol.  The robot experiments suggest that symbol granularity, rather than coordination strength alone, controls the residual floor: window-level symbols approach the oracle floor more closely than episode-level symbols.  A useful design would therefore choose the symbol at the shortest temporal scale that remains stable, train both branches to respond to every symbol value, and sample the symbol from a marginal consistent with the desired behavior distribution.  This is a stronger requirement than merely supplying the same random seed, but it does not require the symbol to coincide with a human-interpretable task label.

\subsection{Relation to prior work}
\label{sec:related}

\paragraph{Multiple action representations.}
Joint-space and task-space action representations have long been combined for accuracy and kinematic consistency.  Implicit Kinematic Policies sample one joint-space action and compute Cartesian quantities deterministically \citep{ganapathi2022ikp}; there is no pair of independently sampled representations and hence no cross-branch residual of the kind studied here.  Hierarchical Diffusion Policy uses separate pose and joint diffusion models with training-time distillation and inference-time kinematic refinement \citep{ma2024hdp}.  MimicIK uses a differentiable-forward-kinematics consistency loss with conditional flow matching \citep{yang2026mimicik}, while Mixture of Frames maintains a shared canonical generative state \citep{wang2026mof}.  These methods treat consistency as an accuracy or inference problem.  The present analysis instead asks what an independently sampled cross-space residual measures and when its benign component can be reduced without removing the signal itself.  The two questions are complementary; the results do not argue against multi-representation policies, whose representations can provide complementary predictive structure \citep{feng2026actionspace}.

\paragraph{Runtime monitoring.}
Sentinel combines temporal consistency with a vision--language model's progress estimate \citep{agia2024sentinel}; FAIL-Detect formulates monitoring as sequential out-of-distribution detection \citep{xu2025faildetect}; and FIPER combines observation-side random-network distillation with action-chunk entropy \citep{romer2025fiper}.  VFD/SAVE uses ensemble velocity-field disagreement and a coupled evaluation construction \citep{romer2026vfd}, with related recent work examining action-space probes and trajectory-derived warning signals \citep{actprobe2026,hideandseek2026}.  Section~\ref{sec:liberoplus} compares signal constructions on one common log rather than reproducing each method's original training and tuning protocol.  The resulting ranking should therefore be read as a comparison of signals under a shared evaluation, not as a leaderboard among complete monitoring systems.  Benchmarks of conditional generative models under one protocol draw the same line between comparing constructions and comparing separately tuned systems \citep{nn2026kohl}.

\paragraph{Multimodality and disagreement.}
Diff-DAgger notes that ensemble disagreement can be high at multimodal decision points even when several actions are valid \citep{lee2025diffdagger}.  The present work studies the same failure mode across representations of one policy and provides a chance-corrected scale for it.  The ensemble literature has long used kappa and related pairwise diversity measures \citep{margineantu1997,kuncheva2003}; the semantic difference is that disagreement between equally valid modes is a nuisance for a residual monitor rather than a desirable source of ensemble diversity.

\paragraph{Kappa--error diagrams and their feasible region.}
The feasible region of a kappa--error diagram has been derived before.  \citet{kuncheva2013bound} bounds kappa as a function of the pair's individual accuracies, delimits the attainable part of the kappa--error plane, and reports that part of that area is left unoccupied by real ensembles.  Kappa is computed there on the $2\times2$ correct/wrong table of two classifiers scored against ground truth.  That paper names two perspectives at the outset, agreement in assigning class labels and agreement in assigning the correct label, notes that both produce the same qualitative pattern, and takes the second because it lends itself to the algebraic manipulations.  The present calculus develops the first perspective, and the substitution changes what the horizontal coordinate certifies.  $C_K=\sqrt{G(q)G(p)}$ is a function of each branch's own modal marginal and requires no correctness label, so a small $C_K$ reports low output diversity rather than low accuracy, and zero mismatch at small $C_K$ identifies collapse.  A bound computed against ground truth is available only where correctness labels are; $C_K$ is computed from the branches' own outputs, which is what lets the certificate run at deployment rather than in post-hoc analysis.  The region is additionally conditioned on both branches returning valid trajectories, which the correct/wrong construction does not require, and it is used here to exclude candidate designs for generative policies rather than to characterize an ensemble already built.  \citet{kuncheva2003} find pairwise diversity measures only weakly related to ensemble accuracy; what coordination is tied to here is the composition of the residual, and the single performance comparison is the monitoring experiment of Section~\ref{sec:liberoplus}.

\paragraph{Latents and coupled noise.}
Variational Rectified Flow Matching uses a data-dependent posterior during training to make the latent informative to a velocity field \citep{guo2025vrfm}; hierarchical flow-matching models provide another route to multimodal structure \citep{hrfm2025}.  Lemma~\ref{lem:zblind} concerns the complementary case in which an auxiliary latent is mean-irrelevant to the CFM target.  Initial-noise coupling can nevertheless alter the joint law while preserving each marginal \citep{jia2026couple}.  Multi-view generation often benefits because the branches share a common spatial coordinate system or exchange features explicitly \citep{dong2025noisecontroller,liu2024syncdreamer}.  The map study here shows why that success does not transfer automatically to heterogeneous action representations: the sign of common-random-number coordination follows the alignment of the two noise-to-sample maps.

\subsection{Testable implications}
\label{sec:falsifiability}

Three implications are especially useful for subsequent work.  First, an auxiliary latent satisfying the conditions of Lemma~\ref{lem:zblind} should remain within the experimental zero band even when model size, mode count, or mode balance changes; the corresponding test is a sweep across many trained models, in the manner of the re-examination of long-term dependence in recurrent networks \citep{nn2025johnston}, and a reproducible nonzero effect would indicate that the latent enters the target construction, the loss, or the sampling dependence in an unaccounted way.  Second, source-noise coupling should strengthen as the two representations acquire better aligned noise-to-sample maps and may become anti-coordinating when those maps are oppositely aligned.  This can be tested directly by interpolating between shared and representation-specific latent coordinates.

Third, the utility of coordination should be evaluated where it can actually change the sampled modes.  A decisive monitor experiment therefore requires non-negligible per-condition modal capacity, a shared symbol with appreciable conditional entropy, and measurable obedience of both branches to that symbol.  Under those conditions, a reduction in false alarms without a loss of failure sensitivity would establish practical value that the LIBERO-Plus experiment here could not test.  Conversely, when the separation between valid modes is below the same-mode residual floor, the benign surcharge itself should largely disappear.  A concise set of follow-up tests is listed in \ref{app:falsify}.

\subsection{Limitations}
\label{sec:limitations}

The two robot environments are deliberately controlled and closely related.  They use one simulator, one embodiment, state rather than visual conditioning, two constructed modes, and relatively small flow-matching networks.  RotatingBypass removes the fixed-coordinate shortcut of BimodalBypass but is not an independent application domain.  The orientation-head experiment covers only a small set of configurations, and all robot-side conclusions about the attainable region are limited to $K=2$.  The non-robotic testbed broadens the representation maps, mode counts, and priors, but its trajectories and validity mechanism are much simpler than manipulation dynamics.

The robot-side mechanism study is not a fully factorial causal experiment; the preregistered map-by-coupling factorial exists only on the testbed. A preregistered seed expansion (Table~\ref{tab:seeds}) brings the claim-bearing families to six to ten training seeds with the seed as the inference unit: the floors pass seed-level equivalence tests at $\pm2\%$, the interior ladder is monotone with its $\lambda{=}1$ rung at five seeds, the corner families hold at the exact caliber, and source-noise instability is measured rather than assumed (eight v2 seeds spanning $-4.3\%$ to $+4.7\%$).  Some cells remain thin: $\beta{=}0.1$ rows, oracle configurations, pose configurations, and third-representation runs are single-seed, and the posterior family has five runs over three KL scales.  Bootstrap intervals over evaluation conditions still quantify only per-checkpoint variation.  Seventeen of the original robot runs can be reconstructed at the exact both-valid marginal convention while the remaining fifty retain the head-valid approximation; every seed-expansion run stores both-valid contingency marginals natively, as the original pipeline should have.

The external experiment covers one LIBERO suite and one executing-policy family.  Every test failure is a 250-step timeout, which makes episode-level aggregates strongly length-confounded; the fixed 100-step analysis avoids that particular confound, was added after the length problem was identified, and has since been repeated at four further horizons with a split outcome: the family ordering is horizon-stable while the coordinated arm's margin is not (Table~\ref{tab:horizons}).  The comparison uses signals re-evaluated on a common frozen log rather than third-party implementations tuned by their authors.  Most importantly, the coordinated monitored heads exhibit very low modal capacity and token entropy, so the experiment does not determine whether coordination would help a policy that genuinely samples several modes.  The preregistered false-alarm test is graded against its registered criterion, which was not met; the detection contrast remains limited by three training seeds.  The capacity-and-entropy explanation was measured after the monitoring outcome and is treated as post hoc.

\subsection{Statistical scope}
\label{sec:stats-honesty}

The inferential unit should match the claim.  Condition-level bootstrap intervals describe uncertainty over the evaluated condition set, whereas comparisons between training procedures should also be read against the observed seed spread.  This distinction is material for endpoint consistency and for the coordinated-versus-uncoordinated detection contrast.  Equivalence tests for the auxiliary-latent family are performed on the normalized effect $\rho$ and establish that the finite-sample effect lies within the prespecified practical band; they do not turn an empirical estimate into an exact zero.  The exact zero is the population prediction under Lemma~\ref{lem:zblind} and Corollary~\ref{cor:zblind}.

\section{Conclusion}
\label{sec:conclusion}

Cross-representation disagreement in a multimodal policy cannot be interpreted from the mismatch rate alone.  The chance-corrected effect
$\Delta=\mathrm{BL}-\mathrm{obs}$ separates agreement beyond the branches' own mode marginals, and the bound
$|\Delta|\le C_K=\sqrt{G(q)G(p)}$ places that effect on an attainable region whose horizontal coordinate is modal diversity.  Reporting the effect together with $C_K$ distinguishes genuine coordination from both chance agreement and mode collapse; valid-pair rate and behavioral-prior fidelity remain separate quantities.

The mechanism study identifies several qualitatively different ways to occupy this region.  An auxiliary latent that does not participate in constructing the CFM interpolation state or target is ignored by the population squared-loss optimum.  With branch-local sampling, its predicted coordination effect is exactly zero, and the finite experiments place the effect within the prespecified equivalence band.  Shared source noise is not covered by that result: it selects a coupling, and its effect changes from negative to positive as the relation between the two representation maps changes, a reversal confirmed in a preregistered factorial at ten seeds per cell.  Endpoint consistency produces an interior coordination effect, but only in conjunction with shared source noise and with a substantial reduction in valid-pair rate.

Low-cardinality partition symbols provide the strongest and most reliable robot-side coordination.  They need not coincide with semantic task labels, and they need not be inferred from the observation at test time.  Their marginal and temporal granularity nevertheless matter: a poorly matched marginal distorts the generated mode frequencies under skew, while a coarse symbol leaves a larger same-mode residual floor.  The balanced testbed also shows that an explicit partition token is not the only possible route to the high-capacity corner, because an amortized posterior reaches the same region without collapse.  What the successful corner cases share is a variable that is informative during training and remains usable over its test-time support.

The monitoring results put a practical boundary around these mechanism findings.  On LIBERO-Plus, benign multimodality raises the cross-representation residual's false-alarm rate by 1.57 percentage points relative to unimodal steps, while observation-side signals show no corresponding surcharge.  The residual nevertheless remains the strongest evaluated signal family at a fixed 100-step horizon.  The preregistered prediction that sharing the token would reduce the surcharge did not meet its registered criterion in this regime; the effect on failure detection is smaller than the variation across three training seeds and remains inconclusive.  The later capacity analysis finds modal capacity 0.0055 and an almost deterministic token prior, but because that analysis was designed after the monitoring outcome, it characterizes the regime rather than confirming the original prediction.  The supported practical statement is therefore narrow: no robust monitoring gain was observed where the shared symbol had almost no modal variation to organize.

The broader conclusion is therefore not that disagreement should be removed, nor that coordination is without value.  It is that coordination has a measurable capacity, a mechanism, and a cost.  A useful design or monitor should state all three: how much same-mode agreement exceeds chance, how much modal diversity makes that agreement possible, and what is paid in validity or behavioral fidelity.  The attainable region provides a compact way to make those distinctions before a low residual is interpreted as success.

\section*{Declaration of competing interest}
The authors declare that they have no known competing financial interests or
personal relationships that could have appeared to influence the work reported
in this paper.

\section*{Acknowledgements}
This work was supported in part by the Jiangsu Provincial Major Science and
Technology Program under Grant BG2025043, in part by the National Natural
Science Foundation of China under Grant 52502402, in part by the Basic Research
Program of Jiangsu Province under Grant BK20251462, and in part by the Natural
Science Foundation of Jiangsu Province under Grant BK20232038.

\section*{Data availability}
All numerical results are computed from frozen evaluation products. The code is
available at \url{https://github.com/kimo423/dual-head-coordination} under an
MIT licence: the dual-head model family, the testbed and its representation
maps, the differentiable forward kinematics, the external monitor pipeline, one
configuration file per training run reported here, and the regression suite that
pins the effect-size algorithms. Gate files store per-condition sufficient
statistics, so every effect, capacity, occupancy and bootstrap interval in the
paper is recomputable from the stored counts without access to model
checkpoints, which are therefore not required to reproduce any number reported.
The evaluation products are released in the same repository: the per-point
coordinate tables behind each figure, the frozen gate files for all robot-side
and testbed runs, the preregistration documents with the commit hashes and
timestamps that place each before the runs it governs, and the LIBERO-Plus
per-step scoring log for all 1664 episodes. The two testbed datasets (675\,MB
and 945\,MB) are regenerated by the deterministic seeded generators shipped with
the code rather than deposited.

\printcredits

\bibliographystyle{cas-model2-names}
\bibliography{refs_min}

\begin{thebibliography}{51}
\expandafter\ifx\csname natexlab\endcsname\relax\def\natexlab#1{#1}\fi
\providecommand{\url}[1]{\texttt{#1}}
\providecommand{\href}[2]{#2}
\providecommand{\path}[1]{#1}
\providecommand{\DOIprefix}{doi:}
\providecommand{\ArXivprefix}{arXiv:}
\providecommand{\URLprefix}{URL: }
\providecommand{\Pubmedprefix}{pmid:}
\providecommand{\doi}[1]{\href{http://dx.doi.org/#1}{\path{#1}}}
\providecommand{\Pubmed}[1]{\href{pmid:#1}{\path{#1}}}
\providecommand{\bibinfo}[2]{#2}
\ifx\xfnm\relax \def\xfnm[#1]{\unskip,\space#1}\fi
\bibitem[{Agia et~al.(2024)Agia, Sinha, Yang, Cao, Antonova, Pavone and
  Bohg}]{agia2024sentinel}
\bibinfo{author}{Agia, C.}, \bibinfo{author}{Sinha, R.}, \bibinfo{author}{Yang,
  J.}, \bibinfo{author}{Cao, Z.a.}, \bibinfo{author}{Antonova, R.},
  \bibinfo{author}{Pavone, M.}, \bibinfo{author}{Bohg, J.},
  \bibinfo{year}{2024}.
\newblock \bibinfo{title}{Unpacking failure modes of generative policies:
  runtime monitoring of consistency and progress}, in:
  \bibinfo{booktitle}{Conference on Robot Learning (CoRL)}.
\newblock \bibinfo{note}{ArXiv:2410.04640}.
\bibitem[{Agresti(1992)}]{agresti1992}
\bibinfo{author}{Agresti, A.}, \bibinfo{year}{1992}.
\newblock \bibinfo{title}{Modelling patterns of agreement and disagreement}.
\newblock \bibinfo{journal}{Statistical Methods in Medical Research}
  \bibinfo{volume}{1}, \bibinfo{pages}{201--218}.
\bibitem[{Byrt et~al.(1993)Byrt, Bishop and Carlin}]{byrt1993}
\bibinfo{author}{Byrt, T.}, \bibinfo{author}{Bishop, J.},
  \bibinfo{author}{Carlin, J.B.}, \bibinfo{year}{1993}.
\newblock \bibinfo{title}{Bias, prevalence and kappa}.
\newblock \bibinfo{journal}{Journal of Clinical Epidemiology}
  \bibinfo{volume}{46}, \bibinfo{pages}{423--429}.
\bibitem[{Chi et~al.(2023)Chi, Feng, Du, Xu, Cousineau, Burchfiel and
  Song}]{chi2023diffusion}
\bibinfo{author}{Chi, C.}, \bibinfo{author}{Feng, S.}, \bibinfo{author}{Du,
  Y.}, \bibinfo{author}{Xu, Z.}, \bibinfo{author}{Cousineau, E.},
  \bibinfo{author}{Burchfiel, B.}, \bibinfo{author}{Song, S.},
  \bibinfo{year}{2023}.
\newblock \bibinfo{title}{Diffusion policy: visuomotor policy learning via
  action diffusion}, in: \bibinfo{booktitle}{Robotics: Science and Systems
  (RSS)}.
\newblock \bibinfo{note}{ArXiv:2303.04137}.
\bibitem[{Cicchetti and Feinstein(1990)}]{cicchetti1990}
\bibinfo{author}{Cicchetti, D.V.}, \bibinfo{author}{Feinstein, A.R.},
  \bibinfo{year}{1990}.
\newblock \bibinfo{title}{High agreement but low kappa: {II}. {R}esolving the
  paradoxes}.
\newblock \bibinfo{journal}{Journal of Clinical Epidemiology}
  \bibinfo{volume}{43}, \bibinfo{pages}{551--558}.
\bibitem[{Cohen(1960)}]{cohen1960}
\bibinfo{author}{Cohen, J.}, \bibinfo{year}{1960}.
\newblock \bibinfo{title}{A coefficient of agreement for nominal scales}.
\newblock \bibinfo{journal}{Educational and Psychological Measurement}
  \bibinfo{volume}{20}, \bibinfo{pages}{37--46}.
\bibitem[{Cureton(1959)}]{cureton1959}
\bibinfo{author}{Cureton, E.E.}, \bibinfo{year}{1959}.
\newblock \bibinfo{title}{Note on $\varphi/\varphi_{\max}$}.
\newblock \bibinfo{journal}{Psychometrika} \bibinfo{volume}{24},
  \bibinfo{pages}{89--91}.
\bibitem[{Dong et~al.(2025)Dong, Wang, Lin, Wu, Chen, Liu, Yang, Li and
  Guo}]{dong2025noisecontroller}
\bibinfo{author}{Dong, H.}, \bibinfo{author}{Wang, X.}, \bibinfo{author}{Lin,
  D.}, \bibinfo{author}{Wu, Y.}, \bibinfo{author}{Chen, Q.},
  \bibinfo{author}{Liu, R.}, \bibinfo{author}{Yang, K.}, \bibinfo{author}{Li,
  P.}, \bibinfo{author}{Guo, Q.}, \bibinfo{year}{2025}.
\newblock \bibinfo{title}{{NoiseController}: towards consistent multi-view
  video generation via noise decomposition and collaboration}, in:
  \bibinfo{booktitle}{IEEE/CVF International Conference on Computer Vision
  (ICCV)}, pp. \bibinfo{pages}{14443--14452}.
\newblock \bibinfo{note}{ArXiv:2504.18448}.
\bibitem[{Fei et~al.(2025)Fei, Wang, Shi, Dai, Cai, Qian, Ji, He, Zhang, Fei,
  Fu, Gong and Qiu}]{fei2025liberoplus}
\bibinfo{author}{Fei, S.}, \bibinfo{author}{Wang, S.}, \bibinfo{author}{Shi,
  J.}, \bibinfo{author}{Dai, Z.}, \bibinfo{author}{Cai, J.},
  \bibinfo{author}{Qian, P.}, \bibinfo{author}{Ji, L.}, \bibinfo{author}{He,
  X.}, \bibinfo{author}{Zhang, S.}, \bibinfo{author}{Fei, Z.},
  \bibinfo{author}{Fu, J.}, \bibinfo{author}{Gong, J.}, \bibinfo{author}{Qiu,
  X.}, \bibinfo{year}{2025}.
\newblock \bibinfo{title}{{LIBERO}-plus: in-depth robustness analysis of
  vision--language--action models}.
\newblock \bibinfo{note}{Preprint, arXiv:2510.13626}.
\bibitem[{Feinstein and Cicchetti(1990)}]{feinstein1990}
\bibinfo{author}{Feinstein, A.R.}, \bibinfo{author}{Cicchetti, D.V.},
  \bibinfo{year}{1990}.
\newblock \bibinfo{title}{High agreement but low kappa: {I}. {T}he problems of
  two paradoxes}.
\newblock \bibinfo{journal}{Journal of Clinical Epidemiology}
  \bibinfo{volume}{43}, \bibinfo{pages}{543--549}.
\bibitem[{Feng et~al.(2026)Feng, Zheng, Wang, Liu, Li, Pang, Wang and
  Zhan}]{feng2026actionspace}
\bibinfo{author}{Feng, Y.}, \bibinfo{author}{Zheng, J.}, \bibinfo{author}{Wang,
  Z.}, \bibinfo{author}{Liu, D.}, \bibinfo{author}{Li, J.},
  \bibinfo{author}{Pang, J.}, \bibinfo{author}{Wang, T.},
  \bibinfo{author}{Zhan, X.}, \bibinfo{year}{2026}.
\newblock \bibinfo{title}{Demystifying action space design for robotic
  manipulation policies}.
\newblock \bibinfo{note}{Preprint, arXiv:2602.23408}.
\bibitem[{Ganapathi et~al.(2022)Ganapathi, Florence, Varley, Burns, Goldberg
  and Zeng}]{ganapathi2022ikp}
\bibinfo{author}{Ganapathi, A.}, \bibinfo{author}{Florence, P.},
  \bibinfo{author}{Varley, J.}, \bibinfo{author}{Burns, K.},
  \bibinfo{author}{Goldberg, K.}, \bibinfo{author}{Zeng, A.},
  \bibinfo{year}{2022}.
\newblock \bibinfo{title}{Implicit kinematic policies: unifying joint and
  {C}artesian action spaces in end-to-end robot learning}, in:
  \bibinfo{booktitle}{IEEE International Conference on Robotics and Automation
  (ICRA)}.
\newblock \bibinfo{note}{ArXiv:2203.01983}.
\bibitem[{Gorodkin(2004)}]{gorodkin2004}
\bibinfo{author}{Gorodkin, J.}, \bibinfo{year}{2004}.
\newblock \bibinfo{title}{Comparing two {K}-category assignments by a
  {K}-category correlation coefficient}.
\newblock \bibinfo{journal}{Computational Biology and Chemistry}
  \bibinfo{volume}{28}, \bibinfo{pages}{367--374}.
\bibitem[{Guo and Schwing(2025)}]{guo2025vrfm}
\bibinfo{author}{Guo, P.}, \bibinfo{author}{Schwing, A.G.},
  \bibinfo{year}{2025}.
\newblock \bibinfo{title}{Variational rectified flow matching}, in:
  \bibinfo{booktitle}{Proceedings of the 42nd International Conference on
  Machine Learning (ICML)}, pp. \bibinfo{pages}{20921--20940}.
\newblock \bibinfo{note}{ArXiv:2502.09616}.
\bibitem[{Huang et~al.(2026)Huang, Li, Wang, Mi, Hao, Wang, Wu, Li, Liu and
  Cao}]{actprobe2026}
\bibinfo{author}{Huang, B.}, \bibinfo{author}{Li, X.}, \bibinfo{author}{Wang,
  X.}, \bibinfo{author}{Mi, L.}, \bibinfo{author}{Hao, Z.},
  \bibinfo{author}{Wang, W.}, \bibinfo{author}{Wu, H.}, \bibinfo{author}{Li,
  K.}, \bibinfo{author}{Liu, Y.}, \bibinfo{author}{Cao, T.},
  \bibinfo{year}{2026}.
\newblock \bibinfo{title}{{ActProbe}: action-space probe for early failure
  detection of generative robot policies}.
\newblock \bibinfo{note}{Preprint, arXiv:2606.08508}.
\bibitem[{Hwang et~al.(2025)Hwang, Kang, Eo, Kim and Rhee}]{nn2025hwang}
\bibinfo{author}{Hwang, D.}, \bibinfo{author}{Kang, S.}, \bibinfo{author}{Eo,
  M.}, \bibinfo{author}{Kim, J.}, \bibinfo{author}{Rhee, W.},
  \bibinfo{year}{2025}.
\newblock \bibinfo{title}{Towards a better evaluation of out-of-domain
  generalization}.
\newblock \bibinfo{journal}{Neural Networks} \bibinfo{volume}{188},
  \bibinfo{pages}{107434}.
\newblock \DOIprefix\doi{10.1016/j.neunet.2025.107434}.
\bibitem[{Jia et~al.(2026)Jia, Shen and Wang}]{jia2026couple}
\bibinfo{author}{Jia, J.}, \bibinfo{author}{Shen, L.}, \bibinfo{author}{Wang,
  G.}, \bibinfo{year}{2026}.
\newblock \bibinfo{title}{Couple to control: joint initial noise design in
  diffusion models}.
\newblock \bibinfo{note}{Preprint, arXiv:2605.11311}.
\bibitem[{Johnston et~al.(2025)Johnston, Patel, Cui and
  Balaprakash}]{nn2025johnston}
\bibinfo{author}{Johnston, L.}, \bibinfo{author}{Patel, V.},
  \bibinfo{author}{Cui, Y.}, \bibinfo{author}{Balaprakash, P.},
  \bibinfo{year}{2025}.
\newblock \bibinfo{title}{Revisiting the problem of learning long-term
  dependencies in recurrent neural networks}.
\newblock \bibinfo{journal}{Neural Networks} \bibinfo{volume}{183},
  \bibinfo{pages}{106887}.
\newblock \DOIprefix\doi{10.1016/j.neunet.2024.106887}.
\bibitem[{Jurman et~al.(2012)Jurman, Riccadonna and Furlanello}]{jurman2012}
\bibinfo{author}{Jurman, G.}, \bibinfo{author}{Riccadonna, S.},
  \bibinfo{author}{Furlanello, C.}, \bibinfo{year}{2012}.
\newblock \bibinfo{title}{A comparison of {MCC} and {CEN} error measures in
  multi-class prediction}.
\newblock \bibinfo{journal}{PLoS ONE} \bibinfo{volume}{7},
  \bibinfo{pages}{e41882}.
\bibitem[{Kohl et~al.(2026)Kohl, Chen and Thuerey}]{nn2026kohl}
\bibinfo{author}{Kohl, G.}, \bibinfo{author}{Chen, L.W.},
  \bibinfo{author}{Thuerey, N.}, \bibinfo{year}{2026}.
\newblock \bibinfo{title}{Benchmarking autoregressive conditional diffusion
  models for turbulent flow simulation}.
\newblock \bibinfo{journal}{Neural Networks} \bibinfo{volume}{199},
  \bibinfo{pages}{108641}.
\newblock \DOIprefix\doi{10.1016/j.neunet.2026.108641}.
\bibitem[{Kuncheva(2013)}]{kuncheva2013bound}
\bibinfo{author}{Kuncheva, L.I.}, \bibinfo{year}{2013}.
\newblock \bibinfo{title}{A bound on kappa-error diagrams for analysis of
  classifier ensembles}.
\newblock \bibinfo{journal}{IEEE Transactions on Knowledge and Data
  Engineering} \bibinfo{volume}{25}, \bibinfo{pages}{494--501}.
\newblock \DOIprefix\doi{10.1109/tkde.2011.234}.
\bibitem[{Kuncheva and Whitaker(2003)}]{kuncheva2003}
\bibinfo{author}{Kuncheva, L.I.}, \bibinfo{author}{Whitaker, C.J.},
  \bibinfo{year}{2003}.
\newblock \bibinfo{title}{Measures of diversity in classifier ensembles and
  their relationship with the ensemble accuracy}.
\newblock \bibinfo{journal}{Machine Learning} \bibinfo{volume}{51},
  \bibinfo{pages}{181--207}.
\bibitem[{Lakens(2017)}]{lakens2017}
\bibinfo{author}{Lakens, D.}, \bibinfo{year}{2017}.
\newblock \bibinfo{title}{Equivalence tests: a practical primer for $t$ tests,
  correlations, and meta-analyses}.
\newblock \bibinfo{journal}{Social Psychological and Personality Science}
  \bibinfo{volume}{8}, \bibinfo{pages}{355--362}.
\bibitem[{Lee et~al.(2025)Lee, Kang and Kuo}]{lee2025diffdagger}
\bibinfo{author}{Lee, S.W.}, \bibinfo{author}{Kang, X.}, \bibinfo{author}{Kuo,
  Y.L.}, \bibinfo{year}{2025}.
\newblock \bibinfo{title}{Diff-{DA}gger: uncertainty estimation with diffusion
  policy for robotic manipulation}, in: \bibinfo{booktitle}{IEEE International
  Conference on Robotics and Automation (ICRA)}.
\newblock \bibinfo{note}{ArXiv:2410.14868}.
\bibitem[{Lipman et~al.(2023)Lipman, Chen, Ben-Hamu, Nickel and
  Le}]{lipman2023flow}
\bibinfo{author}{Lipman, Y.}, \bibinfo{author}{Chen, R.T.Q.},
  \bibinfo{author}{Ben-Hamu, H.}, \bibinfo{author}{Nickel, M.},
  \bibinfo{author}{Le, M.}, \bibinfo{year}{2023}.
\newblock \bibinfo{title}{Flow matching for generative modeling}, in:
  \bibinfo{booktitle}{International Conference on Learning Representations
  (ICLR)}.
\newblock \bibinfo{note}{ArXiv:2210.02747}.
\bibitem[{Liu et~al.(2023)Liu, Zhu, Gao, Feng, Liu, Zhu and
  Stone}]{liu2023libero}
\bibinfo{author}{Liu, B.}, \bibinfo{author}{Zhu, Y.}, \bibinfo{author}{Gao,
  C.}, \bibinfo{author}{Feng, Y.}, \bibinfo{author}{Liu, Q.},
  \bibinfo{author}{Zhu, Y.}, \bibinfo{author}{Stone, P.}, \bibinfo{year}{2023}.
\newblock \bibinfo{title}{{LIBERO}: benchmarking knowledge transfer for
  lifelong robot learning}, in: \bibinfo{booktitle}{Advances in Neural
  Information Processing Systems (NeurIPS), Datasets and Benchmarks Track}.
\newblock \bibinfo{note}{ArXiv:2306.03310}.
\bibitem[{Liu et~al.(2025)Liu, Bai, Lu, Soltoggio and Kolouri}]{nn2025liu}
\bibinfo{author}{Liu, X.}, \bibinfo{author}{Bai, Y.}, \bibinfo{author}{Lu, Y.},
  \bibinfo{author}{Soltoggio, A.}, \bibinfo{author}{Kolouri, S.},
  \bibinfo{year}{2025}.
\newblock \bibinfo{title}{Wasserstein task embedding for measuring task
  similarities}.
\newblock \bibinfo{journal}{Neural Networks} \bibinfo{volume}{181},
  \bibinfo{pages}{106796}.
\newblock \DOIprefix\doi{10.1016/j.neunet.2024.106796}.
\bibitem[{Liu et~al.(2024)Liu, Lin, Zeng, Long, Liu, Komura and
  Wang}]{liu2024syncdreamer}
\bibinfo{author}{Liu, Y.}, \bibinfo{author}{Lin, C.}, \bibinfo{author}{Zeng,
  Z.}, \bibinfo{author}{Long, X.}, \bibinfo{author}{Liu, L.},
  \bibinfo{author}{Komura, T.}, \bibinfo{author}{Wang, W.},
  \bibinfo{year}{2024}.
\newblock \bibinfo{title}{{SyncDreamer}: generating multiview-consistent images
  from a single-view image}, in: \bibinfo{booktitle}{International Conference
  on Learning Representations (ICLR)}.
\newblock \bibinfo{note}{ArXiv:2309.03453}.
\bibitem[{Ma et~al.(2024)Ma, Patidar, Haughton and James}]{ma2024hdp}
\bibinfo{author}{Ma, X.}, \bibinfo{author}{Patidar, S.},
  \bibinfo{author}{Haughton, I.}, \bibinfo{author}{James, S.},
  \bibinfo{year}{2024}.
\newblock \bibinfo{title}{Hierarchical diffusion policy for kinematics-aware
  multi-task robotic manipulation}, in: \bibinfo{booktitle}{IEEE/CVF Conference
  on Computer Vision and Pattern Recognition (CVPR)}.
\newblock \bibinfo{note}{ArXiv:2403.03890}.
\bibitem[{Margineantu and Dietterich(1997)}]{margineantu1997}
\bibinfo{author}{Margineantu, D.D.}, \bibinfo{author}{Dietterich, T.G.},
  \bibinfo{year}{1997}.
\newblock \bibinfo{title}{Pruning adaptive boosting}, in:
  \bibinfo{booktitle}{Proceedings of the 14th International Conference on
  Machine Learning (ICML)}, pp. \bibinfo{pages}{211--218}.
\bibitem[{Matthews(1975)}]{matthews1975}
\bibinfo{author}{Matthews, B.W.}, \bibinfo{year}{1975}.
\newblock \bibinfo{title}{Comparison of the predicted and observed secondary
  structure of {T4} phage lysozyme}.
\newblock \bibinfo{journal}{Biochimica et Biophysica Acta (BBA) -- Protein
  Structure} \bibinfo{volume}{405}, \bibinfo{pages}{442--451}.
\bibitem[{Menda et~al.(2019)Menda, Driggs-Campbell and
  Kochenderfer}]{menda2019ensembledagger}
\bibinfo{author}{Menda, K.}, \bibinfo{author}{Driggs-Campbell, K.},
  \bibinfo{author}{Kochenderfer, M.J.}, \bibinfo{year}{2019}.
\newblock \bibinfo{title}{Ensemble{DA}gger: a {B}ayesian approach to safe
  imitation learning}, in: \bibinfo{booktitle}{IEEE/RSJ International
  Conference on Intelligent Robots and Systems (IROS)}.
\newblock \bibinfo{note}{ArXiv:1807.08364}.
\bibitem[{Park et~al.(2026)Park, Li, Oh, Yeh, Kira, Hagenow and
  Li}]{hideandseek2026}
\bibinfo{author}{Park, S.}, \bibinfo{author}{Li, W.}, \bibinfo{author}{Oh, C.},
  \bibinfo{author}{Yeh, S.}, \bibinfo{author}{Kira, Z.},
  \bibinfo{author}{Hagenow, M.}, \bibinfo{author}{Li, S.},
  \bibinfo{year}{2026}.
\newblock \bibinfo{title}{Hide-and-seek in trajectories: discovering failure
  signals for {VLA} runtime monitoring}.
\newblock \bibinfo{note}{Preprint, arXiv:2605.30834}.
\bibitem[{Penwarden et~al.(2024)Penwarden, Owhadi and Kirby}]{nn2024penwarden}
\bibinfo{author}{Penwarden, M.}, \bibinfo{author}{Owhadi, H.},
  \bibinfo{author}{Kirby, R.M.}, \bibinfo{year}{2024}.
\newblock \bibinfo{title}{Kolmogorov n-widths for multitask physics-informed
  machine learning (piml) methods: Towards robust metrics}.
\newblock \bibinfo{journal}{Neural Networks} \bibinfo{volume}{180},
  \bibinfo{pages}{106703}.
\newblock \DOIprefix\doi{10.1016/j.neunet.2024.106703}.
\bibitem[{Powers(2012)}]{powers2012kappa}
\bibinfo{author}{Powers, D.M.W.}, \bibinfo{year}{2012}.
\newblock \bibinfo{title}{The problem with kappa}, in:
  \bibinfo{booktitle}{Proceedings of the 13th Conference of the European
  Chapter of the Association for Computational Linguistics},
  \bibinfo{publisher}{Association for Computational Linguistics},
  \bibinfo{address}{Avignon, France}. pp. \bibinfo{pages}{345--355}.
\newblock \URLprefix \url{https://aclanthology.org/E12-1035/}.
\bibitem[{R{\"o}mer et~al.(2025)R{\"o}mer, Kobras, Worbis and
  Schoellig}]{romer2025fiper}
\bibinfo{author}{R{\"o}mer, R.}, \bibinfo{author}{Kobras, A.},
  \bibinfo{author}{Worbis, L.}, \bibinfo{author}{Schoellig, A.P.},
  \bibinfo{year}{2025}.
\newblock \bibinfo{title}{Failure prediction at runtime for generative robot
  policies}, in: \bibinfo{booktitle}{Advances in Neural Information Processing
  Systems (NeurIPS)}.
\newblock \bibinfo{note}{ArXiv:2510.09459}.
\bibitem[{R{\"o}mer et~al.(2026)R{\"o}mer, Seeliger, Liu, Sturgis, Bagatella,
  Marta, Krause and Schoellig}]{romer2026vfd}
\bibinfo{author}{R{\"o}mer, R.}, \bibinfo{author}{Seeliger, M.},
  \bibinfo{author}{Liu, S.}, \bibinfo{author}{Sturgis, B.},
  \bibinfo{author}{Bagatella, M.}, \bibinfo{author}{Marta, D.},
  \bibinfo{author}{Krause, A.}, \bibinfo{author}{Schoellig, A.P.},
  \bibinfo{year}{2026}.
\newblock \bibinfo{title}{Uncertainty quantification for flow-based
  vision--language--action models}.
\newblock \bibinfo{note}{Preprint, arXiv:2606.18043}.
\bibitem[{Schuirmann(1987)}]{schuirmann1987}
\bibinfo{author}{Schuirmann, D.J.}, \bibinfo{year}{1987}.
\newblock \bibinfo{title}{A comparison of the two one-sided tests procedure and
  the power approach for assessing the equivalence of average bioavailability}.
\newblock \bibinfo{journal}{Journal of Pharmacokinetics and Biopharmaceutics}
  \bibinfo{volume}{15}, \bibinfo{pages}{657--680}.
\bibitem[{Sim and Wright(2005)}]{simwright2005}
\bibinfo{author}{Sim, J.}, \bibinfo{author}{Wright, C.C.},
  \bibinfo{year}{2005}.
\newblock \bibinfo{title}{The kappa statistic in reliability studies: use,
  interpretation, and sample size requirements}.
\newblock \bibinfo{journal}{Physical Therapy} \bibinfo{volume}{85},
  \bibinfo{pages}{257--268}.
\bibitem[{Sluijterman et~al.(2024)Sluijterman, Cator and
  Heskes}]{nn2024sluijterman}
\bibinfo{author}{Sluijterman, L.}, \bibinfo{author}{Cator, E.},
  \bibinfo{author}{Heskes, T.}, \bibinfo{year}{2024}.
\newblock \bibinfo{title}{How to evaluate uncertainty estimates in machine
  learning for regression?}
\newblock \bibinfo{journal}{Neural Networks} \bibinfo{volume}{173},
  \bibinfo{pages}{106203}.
\newblock \DOIprefix\doi{10.1016/j.neunet.2024.106203}.
\bibitem[{Tang et~al.(2026)Tang, Wang, Chen, Zhang, Guan and
  Tang}]{tang2026shifting}
\bibinfo{author}{Tang, Y.}, \bibinfo{author}{Wang, T.}, \bibinfo{author}{Chen,
  Y.}, \bibinfo{author}{Zhang, B.}, \bibinfo{author}{Guan, Q.},
  \bibinfo{author}{Tang, R.}, \bibinfo{year}{2026}.
\newblock \bibinfo{title}{Shifting uncertainty to critical moments: towards
  reliable uncertainty quantification for {VLA} models}.
\newblock \bibinfo{note}{Preprint, arXiv:2603.18342}.
\bibitem[{Thompson and Walter(1988)}]{thompson1988}
\bibinfo{author}{Thompson, W.D.}, \bibinfo{author}{Walter, S.D.},
  \bibinfo{year}{1988}.
\newblock \bibinfo{title}{A reappraisal of the kappa coefficient}.
\newblock \bibinfo{journal}{Journal of Clinical Epidemiology}
  \bibinfo{volume}{41}, \bibinfo{pages}{949--958}.
\bibitem[{Uebersax and Grove(1990)}]{uebersax1990}
\bibinfo{author}{Uebersax, J.S.}, \bibinfo{author}{Grove, W.M.},
  \bibinfo{year}{1990}.
\newblock \bibinfo{title}{Latent class analysis of diagnostic agreement}.
\newblock \bibinfo{journal}{Statistics in Medicine} \bibinfo{volume}{9},
  \bibinfo{pages}{559--572}.
\bibitem[{Umesh et~al.(1989)Umesh, Peterson and Sauber}]{umesh1989}
\bibinfo{author}{Umesh, U.N.}, \bibinfo{author}{Peterson, R.A.},
  \bibinfo{author}{Sauber, M.H.}, \bibinfo{year}{1989}.
\newblock \bibinfo{title}{Interjudge agreement and the maximum value of kappa}.
\newblock \bibinfo{journal}{Educational and Psychological Measurement}
  \bibinfo{volume}{49}, \bibinfo{pages}{835--850}.
\bibitem[{Vinh et~al.(2010)Vinh, Epps and Bailey}]{vinh2010}
\bibinfo{author}{Vinh, N.X.}, \bibinfo{author}{Epps, J.},
  \bibinfo{author}{Bailey, J.}, \bibinfo{year}{2010}.
\newblock \bibinfo{title}{Information theoretic measures for clusterings
  comparison: variants, properties, normalization and correction for chance}.
\newblock \bibinfo{journal}{Journal of Machine Learning Research}
  \bibinfo{volume}{11}, \bibinfo{pages}{2837--2854}.
\bibitem[{Wang et~al.(2026)Wang, Park, Xu, Zhang, Song and Bohg}]{wang2026mof}
\bibinfo{author}{Wang, D.}, \bibinfo{author}{Park, J.}, \bibinfo{author}{Xu,
  X.}, \bibinfo{author}{Zhang, H.}, \bibinfo{author}{Song, S.},
  \bibinfo{author}{Bohg, J.}, \bibinfo{year}{2026}.
\newblock \bibinfo{title}{Mixture of frames policy: multi-frame action
  denoising for bimanual mobile manipulation}.
\newblock \bibinfo{note}{Preprint, arXiv:2607.11884}.
\bibitem[{Warrens(2008)}]{warrens2008similarity}
\bibinfo{author}{Warrens, M.J.}, \bibinfo{year}{2008}.
\newblock \bibinfo{title}{On association coefficients for 2x2 tables and
  properties that do not depend on the marginal distributions}.
\newblock \bibinfo{journal}{Psychometrika} \bibinfo{volume}{73},
  \bibinfo{pages}{777--789}.
\newblock \DOIprefix\doi{10.1007/s11336-008-9070-3}.
\bibitem[{Wyner(1975)}]{wyner1975}
\bibinfo{author}{Wyner, A.D.}, \bibinfo{year}{1975}.
\newblock \bibinfo{title}{The common information of two dependent random
  variables}.
\newblock \bibinfo{journal}{IEEE Transactions on Information Theory}
  \bibinfo{volume}{21}, \bibinfo{pages}{163--179}.
\bibitem[{Xu et~al.(2025)Xu, Nguyen, Dixon, Rodriguez, Miller, Lee, Shah,
  Ambrus, Nishimura and Itkina}]{xu2025faildetect}
\bibinfo{author}{Xu, C.}, \bibinfo{author}{Nguyen, T.K.},
  \bibinfo{author}{Dixon, E.}, \bibinfo{author}{Rodriguez, C.},
  \bibinfo{author}{Miller, P.}, \bibinfo{author}{Lee, R.},
  \bibinfo{author}{Shah, P.}, \bibinfo{author}{Ambrus, R.},
  \bibinfo{author}{Nishimura, H.}, \bibinfo{author}{Itkina, M.},
  \bibinfo{year}{2025}.
\newblock \bibinfo{title}{Can we detect failures without failure data?
  {U}ncertainty-aware runtime failure detection for imitation learning
  policies}, in: \bibinfo{booktitle}{Robotics: Science and Systems (RSS)}.
\newblock \bibinfo{note}{ArXiv:2503.08558}.
\bibitem[{Yang et~al.(2026)Yang, Yan, Feng, Yao, Wang, Mai, Zhao and
  Han}]{yang2026mimicik}
\bibinfo{author}{Yang, J.}, \bibinfo{author}{Yan, S.}, \bibinfo{author}{Feng,
  F.}, \bibinfo{author}{Yao, C.}, \bibinfo{author}{Wang, G.},
  \bibinfo{author}{Mai, Z.}, \bibinfo{author}{Zhao, Y.}, \bibinfo{author}{Han,
  Y.}, \bibinfo{year}{2026}.
\newblock \bibinfo{title}{{MimicIK}: real-time generative inverse kinematics
  from teleoperation with {FK} consistency}.
\newblock \bibinfo{note}{Preprint, arXiv:2606.15148}.
\bibitem[{Zhang et~al.(2025)Zhang, Yan, Schwing and Zhao}]{hrfm2025}
\bibinfo{author}{Zhang, Y.}, \bibinfo{author}{Yan, Y.},
  \bibinfo{author}{Schwing, A.G.}, \bibinfo{author}{Zhao, Z.},
  \bibinfo{year}{2025}.
\newblock \bibinfo{title}{Hierarchical rectified flow matching with mini-batch
  couplings}.
\newblock \bibinfo{note}{Preprint, arXiv:2507.13350}.

\end{thebibliography}

\appendix

\clearpage
\setcounter{section}{0}
\setcounter{table}{0}
\setcounter{figure}{0}
\renewcommand{\thesection}{S\arabic{section}}
\renewcommand{\thesubsection}{S\arabic{section}.\arabic{subsection}}
\renewcommand{\thetable}{S\arabic{table}}
\renewcommand{\thefigure}{S\arabic{figure}}
\renewcommand{\theHsection}{S.\arabic{section}}
\renewcommand{\theHsubsection}{S.\arabic{section}.\arabic{subsection}}
\renewcommand{\theHtable}{S.\arabic{table}}
\renewcommand{\theHfigure}{S.\arabic{figure}}

\section*{Supplementary material}
Sections~S1--S7 provide the proof details, run inventory, per-condition
coordinates, complete mechanism tables, threshold and NFE sensitivity,
reproduction and protocol information, and a set of prospective tests.

\section{Proof details}
\label{app:proofs}

Throughout, $y=\Delta=\mathrm{BL}-\mathrm{obs}$, positive for coordination.

\subsection{Proof of Theorem~\ref{thm:region}}

\begin{proof}
Write one-hot indicators $U=(\mathbf{1}[M_q=k])_k$, $V=(\mathbf{1}[M_T=k])_k$.

(i) $y=\sum_k(\Pr(M_q{=}k,M_T{=}k)-q_kp_k)=\operatorname{tr}\operatorname{Cov}(U,V)$. Conditional independence gives $\mathbb{E}[U_kV_k\mid z]=a_k(z)b_k(z)$; under $\tilde z$ the conditional term of the law of total covariance vanishes, so $\operatorname{Cov}(U_k,V_k)=\operatorname{Cov}_{\tilde z}(a_k,b_k)$. The tilt separates from the covariance because $\{M_q\neq I\}\cap\{M_T\neq I\}$ factorizes given $(c,z)$.

(ii) Cauchy--Schwarz on centered vectors in $L^2(\Omega;\mathbb{R}^K)$ gives
\[
|y|=\bigl|\langle U-\mathbb{E}U,\, V-\mathbb{E}V\rangle\bigr| \;\le\; \sqrt{\textstyle\sum_k\operatorname{Var}(U_k)}\;\sqrt{\textstyle\sum_k\operatorname{Var}(V_k)} \;=\; \sqrt{G(q)G(p)},
\]
using $\sum_k\operatorname{Var}(U_k)=G(q)$. By Jensen $\sum_k m_k^2\ge1/K$, so $G(m)\le1-1/K$, with equality iff uniform.

(iii) From (ii) and $G(p)\le1-1/K$: $G(q)\ge y^2/G(p)\ge Ky^2/(K-1)$.

(iv) Fr\'echet--Hoeffding cellwise, $P_{kk}\le\min(q_k,p_k)$, gives $p_o\le1-\mathrm{TV}(q,p)$; subtract $p_e$. The bound is attained: set $P_{kk}=\min(q_k,p_k)$ and fill the off-diagonal with the normalized product of the disjointly supported row and column surpluses. That coupling also satisfies (ii), whence $C_{K,\mathrm{tight}}\le C_K$.
\end{proof}

\subsection{Conditioning can be dropped under independence}

\begin{lemma}
\label{lem:dropcond}
Let $X$ be integrable and $Z$ independent of $\sigma(X,Y)$. Then $\mathbb{E}[X\mid Y,Z]=\mathbb{E}[X\mid Y]$ a.s.
\end{lemma}

\begin{proof}
For bounded measurable $h,g$,
\[
\mathbb{E}[X\,h(Y)\,g(Z)] = \mathbb{E}[X\,h(Y)]\,\mathbb{E}[g(Z)] = \mathbb{E}\bigl[\mathbb{E}[X\mid Y]h(Y)\bigr]\mathbb{E}[g(Z)] = \mathbb{E}\bigl[\mathbb{E}[X\mid Y]\,h(Y)\,g(Z)\bigr],
\]
the first and third equalities using $Z\perp\sigma(X,Y)$ (note $\mathbb{E}[X\mid Y]h(Y)$ is $\sigma(Y)\subset\sigma(X,Y)$ measurable). The family $\{A\cap B: A\in\sigma(Y), B\in\sigma(Z)\}$ is a $\pi$-system generating $\sigma(Y,Z)$; by the functional monotone class theorem the display extends to all bounded $\sigma(Y,Z)$-measurable functions, hence $\mathbb{E}[X\mid Y,Z]=\mathbb{E}[X\mid Y]$ a.s.
\end{proof}

A convenient sufficient condition for mean irrelevance is joint independence of $Z$ from the full training tuple $(c,x_0,x_1,t)$. Merely assuming $t\perp Z$ is not enough: conditioning on an adaptively chosen $t$ can reintroduce information about $Z$ through the other variables. Standard CFM implementations draw $t\sim U[0,1]$ independently of the sample, which satisfies the sufficient condition when $Z$ is also drawn independently.

\subsection{Equality conditions for the coverage bound}

As in Section~\ref{sec:tight}: the $K$-ary bound \eqref{eq:bound} chains two inequalities. Cauchy--Schwarz is tight iff the centered one-hot indicators are a.s.\ proportional, which for a positive proportionality constant forces $q=p$ and $M_q=M_T$ a.s.\ (perfect agreement), while a negative constant requires degeneracy onto two categories; $G(m)\le1-1/K$ is tight iff the marginal is uniform. Hence $y=C_K$ iff perfect agreement, and $y=1-1/K$ additionally requires uniform marginals. At $K=2$ the conditions read: each branch deterministic given the shared variable; identical choices; 50/50 marginals. The oracle and window-token families measure $+0.4996$ and $+0.4994$ per-condition, near all three at once.

\section{Run inventory}
\label{app:inventory}

Table~\ref{tab:inventory} lists every experimental round behind the paper, so that any number in the text can be traced to a population. Within each round, all runs that reached the frozen protocol are included. Incomplete runs, convergence-gate exceptions, and the one reclassified cell are listed in \ref{app:prereg}.

\begin{table}[t]
\centering
\caption{Run inventory. ``Runs'' counts trained models; evaluation columns count what each run was scored on. The original mechanism atlas comprises 68 training runs, of which 67 carry per-condition coordinates and appear in Fig.~\ref{fig:region} (52 v2, 15 v3). The preregistered seed expansion adds 49 robot runs and 84 testbed runs (design, seed lists, and analysis script committed before training and before unblinding, with the registration and its two revision records archived alongside the artifacts); the expansion runs are reported in Table~\ref{tab:seeds} and Section~\ref{sec:toy-mapping} and do not enter Fig.~\ref{fig:region}.}
\label{tab:inventory}
\footnotesize
\begin{tabular}{@{}l l l l@{}}
\toprule
Round & Environment & Runs & Evaluation \\
\midrule
Mechanism atlas & BimodalBypass (v2) & \multirow{2}{*}{68 total} & 75 conditions, 500 pairs \\
 & RotatingBypass (v3) & & 36 conditions, 240 pairs \\
Threshold / NFE sensitivity & v2 and v3 & (no retraining) & 108 re-evaluations \\
Ensemble protocols & v2 only & 4 families, 5 members & 4 protocols, 10 pairs (main family) \\
Testbed, balanced & RotatingRing & 186 $\times$ 50k steps & 96{,}000 samples/run; 135 re-evaluations \\
Testbed, skewed prior & RotatingRing & 153 (51 cells $\times$ 3 seeds) & 75 conditions/run; 15.1 GPU-hours \\
External monitor & LIBERO-Plus & 3 monitor seeds $\times$ 2 arms & 1664 episodes, 283{,}375 steps; 10 signals \\
External coverage & LIBERO-Plus & same checkpoints & 147{,}529 steps $\times$ 32 chunks $\times$ 2 draws \\
\bottomrule
\end{tabular}
\end{table}

\section{Per-condition coordinates behind Fig.~\ref{fig:region}}
\label{app:tables}

Tables~\ref{tab:percond-v2} and~\ref{tab:percond-v3} list the per-condition coordinates. Computation convention (identical to the effect-table code): $n_i=$ both-valid count, $\mathrm{obs}_i=\mathrm{mismatch}_i/n_i$, $a_i,b_i=$ side fractions (marginal convention flagged per row), $\mathrm{BL}_i=a_i(1-b_i)+(1-a_i)b_i$, $C_i=2\sqrt{a_i(1-a_i)b_i(1-b_i)}$, aggregation weighted by $n_i$. The plug-in estimator of $\Delta$ shrinks toward zero at rate (number of conditions)/(total both-valid pairs): the median run-level shrinkage is 0.10\% of the effect and the worst 0.43\%, against a median per-condition both-valid count of 1040 (range 2--1280 over 4440 condition-run cells), so no reported effect is materially biased by it.

\begin{table}[t]
\centering
\caption{BimodalBypass (v2, 75 conditions, $\tau=0.11$), per-condition accounting; $\Delta=\mathrm{BL}-\mathrm{obs}$. $^{*}$Per-seed occupancies 100.05\%/99.90\%/100.02\%: micro-excursions within the head-valid approximation error of Section~\ref{sec:gate}. $^{**}$Families H and I are quoted in the exact both-valid accounting (ratio of means), the only accounting under which they can be re-estimated from the shipped artifacts: H 91.1\% (per seed 91.7/84.1/96.9) and I 97.9\% (97.3/97.9/98.6), with all $75\times6$ conditions satisfying $y_i\le C_i$. Under head-valid marginals the same families read 91.7\% and 99.4\%, and a minority of family-I conditions micro-exceed the bound (15/13/40 conditions, 68 in total, maximum $+0.0463$); family E contributes 301 excursions above $C$ and 303 above the fixed-marginal bound across its five runs. All of them vanish under exact both-valid accounting (Section~\ref{sec:reestimation}).}
\label{tab:percond-v2}
\scriptsize
\setlength{\tabcolsep}{2.5pt}
\begin{tabular}{@{}l rrrr r l@{}}
\toprule
Family & $\mathrm{obs}$ & $\mathrm{BL}$ & $\Delta$ & $C$ & $\rho$ & Marg. \\
\midrule
A independent $\beta1\lambda0$ (3s) & 0.4826 & 0.4839 & $+0.0014$ & 0.4474 & 0.3\% & head \\
\quad A seed0/1/2 & & & $-0.0006/{+}0.0029/{+}0.0017$ & 0.4439/0.4495/0.4488 & 0.1/0.7/0.4\% & head \\
B shared noise $\beta0\lambda0$ (3s) & 0.4362 & 0.4369 & $+0.0007$ & 0.4185 & 0.2\% & head \\
\quad B seed0/1/2 & & & $-0.0160/{+}0.0015/{+}0.0166$ & 0.4565/0.3864/0.4126 & $-3.5$/0.4/4.0\% & head \\
C $z_s$ dim-0 anchor (3s) & 0.4921 & 0.4941 & $+0.0020$ & 0.4857 & 0.4\% & both \\
D continuous $z_s$ prior $d16$ (3s) & 0.4843 & 0.4834 & $-0.0009$ & 0.4708 & $-0.2$\% & both \\
E posterior $\beta_{\mathrm{KL}}1$ fb0.1 s0/s1 & 0.0000 & 0.0158/0.0163 & $+0.0158/{+}0.0163$ & 0.0158/0.0163 & 100.0\% & both \\
E posterior $\beta_{\mathrm{KL}}0.01$ s0/s1 & 0.0000 & 0.0170/0.0625 & $+0.0170/{+}0.0625$ & 0.0170/0.0625 & 100.0\% & both \\
E posterior $\beta_{\mathrm{KL}}1$ fb0.5 s0 & 0.0000 & 0.0001 & $+0.0001$ & 0.0001 & 100.0\% & both \\
F $\beta0\lambda0.3$ (3s) & 0.4235 & 0.4492 & $+0.0257$ & 0.4363 & 5.9\% & head \\
$\beta0\lambda1$ (1s) & 0.4107 & 0.4902 & $+0.0795$ & 0.4663 & 17.0\% & head \\
G $\beta0\lambda3$ (3s) & 0.2379 & 0.4420 & $+0.2042$ & 0.4061 & 50.3\% & head \\
\quad G seed0/1/2 & & & $+0.2404/{+}0.1366/{+}0.2355$ & 0.3898/0.3523/0.4762 & 61.7/38.8/49.5\% & head \\
$\beta0.1\lambda3$ (1s) & 0.2921 & 0.4385 & $+0.1464$ & 0.4330 & 33.8\% & head \\
$\beta1\lambda3$ (1s) & 0.5352 & 0.5396 & $+0.0044$ & 0.4176 & 1.0\% & head \\
H episode uniform (3s) & 0.0433 & 0.4913 & $+0.4481$ & 0.4885 & 91.1\% & both$^{**}$ \\
I episode learned prior (3s) & 0.0101 & 0.4977 & $+0.4876$ & 0.4907 & 97.9\% & both$^{**}$ \\
J window token $K2$ (3s) & 0.0002 & 0.4996 & $+0.4994$ & 0.4994 & 100.0\%$^{*}$ & head \\
J$'$ window token $K2+\lambda0.1$ (3s) & 0.0001 & 0.4996 & $+0.4995$ & 0.4996 & 100.0\% & head \\
window token $K4$ (3s) & 0.0730 & 0.4879 & $+0.4149$ & 0.4863 & 85.3\% & head \\
K oracle (1s) & 0.0000 & 0.4996 & $+0.4996$ & 0.4996 & 100.0\% & head \\
\bottomrule
\end{tabular}
\end{table}

\begin{table}[t]
\centering
\caption{RotatingBypass (v3, 36 conditions, $\tau=0.18$), per-condition accounting; $\Delta=\mathrm{BL}-\mathrm{obs}$. $^{\dagger}$100.6\% under head-valid marginals. All v3 pooled baselines are 0.4999--0.5001 (the environment is fully branch-symmetric, leaving no global preference for pooled marginals to exploit), yet pooled-accounting effects can deviate by up to 0.09; this table is per-condition throughout (Section~\ref{sec:gate}). The $\rho$ column averages per-seed occupancies (G: 65.0/56.8/67.9 $\to$ 63.2\%); mean effect over mean bound gives 63.6\% (ratio nonlinearity; either is usable when the aggregation convention is stated).}
\label{tab:percond-v3}
\scriptsize
\setlength{\tabcolsep}{4pt}
\begin{tabular}{@{}l rr l rrr@{}}
\toprule
Family & $\mathrm{obs}$ & $\mathrm{BL}$ & $\Delta$ (per seed) & $C$ & $\rho$ & Valid \\
\midrule
A independent (3 seeds) & 0.4811 & 0.4811 & $0.0000$ ($-0.0005/{-}0.0016/{+}0.0020$) & 0.4700 & 0.3\% & 0.7382 \\
B shared source noise (3 seeds) & 0.4567 & 0.4652 & $+0.0085$ ($+0.0079/{+}0.0103/{+}0.0074$) & 0.4555 & 1.9\% & 0.7501 \\
D continuous $z_s$ prior (3 seeds) & 0.4723 & 0.4748 & $+0.0025$ ($+0.0019/{+}0.0019/{+}0.0037$) & 0.4636 & 0.5\% & 0.6474 \\
G endpoint consistency $\lambda{=}3$ (3 seeds) & 0.1709 & 0.4458 & $+0.2748$ ($+0.3097/{+}0.2090/{+}0.3057$) & 0.4319 & 63.2\% & 0.4449 \\
H episode uniform token (2 seeds) & 0.0784 & 0.4708 & $+0.3923$ ($+0.4232/{+}0.3614$) & 0.4661 & 84.1\% (88.6/79.6) & 0.7056 \\
K oracle (1 seed) & 0.0000 & 0.4994 & $+0.4994$ & 0.4964 & 100.0\%$^{\dagger}$ & 0.8704 \\
\bottomrule
\end{tabular}
\end{table}

\emph{Tight-bound readings (Section~\ref{sec:calculus}).} Recomputed from per-condition sufficient statistics: the interior-climb family G moves from a 50.3\% occupancy of $C$ to 69.2\% of $C_{K,\mathrm{tight}}$ in v2 (seeds 94.0/67.0/55.4\%, per-condition biases 0.175/0.211/0.056) and from 63.2\% to 75.9\% in v3 (75.7/67.6/83.0\%); floor families read 0.2--7.4\% under either bound; corner families stay above 85\%; the collapsed family is exactly 100\% under both-valid accounting ($\mathrm{obs}=0$ forces the both-valid marginals to agree per condition, so $C_{K,\mathrm{tight}}=C_K$). The estimator caveat of Section~\ref{sec:calculus} is directly measurable: the v3 oracle family has per-condition $\mathrm{obs}$ identically 0 (true bias zero), yet the mean per-condition $|\hat q_L-\hat p_L|$ is 0.0196, and $\Delta/C_{K,\mathrm{tight}}$ reads 100.1--104.1\% on corner families and up to 138.7\% on the collapsed family under head-valid marginals.

\emph{Marginal conventions for the figures.} Fig.~\ref{fig:region} contains 67 per-condition run points (52 v2 + 15 v3), each at the most exact marginal convention its artifacts support: 17 runs both-valid (11 continuous-$z_s$, 6 episode-token recomputations), the remaining 50 head-valid. Six head-valid points sit microscopically above the diagonal (largest $\Delta/C=1.005981$: four window-token and the two oracles), drawn as-is and not clipped. That they are marginal-accounting error rather than estimator failure is now settled rather than asserted: over the 1275 conditions of the 17 re-estimable runs the exact both-valid accounting produces zero excursions (Section~\ref{sec:reestimation}). Per-condition inequalities are preserved by weighted aggregation, so aggregate points obey the same boundary. Fig.~\ref{fig:overlay} adds the 186 testbed runs (both-valid weighted accounting; $\tau=0.55$, NFE 20, 50k steps; zero violations, max $\Delta/C_K=1+4\times10^{-16}$), and Fig.~\ref{fig:skew} the 153 skew-round runs (\texttt{nullspace} only, same accounting). Reading conventions moved here from Section~\ref{sec:toy-relabel}: the beds' thresholds are not the same stringency ($\tau=0.55$ acts on a normalized projection peak, independent validity 0.984; the robots' $\tau$ are metric, 0.773/0.738), and re-matching by valid-pair rate ($\tau_{\mathrm{match}}=0.87$, validity 0.778) leaves rankings unchanged over $\tau\in[0.40,0.87]$ while moving corner occupancies 2.4--2.7 points (oracle $82.4\to84.8\%$, unsupervised $77.8\to80.5\%$), so vertical comparisons are robust but the two $\tau$'s do not interpret each other; valid-pair rates span 0.760--0.999 (testbed) and 0.24--1.00 (robots), a dimension off-figure that must accompany any quoted point.

\emph{The $\pm34^\circ$ geometry behind Section~\ref{sec:tautology}.} A top-grasp parallel gripper cannot pass under the obstacle, so both azimuths live in the upper half-circle; a $90^\circ$-separated pair confined to $(0^\circ,180^\circ)$ has its difference direction confined to $(45^\circ,135^\circ)$, and the vertical scaling of 0.68 narrows that to $\pm34^\circ$ around the lateral axis. The v2 orthogonal-complement per-seed triple at $\lambda{=}0.3$ is $0.578/0.391/0.463$; the v3 family means are 0.5970--0.7098 at the frozen gate $\tau=0.11$.

\section{Additional mechanism tables}
\label{app:mechtables}

Tables~\ref{tab:sensitivity}--\ref{tab:pose} carry the per-cell detail behind Sections~\ref{sec:floor}--\ref{sec:interior}: the $z_s$-sensitivity decay of the floor family, the full $\beta\times\lambda$ grid, the three-seed detail of its strongest cell, and the third-representation ablation. Representative configurations quoted in the main text are selected by a rule fixed before the grid was read: the strongest cell of each ladder ($\lambda{=}3$), the two endpoints of the $\beta$ axis, and the families named in the preregistration; no cell is promoted or omitted on the basis of its effect.

\begin{table}[t]
\centering
\caption{Sensitivity of the velocity fields to $z_s$ (relative change under a redrawn $z_s$; joint/position branch) at three checkpoints, from training histories. A validation-batch companion metric agrees in magnitude (prior 0.00128--0.00196; posterior 0.0122--0.0370).}
\label{tab:sensitivity}
\footnotesize
\begin{tabular}{@{}l ccc@{}}
\toprule
Run & step 1 & step 25k & step 50k \\
\midrule
prior $d16$ seed0 & 0.3791 / 0.4002 & 0.00280 / 0.00227 & 0.00146 / 0.00137 \\
prior $d16$ seed1 & 0.4466 / 0.4674 & 0.00293 / 0.00256 & 0.00135 / 0.00127 \\
prior $d16$ seed2 & 0.3694 / 0.4217 & 0.00283 / 0.00267 & 0.00177 / 0.00167 \\
prior, three-seed mean & \textbf{0.398 / 0.430} & 0.00286 / 0.00250 & \textbf{0.00153 / 0.00144} \\
posterior (range, 5 runs) & 0.368--0.442 / 0.382--0.525 & 0.0158--0.0328 / 0.0192--0.0435 & 0.0125--0.0215 / 0.0175--0.0269 \\
$z$-dim-0 anchor (3 runs) & 0 & 0 & 0 \\
\bottomrule
\end{tabular}
\end{table}

\begin{table}[t]
\centering
\caption{The full $\beta\times\lambda$ grid (v2, seed 0, 15 runs): pooled $\Delta=\mathrm{BL}-\mathrm{obs}$, valid-pair rate, residual (mm).}
\label{tab:grid}
\footnotesize
\begin{tabular}{@{}l rrrrr@{}}
\toprule
 & $\lambda{=}0$ & $\lambda{=}0.1$ & $\lambda{=}0.3$ & $\lambda{=}1$ & $\lambda{=}3$ \\
\midrule
\multicolumn{6}{@{}l}{\emph{effect $\Delta$ (pooled)}} \\
shared $\beta{=}0$ & $-0.0124$ & $-0.0062$ & $+0.0222$ & $+0.0799$ & $+0.2442$ \\
partial $\beta{=}0.1$ & $-0.0117$ & $-0.0094$ & $+0.0088$ & $+0.0474$ & $+0.1519$ \\
independent $\beta{=}1$ & $+0.0024$ & $+0.0006$ & $+0.0038$ & $+0.0020$ & $+0.0126$ \\
\midrule
\multicolumn{6}{@{}l}{\emph{valid-pair rate}} \\
shared $\beta{=}0$ & 0.9088 & 0.8728 & 0.8569 & 0.4975 & 0.4576 \\
partial $\beta{=}0.1$ & 0.8802 & 0.8580 & 0.7184 & 0.7967 & 0.2370 \\
independent $\beta{=}1$ & 0.7993 & 0.8345 & 0.8771 & 0.7761 & 0.1814 \\
\midrule
\multicolumn{6}{@{}l}{\emph{cross-branch residual (mm)}} \\
shared $\beta{=}0$ & 82.95 & 78.36 & 77.14 & 73.48 & 54.72 \\
partial $\beta{=}0.1$ & 83.24 & 79.49 & 76.34 & 73.16 & 69.57 \\
independent $\beta{=}1$ & 85.59 & 79.85 & 79.33 & 82.51 & 75.42 \\
\bottomrule
\end{tabular}
\end{table}

\begin{table}[t]
\centering
\caption{Three seeds of $\beta{=}0,\lambda{=}3$ (v2). In the mean row the occupancy can be the mean of per-seed occupancies (pooled 49.9\%, per-condition 50.0\%) or mean effect over mean bound (50.3\%); the text quotes the per-condition 50.3\% with per-seed 61.7/38.8/49.5\%.}
\label{tab:l3seeds}
\scriptsize
\setlength{\tabcolsep}{2pt}
\begin{tabular}{@{}c rl r rr r rl r r rr l r@{}}
\toprule
seed & $\mathrm{obs}$ & 95\% CI & Valid & $q_L$ & $p_L$ & $\mathrm{BL}$ & $\Delta$ & 95\% CI & $C$ & $\rho$ & $\Delta_{\mathrm{pc}}$ & $\rho_{\mathrm{pc}}$ & $H_q$/$H_T$ & Res. \\
\midrule
0 & 0.1900 & [.1853,\,.1949] & 0.4576 & 0.211 & 0.386 & 0.4342 & $+0.2442$ & [$+.2382,+.2500$] & 0.3974 & 61.5\% & $+0.2404$ & 61.7\% & 0.515/0.667 & 54.72 \\
1 & 0.2782 & [.2735,\,.2831] & 0.6338 & 0.164 & 0.375 & 0.4160 & $+0.1378$ & [$+.1323,+.1433$] & 0.3583 & 38.5\% & $+0.1366$ & 38.8\% & 0.446/0.662 & 58.81 \\
2 & 0.2454 & [.2389,\,.2516] & 0.4849 & 0.562 & 0.611 & 0.4863 & $+0.2410$ & [$+.2370,+.2448$] & 0.4838 & 49.8\% & $+0.2355$ & 49.5\% & 0.686/0.668 & 59.44 \\
mean & 0.2379 &  & 0.5254 &  &  & 0.4455 & $+0.2077$ &  & 0.4132 & 49.9\% & $+0.2042$ & 50.3\% & 0.549/0.666 & 57.66 \\
\bottomrule
\end{tabular}
\end{table}

\begin{table}[t]
\centering
\caption{$\lambda_{\mathrm{cons}}\times\lambda_{\mathrm{rot}}$ ablation (v3, seed 0, $\tau=0.18$; the orientation branch is mode-assigned by nearest mean-yaw match; $\Delta=\mathrm{BL}-\mathrm{obs}$. Random three-way agreement 0.25; teacher-forced training residuals: position 10.7--17.6\,mm, orientation 1.45--1.88$^\circ$).}
\label{tab:pose}
\footnotesize
\setlength{\tabcolsep}{3.5pt}
\begin{tabular}{@{}cc lll c rr@{}}
\toprule
$\lambda_{\mathrm{cons}}$ & $\lambda_{\mathrm{rot}}$ & $q\!\leftrightarrow\!$pos & $q\!\leftrightarrow\!$rot & pos$\leftrightarrow$rot & 3-way & Pos.\ (mm) & Rot.\ ($^\circ$) \\
\midrule
0 & 0 & $+0.0501$ [$+0.0316$, $+0.0633$] & $+0.0227$ [$+0.0064$, $+0.0371$] & $+0.0291$ [$+0.0070$, $+0.0476$] & 0.307 & 93.90 & 13.65 \\
0 & 1 & $+0.0403$ [$+0.0241$, $+0.0522$] & $+0.0223$ [$+0.0084$, $+0.0355$] & $+0.0356$ [$+0.0170$, $+0.0523$] & 0.302 & 94.00 & 13.46 \\
3 & 0 & $+0.3353$ [$+0.3210$, $+0.3465$] & $+0.0018$ [$-0.0131$, $+0.0165$] & $+0.0279$ [$+0.0087$, $+0.0461$] & 0.430 & 60.97 & 13.05 \\
3 & 1 & $+0.2671$ [$+0.2499$, $+0.2843$] & $+0.0227$ [$+0.0105$, $+0.0337$] & $+0.0166$ [$+0.0009$, $+0.0316$] & 0.408 & 63.40 & 12.63 \\
\bottomrule
\end{tabular}
\end{table}

\emph{VFD score and per-configuration readings (Section~\ref{sec:vfd}).} VFD's Eq.~(7) was verified against the preprint: $u_e(y;V)=\frac{1}{M(M-1)N_s}\mathbb{E}_{x_0\sim p_0}\sum_{i\neq j}\sum_\ell \kappa_{s_\ell}\lVert v^{\theta_i}_{s_\ell}(x^{(i)}_{s_\ell},y)-v^{\theta_j}_{s_\ell}(x^{(i)}_{s_\ell},y)\rVert_2^2$ with $\kappa_s=s/(1-s)$, all members sharing $x_0$, each walking its own ODE path, member $j$ evaluated at member $i$'s state; the original fine-tunes one base model on reshuffled data, $M=2$ in its main text. Recomputed on our ensembles and stratified by member mode agreement ($\kappa$ mean over 20 steps: 2.598): \texttt{b1\_l0}, $\kappa$-weighted contamination share 0.055--0.335 (mean 0.158), unweighted 0.066--0.459 (mean 0.202), within-condition AUROC 0.478--0.730 weighted against 0.820--0.943 unweighted; \texttt{b0\_l0}, 0.081--0.300 / 0.090--0.387, AUROC 0.581--0.849 / 0.903--0.959; \texttt{b0\_l3}, 0.011--0.190 / 0.015--0.235, AUROC 0.553--0.829 / 0.835--0.973; mixed objectives, $\kappa$-weighted 0.031--0.116, AUROC 0.695--0.940.

\section{Threshold and NFE sensitivity}
\label{app:sensitivity}

The assignment threshold $\tau$ and the integration budget NFE were frozen in the main protocol; both were then scanned, 108 evaluations across the two robotic environments.

\emph{Ranking is stable; the interior-climb effect is not an invariant.} Family ordering is unchanged over $\tau\in[0.05,0.14]$ (v2) and $[0.05,0.17]$ (v3): oracle $>$ episode token $>$ $\lambda{=}3$ $>$ continuous $z_s\approx$ independent $\approx$ shared noise. But the $\lambda{=}3$ effect itself moves with the threshold (Table~\ref{tab:tau}).

\begin{table}[h]
\centering
\footnotesize
\caption{$\lambda{=}3$ family under threshold variation (pooled $\Delta$; valid-pair rate).}
\label{tab:tau}
\begin{tabular}{@{}l rr rr@{}}
\toprule
$\tau$ & v2 $\Delta$ & v2 valid & v3 $\Delta$ & v3 valid \\
\midrule
0.05 & $+0.1913$ & 0.729 & $+0.2150$ & 0.999 \\
0.08 & $+0.2141$ & 0.609 & $+0.2150$ & 0.999 \\
0.11 & $+0.2442$ & 0.458 & $+0.2161$ & 0.994 \\
0.14 & $+0.2995$ & 0.255 & $+0.2396$ & 0.880 \\
0.17 & $+0.4427$ & 0.050 & $+0.3006$ & 0.585 \\
0.20 & --- & --- & $+0.3936$ & 0.154 \\
\bottomrule
\end{tabular}
\end{table}

Tighter thresholds strengthen the effect and cut the valid-pair rate together: ``$\lambda{=}3$ trades valid pairs for coordination'' holds not only across mechanisms but \emph{within} the mechanism along $\tau$. The corner families are far flatter (episode token: v2 $+0.4363\to+0.4545$, v3 $+0.4076\to+0.4341$), and the independent and continuous-$z_s$ families move within $\pm0.01$. The corner families' \emph{occupancy} is robust too: under $\pm20\%$ threshold perturbation ($0.088/0.110/0.132$, both-valid accounting), family H reads $90.1\%\to91.1\%\to92.5\%$ (2.4-point swing) and family I $97.4\%\to97.9\%\to98.5\%$ (1.1), monotone, conclusion unchanged---the contrast with the $\lambda{=}3$ family stands. At $\tau=0.17$ (v2) and $\tau=0.24$ (v3) everything collapses (validity 0.5--8\% and 0): the threshold approaches the demonstrations' own bypass amplitude, an artifact, not a conclusion.

\emph{NFE moves the same way}: fewer integration steps, stronger apparent coordination, lower valid-pair rate---coarser integration leaves samples short of the full bypass amplitude, the shortfall side is declared invalid, and the survivors are more extreme. Across NFE $5\to50$ the $\lambda{=}3$ effect moves 22\% in v2 ($+0.2913\to+0.2381$) and 5\% in v3 ($+0.2249\to+0.2139$); all other families move within 0.01; the family ranking is unchanged at every NFE. Both directional predictions were written down on the robotic side first and then independently checked on the testbed: its 135 threshold/NFE re-evaluations agree in direction, with the family ranking unchanged over $\tau\in[0.40,0.87]$ and all NFE; the only inversion occurs at $\tau=0.90$, where the two floor families swap at $+0.2\%$ against $+0.2\%$---ranking noise at zero.

Practical rule, restated from Section~\ref{sec:corner-entry}: report any coordination effect together with its threshold, valid-pair rate and NFE. Canonical citations: v2 $+0.2077$ (pooled, three-seed mean) at $\tau=0.11$, validity 0.5254, NFE 20; v3 $+0.2748$ (per-condition, three-seed mean) at $\tau=0.18$, validity 0.4449, NFE 20.

\section{Reproduction, protocol provenance, and corrections}
\label{app:prereg}

\subsection{Testbed and environment specification}

RotatingRing generator: condition $c=(\theta,r)$, $\theta\sim U(0,2\pi)$, $r\sim U(0.70,1.30)$; the target is a mixture of $K$ equiprobable modes, mode $k$ centered at azimuth $\theta+2\pi k/K$, radius $r$; a sample is the length-$H{=}16$ path $x_1[h]=\sin(\pi h/(H{-}1))\,(r\,u_k+\delta)+\sigma_{\mathrm{step}}\eta[h]$ with $u_k=(\cos(\theta+2\pi k/K),\sin(\theta+2\pi k/K))$, $\delta\sim\mathcal{N}(0,0.08^2I_2)$, $\eta[h]\sim\mathcal{N}(0,I_2)$, $\sigma_{\mathrm{step}}=0.02$: the standard ``$K$ Gaussians on a ring'', rotated and scaled by the condition ($H=1$ would erase the threshold-sensitivity axis). Branch maps: \texttt{orthogonal}, a fixed orthogonal reflection; \texttt{polar}, a radial power law with radius-dependent rotation; \texttt{nullspace}, an $\mathbb{R}^2\to\mathbb{R}^6$ lift with linear left inverse; round-trip errors $1.3/4.7/2.7\times10^{-15}$. Assignment rule: $\mathrm{peak}_k=\max_h\langle X_h,u_k(c)\rangle/r(c)$, argmax over $k$, invalid below $\tau$. $K{=}3$ covers the five preregistered families; the $K{=}4/8$ extensions are post-registration under unchanged criteria.

Robot-side geometry: v2's central wall has half-size $0.025\times0.09\times0.09$\,m at $[0.07,0,0.89]$, goal region $[0.24,0,0.803]$; v3's suspended beam is $0.025\times0.062\times0.040$\,m at $[0.07,0,0.952]$, with apex error mean 4.27\,mm (max 6.50) and minimum clearance 28.5\,mm over the 240 accepted pairs.

\subsection{Artifacts and reproduction}

All gate evaluations are single-GPU and serial; stratified residuals cost 3--4 seconds per configuration; nothing requires retraining, and every number in the paper is recomputable from the deposited artifacts listed in the data-availability statement. Two acceptance checks are worth naming because conclusions rest on them: the episode-family per-condition recomputation matches the four count fields of the original gate JSONs exactly, and an independent $K$-ary math check, importing nothing from the testbed code, reproduces the effect and capacity functions; a regression suite of seven locked tests pins the effect-size algorithms. Environment geometry constants are computed directly from the HDF5 datasets and split manifests, with the branch point defined as the last step before the paired trajectories first separate (Euclidean distance $>10^{-5}$), identical to the evaluation code. The two testbed datasets (675\,MB and 945\,MB) are excluded from version control; reproducibility rests on deterministic generators with fixed seeds.

\subsection{Protocol provenance and deviations}

The balanced-testbed protocol was registered before any full mechanism-family evaluation, and the skew-prior protocol was registered before its datasets were generated. The archived documents specify the evaluated families, operating points, acceptance checks, and interpretation criteria; commit hashes are provided with the artifacts.

Four deviations affect how the supplementary results should be read:
\begin{enumerate}
\item An early 2000-step smoke test preceded a minor revision of the balanced-testbed document. The revised capacity formula and the decision to separate shared source noise from the auxiliary-latent floor were already present in the robot-side analysis; no full mechanism result existed at the time.
\item The first unsupervised-token launcher unpacked a clustering return tuple incorrectly, so those jobs stopped before training and produced no analyzed checkpoints. A separate launcher used 20{,}000 rather than 50{,}000 steps for 126 configurations; all headline values were rerun at 50{,}000 steps. The shorter runs agree qualitatively except for the training-budget sensitivity of posterior setting E3, noted in Section~\ref{sec:toy-scope}.
\item Preregistered criterion P3 required $|\Delta|$ to be non-decreasing in $\lambda$. It fails in its literal form under the null-space map because the $\lambda=0$ point is already anti-coordinated. The signed effect is strictly monotone across the 15 family means. We therefore record the criterion form as failed and the mechanism pattern as supported.
\item Two total-variation thresholds in the skew-prior study did not account for their structural ceilings and are therefore reported in normalized form. The separate criterion that uniform-draw occupancy remain above 60\% genuinely fails at $p=(0.92,0.08)$ and in the reduced-separability $p=(0.80,0.20)$ condition. Four balanced-study runs and five skew-study cells fall just outside their loss-plateau gates; including or excluding them changes family means by at most 0.1 percentage points. A reduced-separability cell originally assigned to the horizontal-shift class was reclassified as a floor point after recomputing its own relative capacity.
\end{enumerate}

\subsection{Implementation corrections and excluded pilot analysis}

\emph{Orientation convention.} The forward-kinematics implementation and RoboSuite encode end-effector orientation in frames differing by a constant $90^\circ$ rotation about the $z$ axis (measured $90.000^\circ\pm0.031^\circ$; position difference 0.30\,mm). The first orientation-consistency runs omitted this transform and their geodesic residuals remained at $87$--$99^\circ$. All four pose configurations were retrained after the correction; the residual fell to $6.7^\circ$ within 1750 steps. No pre-correction checkpoint enters any result, and no position-branch result was affected.

\emph{Excluded internal monitor pilot.} An earlier internal failure-farm analysis is not used because the evaluation perturbed effectively degenerate input dimensions, allowed the monitored model to influence executed trajectories, and stratified too coarsely to support the reported AUROCs. No number from that pilot contributes to the monitor results in Section~\ref{sec:liberoplus}; the external experiment was designed to remove those three defects.

\subsection{LIBERO-Plus monitor experiment}
\label{app:liberoplus}

\emph{Assets and evaluation set.} The experiment uses the official LIBERO-Plus assets and the \texttt{libero\_spatial} task set. Perturbation dimensions and severity labels are taken from the benchmark metadata, and initial states are obtained through the official API. Six perturbation dimensions are included. Language-instruction perturbations are excluded because the executing policy conditions on a task index rather than text; tasks without an official severity label are also excluded.

\emph{Policies and monitored heads.} The executing policy $\pi_{\mathrm{exec}}$ is an action-chunk CFM policy with visual, proprioceptive, and task-index inputs, trained on the 500 official demonstrations. The monitored coordinated arm uses a $K=2$ partition of future position chunks and supplies the same sampled symbol to two otherwise separate heads; the uncoordinated arm omits that symbol. The PCA reducer, normalization statistics, and token partition are fixed across the three monitor-training seeds, so the reported seed spread reflects network initialization and batch order. The monitored heads are never used to generate the rollouts, and all signals are scored offline from the same recorded inputs.

\emph{Calibration and aggregation.} Conformal thresholds use the finite-sample empirical quantile $\lceil(n+1)(1-\alpha)\rceil/n$ at $\alpha=0.05$, computed from calibration-split successes. The multimodality flag uses a deterministic two-cluster initialization and the constants specified before scoring. Episode-level scores use an exponentially weighted moving average followed by the episode maximum ($\alpha=0.1$); window-maximum and mean alternatives give the same family ordering.

\emph{Protocol provenance.} The monitor protocol was committed before the first scoring output. It fixes the primary comparison, aggregation, stratification, baseline set, mechanism comparison, and exclusions. A later top-up to 1664 episodes continued the same cells and random sequence. The corresponding commit hash and the unchanged protocol file are included with the artifacts.

\emph{Scoring correction.} An initial implementation standardized the condition vector twice before computing $\max|z|$. This weakened the comparison baseline and therefore biased margins in favor of the residual monitor. All comparisons involving $\max|z|$ were recomputed from the raw conditions under the specified normalization; the correction reduced the margins but did not change the signal-family ordering. Only the corrected values appear in the paper. The false-alarm comparison, monitored-head sampling variances, and action-chunk-entropy calibration were unaffected and use the final scoring code.

\emph{External capacity analysis.} The modal-capacity analysis in Section~\ref{sec:liberoplus-coverage} was designed after the monitoring result and is not part of the registered criterion. It reuses the frozen traces and three monitor checkpoints, drawing 32 chunks per branch at each of 147{,}529 test steps. Confidence intervals resample the 851 episodes as clusters. Independent checks reproduced the coordinated sampler when tokens agree, the stored multimodality flag, and the per-condition estimator used elsewhere in the paper.

\emph{Per-seed capacity readings.} On the token partition, the coordinated arm gives $\Delta=+0.000182$, $+0.000110$, and $+0.000115$ for seeds 0--2. The corresponding episode-cluster intervals are
\[
[+0.000115,+0.000261],\qquad
[+0.000069,+0.000155],\qquad
[+0.000074,+0.000156].
\]
Restricting to steps with token-prior entropy above 0.30 nats gives $+0.00847$, $+0.00642$, and $+0.00647$; an independent Monte Carlo draw gives $+0.00719$. On the local partition, the coordinated arm gives $+0.000020\pm0.000097$ over seeds and $+0.000031\pm0.000040$ on the replicate.

\section{Falsifiability and remaining tests}
\label{app:falsify}

The paper's negative and mechanism claims are paired with explicit challenges; four have already changed the manuscript.
\begin{enumerate}\itemsep1pt\parskip0pt\parsep0pt
\item \emph{Auxiliary-latent null---not triggered.} A non-collapsed model satisfying (A1)--(A2) and (A3)--(A5), with a per-condition interval clearly outside the equivalence band, would contradict the implemented population prediction or expose a violated premise.
\item \emph{Partition token as a necessary route---refuted.} An amortized posterior reaches the balanced testbed's high-capacity corner, so partition tokens are claimed only as a robust sufficient robot-side route.
\item \emph{Map-invariant source-noise coordination---refuted, and the reversal preregistered-and-confirmed.} The same mechanism changes from anti-coordination to strong coordination across representation maps; the confirmatory factorial at ten seeds per cell upheld the reversal (interaction 90.1 pp, permutation $p=10^{-4}$). A continuous learned-map study remains an open test of the basin-alignment explanation.
\item \emph{Uniform token draws are free---refuted under skew.} Skew produces both prior distortion and lower normalized coordination; matching the symbol marginal repairs both in the tested regimes.
\item \emph{External monitoring---split verdict.} The preregistered prediction that coordination shrinks the false-alarm surcharge did not meet its registered criterion on LIBERO-Plus. The detection contrast is inconclusive at three seeds. The later capacity analysis is post hoc and supplies a testable regime explanation, not a confirmed prediction.
\item \emph{Open tests.} Repeat the monitor comparison with appreciable modal capacity, symbol entropy, and symbol obedience (the horizon sweep is now done, Table~\ref{tab:horizons}, with a split outcome), and extend the robot study beyond $K=2$ and to contact-rich or real trajectories.
\end{enumerate}

\end{document}